\documentclass[journal,onecolumn]{report}

\usepackage{epsfig,graphicx,color,hyperref,url} %graphics ,wrapfig %\usepackage{subfig}subcaption,
\usepackage{amssymb,amsmath,algorithm,algorithmic,amsfonts,amsthm, bm, bbm,dsfont,epstopdf,tikz,pgfplots,booktabs,mathtools,subcaption,times,comment} %subfigure %,microtype ,

\usetikzlibrary{decorations.pathreplacing,angles,quotes}
\usetikzlibrary{shapes,arrows,positioning}

\usepackage{pgfplots, pgfplotstable} \usepackage{mathtools,wrapfig}

\usepgfplotslibrary{groupplots}
\pgfplotsset{select coords between index/.style 2 args={
		x filter/.code={
			\ifnum\coordindex<#1\fi
			\ifnum\coordindex>#2\fi
		}
}}

\newtheorem{theorem}{Theorem}[chapter]

	\newtheorem{corollary}{Corollary}[chapter]

	\newtheorem{definition}{Definition}[chapter]
	\newtheorem{remark}{Remark}[chapter]

\newtheorem{tempassu}{Temporary Assumption}[chapter]
\newtheorem{fact}{Fact}[chapter]

\newtheorem{assu}{Assumption}[chapter]

\begin{document}

\newcommand{\x}{\bm{x}}
\newcommand{\X}{\bm{X}}
\newcommand{\s}{\bm{s}}
\renewcommand{\S}{\bm{S}}
\newcommand{\vv}{\bm{v}}
\renewcommand{\c}{c}

\newcommand{\e}{\bm{e}}
\newcommand{\y}{\bm{y}}
\newcommand{\w}{\bm{w}}
\newcommand{\U}{{\bm{U}}}
\newcommand{\R}{{\bm{R}}}
\newcommand{\G}{\bm{G}}
\newcommand{\M}{\bm{M}}
\renewcommand{\H}{\bm{H}}

\newcommand{\I}{\bm{I}}
\newcommand{\Y}{\bm{Y}}
\newcommand{\uhat}{{\bm{\u}}}
\newcommand{\Zhat}{{\bm{\hat{Z}}}}
\newcommand{\Z}{{\bm{{Z}}}}
\newcommand{\D}{{\bm{D}}}
\newcommand{\Q}{{\bm{Q}}}
\newcommand{\F}{{\bm{F}}}
\newcommand{\iidsim}{\stackrel{\mathrm{iid}}{\thicksim }}
\newcommand{\n}{{\cal{N}}}
\newcommand{\indepsim}{\stackrel{\mathrm{indep.}}{\thicksim }}
\newcommand{\SE}{\text{{SubsDist}}}
\newcommand{\dist}{\mathrm{dist}}
\newcommand{\tta}{\ta^{\text{trunc}}}
\newcommand{\A}{\bm{A}}
\newcommand{\full}{{\mathrm{full}}}

\newcommand{\Span}{\mathrm{span}}
\newcommand{\rank}{\mathrm{rank}}
\newcommand{\evdeq}{\overset{\mathrm{EVD}}=} %{\stackrel{EVD}{=}}
\newcommand{\svdeq}{\overset{\mathrm{SVD}}=} %{\stackrel{EVD}{=}}
\newcommand{\qreq}{\overset{\mathrm{QR}}=} %{\stackrel{EVD}{=}}
\newcommand{\bi}{\begin{itemize}} \newcommand{\ei}{\end{itemize}}
\newcommand{\ben}{\begin{enumerate}} \newcommand{\een}{\end{enumerate}}
\newcommand{\vs}{\vspace{0.1in}}
\newcommand{\vsl}{\vspace{0.05in}}
\newcommand{\vsm}{\vspace{-0.1in}}

\renewcommand{\implies}{\Rightarrow}
\newcommand{\V}{{\bm{V}}}
\newcommand{\W}{{\bm{W}}}

%%uncomment above for making this main file

\newcommand{\cblue}{\color{blue}}
\newcommand{\cbl}{\color{black}}
\newcommand{\cred}{\color{red}}
\newcommand{\skipit}{ }

\newcommand{\trace}{\mathrm{tr}} %{\mathrm{trace}}

\newcommand{\qfull}{q_\full}
\newcommand{\sub}{{\mathrm{sub}}}

\newcommand{\bbf}{\mathbb{F}}
\newcommand{\dsW}{\mathds{W}}
\newcommand{\dsZ}{\mathds{Z}}

\newcommand{\mtx}[1]{\mathbf{#1}}
\newcommand{\vct}[1]{\mathbf{#1}}
\newcommand{\abs}[1]{\left|#1\right|}
\newcommand{\p}{\bm{p}}
\renewcommand{\j}{\bm{j}}
\newcommand{\uu}{{\bm{u}}}
\newcommand{\tot}{\mathrm{tot}}
\renewcommand{\a}{\bm{a}}
\newcommand{\h}{\bm{h}}
\renewcommand{\b}{\bm{b}}
\newcommand{\B}{\bm{B}}
\renewcommand{\aa}{\bm{a}}

\newcommand{\xhat}{\hat\x}
\newcommand{\Xhat}{\hat\X}

\newcommand{\bhat}{\hat\b}
\newcommand{\Uhat}{\hat\U}

\newcommand{\bP}{{\bm{P}}}

\newcommand{\z}{\bm{z}}
\newcommand{\indic}{\mathbbm{1}}
\newcommand{\one}{\mathbm{1}}

\newcommand{\C}{\bm{C}}
\newcommand{\Chat}{\bm{\hat{C}}}
\newcommand{\cbold}{\bm{c}}

\newcommand{\bz}{\boldsymbol{z}}

\setlength{\arraycolsep}{0.01cm}%{0.03cm}

\newcommand{\matdist}{\text{mat-dist}}
\newcommand{\Bstar}{{\B^*}}
\newcommand{\bstar}{\b^*}
\newcommand{\tB}{\B^*}
\newcommand{\tb}{\b^*}

\newcommand{\td}{\tilde{\bm{d}}^*}
\newcommand{\init}{{\mathrm{init}}}

\newcommand{\Ustar}{\U^*{}}
\newcommand{\Xstar}{\X^*}
\newcommand{\xstar}{\x^*}
\newcommand{\sstar}{\s^*}
\newcommand{\Sstar}{\S^*}
\newcommand{\Vstar}{\V^*{}}

\newcommand{\deltinit}{\delta_\init}
\newcommand{\deltapt}{\delta_{t}}
\newcommand{\deltaptplus}{\delta_{t+1}}

\newcommand{\bSigma}{{\bm\Sigma^*}}
\newcommand{\tSigma}{\bm{E}_{det}}
\newcommand{\sigmin}{{\sigma_{\min}^*}}
\newcommand{\sigmax}{{\sigma_{\max}^*}}

\newcommand{\HH}{{\bm{D}}}
\newcommand{\GG}{{\bm{M}}}
\newcommand{\SSS}{{\bm{S}}}
\newcommand{\ik}{{ik}}
\newcommand{\J}{\mathcal{J}}

%---
\renewcommand{\P}{\bm{P}}
\newcommand{\proj}{\mathcal{P}}
\newcommand{\E}{\mathbb{E}}
\newcommand{\norm}[1]{\left\|#1\right\|}

\renewcommand{\P}{\bm{U}}
\newcommand{\Phat}{\hat\P} %\renewcommand{\Phatt}{\hat\P_t}
\newcommand{\Lam}{\bm\Lambda} %\renewcommand{\Phatt}{\hat\P_t}
\renewcommand{\L}{\bm{L}}
\renewcommand{\V}{\bm{V}}

\renewcommand{\l}{\bm{\ell}}
\renewcommand{\v}{\bm{v}}
\newcommand{\tty}{\tilde\y}

\newcommand{\lhat}{\hat\l}

\newcommand{\at}{\a_t}
\newcommand{\yt}{\y_t}
\newcommand{\lt}{\l_t}
\newcommand{\xt}{\x_t}
\newcommand{\vt}{\v_t}
\newcommand{\et}{\e_t}

\newcommand{\T}{\mathcal{T}}
\newcommand{\Tmiss}{{\T_{miss}}}
\newcommand{\Tspar}{{\T_{spar}}}
\newcommand{\outfrac}{\text{\scriptsize{max-outlier-frac}}}
\newcommand{\outfracrow}{\text{\scriptsize{max-outlier-frac-row}}}
\newcommand{\outfraccol}{\text{\scriptsize{max-outlier-frac-col}}}
\newcommand{\missfracrow}{\text{\scriptsize{max-miss-frac-row}}}
\newcommand{\missfraccol}{\text{\scriptsize{max-miss-frac-col}}}

\newcommand{\smin}{s_{min}}
\newcommand{\smax}{s_{max}}

\newcommand{\bea}{\begin{eqnarray}}
\newcommand{\eea}{\end{eqnarray}}

\newcommand{\nn}{\nonumber}
\newcommand{\ds}{\displaystyle}

\newcommand{\snr}{\text{SNR}}

\newcommand{\tk}{\tilde{k}}
\newcommand{\tl}{{\ell}}
\renewcommand{\tot}{L}

\newcommand{\bfpara}[1]{\noindent {\bf #1. }} %
\newcommand{\empara}[1]{\noindent {\em #1. }} %
\newcommand{\para}[1]{\noindent {\bf #1. }} %
\newcommand{\Subsubsection}[1]{ \vspace{-0.14in} \subsubsection{#1}  \vspace{-0.1in} }
\newcommand{\Section}[1]{ \vspace{-0.15in} \section{#1}  \vspace{-0.12in} } %  \vspace{-0.1in}   %\vspace{-0.05in}
\newcommand{\Subsection}[1]{ \vspace{-0.15in} \subsection{#1}  \vspace{-0.1in} }   %{ \vspace{-0.175in} \Subsection{#1}  \vspace{-0.1in} }
\newcommand{\Item}{ \vspace{-0.05in} \item \vspace{-0.1in} }
\renewcommand{\forall}{\text{ for all }}

\renewcommand{\subsubsection}{\bfpara}%{#1}

\newcommand{\tX}{\bm{\mathcal{X}}}
\newcommand{\tXstar}{\bm{\mathcal{X}^*}}
\newcommand{\tG}{\bm{\mathcal{B}}}%{\bm{\mathcal{G}}}
\newcommand{\tGstar}{\bm{\mathcal{G}^*}}
\newcommand{\tS}{\bm{\mathcal{S}}}
\newcommand{\tSstar}{\bm{\mathcal{S}^*}}
\newcommand{\tA}{\bm{\mathcal{A}}}
\newcommand{\tY}{\bm{\mathcal{Y}}}
\newcommand{\tZ}{\bm{\mathcal{Z}}}
\newcommand{\ttY}{\bm{\tilde{\mathcal{Y}}}}
\newcommand{\Ph}{\bm{\Phi}}
\newcommand{\Ps}{\bm{\Psi}}

\newcommand{\eps}{\epsilon}

\newcommand{\zero}{\bm{0}}
\renewcommand{\Z}{{\bm{Z}}}
\newcommand{\argmin}{\arg\min}
\newcommand{\Zfull}{\bm{Z}}
\newcommand{\Za}{{\bm{Z}_a}} \newcommand{\Zb}{{\bm{Z}_b}}
\renewcommand{\S}{\bm{S}}
%{\bm{S}}

\newcommand{\ymag}{{\y_{mag}}}

\newcommand{\bfA}{\mathbf{C}}  %{\mathbf{A}}
\newcommand{\bfB}{\mathbf{D}}  %{\mathbf{B}}
\newcommand{\bfx}{\mathbf{x}}

\newcommand{\calS}{\mathcal{S}}
\newcommand{\calD}{\mathcal{D}}

\renewcommand{\Ustar}{{\U^*}}
\renewcommand{\Bstar}{\B^*}
\renewcommand{\P}{\mathcal{P}}
\renewcommand{\Y}{{\bm{Y}}}
\renewcommand{\C}{\bm{C}}

\newcommand{\Pperp}{\P_{\Ustar,\perp}}

\renewcommand{\xhat}{\x}
\renewcommand{\Xhat}{\X}
\newcommand{\g}{\bm{g}}

\newcommand{\gradU}{\mathrm{gradU}}
\newcommand{\Utilde}{\tilde{\U}}
\newcommand{\sigmamin}{\sigma_{\min}} \newcommand{\sigmamax}{\sigma_{\max}}

\newcommand{\Err}{\mathrm{Err}}

\renewcommand{\ik}{{ki}}
\newcommand{\jk}{{jk}}
\renewcommand{\u}{\bm{u}}
\newcommand{\ustar}{{\u^*}}
\renewcommand{\b}{\bm{b}}
\newcommand{\m}{\bm{m}}
\newcommand{\calU}{\mathcal{U}}

\newcommand{\totl}{\gamma}
\newcommand{\mathcalA}{\mathcal{A}}

\title{AltGDmin: Alternating GD and Minimization \\ for Partly-Decoupled (Federated) Optimization}%Fast Communication-
%\title{Alternating GD and Minimization (AltGDmin) for  Fast Federated Few Shot Learning} %Communication-Efficient and
\author{Namrata Vaswani %, Seyedehsara Nayer, Silpa Babu, Shana Moothedath, Ankit Singh, Ahmed Abbasi
\\ Iowa State University, Ames, IA, USA, \\ Contact: namrata@iastate.edu}
\maketitle

\begin{abstract}
This article describes a novel optimization solution framework, called alternating gradient descent (GD) and minimization (AltGDmin), that is useful for many problems for which alternating minimization (AltMin) is a popular solution.  AltMin is a special case of the block coordinate descent algorithm that is useful for problems in which minimization w.r.t one subset of variables keeping the other fixed is closed form or otherwise reliably solved. Denote the two blocks/subsets of the optimization variables $\Z$ by $\Za, \Zb$, i.e.,  $\Z = \{\Za, \Zb\}$.
AltGDmin is often a faster solution than AltMin for any problem  for which (i) the minimization over one set of variables, $\Zb$, is much quicker than that over the other set, $\Za$; and (ii) the cost function is differentiable w.r.t. $\Za$. Often, the reason for one minimization to be quicker is that the problem is ``decoupled" for $\Zb$ and each of the decoupled problems is quick to solve. This decoupling is also what makes AltGDmin communication-efficient for federated settings.

Important examples where this assumption holds include (a)  low rank column-wise compressive sensing (LRCS), low rank matrix completion (LRMC), (b) their outlier-corrupted extensions such as robust PCA, robust LRCS and robust LRMC; (c) phase retrieval and its sparse and low-rank model based extensions; (d)  tensor extensions of many of these problems such as tensor LRCS and tensor completion; and (e) many partly discrete problems where GD does not apply -- such as clustering, unlabeled sensing, and mixed linear regression. LRCS finds important applications in multi-task representation learning and few shot learning, federated sketching, and accelerated dynamic MRI. LRMC and robust PCA find important applications in recommender systems, computer vision and video analytics.
\end{abstract}

\tableofcontents

\part{Commonly Used Optimization Algorithms and AltGDmin}

%
%%This review article will explain the above problems in detail; explain the AltGDmin algorithm idea and its advantages; briefly describe the theoretical sample, time, and communication complexity guarantees that have been proved for it and that substantiate our ``efficiency" claim. We will also describe a secure (Byzantine attack resilient) version of AltGDmin.  Finally, we will provide a brief summary of the practical power (speed and accuracy) of AltGDmin over existing approaches for accelerated dynamic MRI.
%\end{abstract}

%
\chapter{Introduction}
This article describes a novel algorithmic framework, called Alternating Gradient Descent (GD) and Minimization or AltGDmin for short, that is useful for optimization problems that are ``partly decoupled'' \cite{lrpr_gdmin}.
Consider the optimization problem $\min_{\Z} f(\Z)$. This is partly-decoupled if we can split the set of optimization variables $\Z$ into two blocks, $\Z= \{\Za, \Zb\}$, so that the minimization over $\Zb$, keeping $\Za$ fixed, is decoupled. This means that it can be solved by solving many smaller-dimensional, and hence much faster, minimization problems over disjoint subsets of $\Zb$. That over $\Za$, keeping $\Zb$ fixed, may or may not be decoupled. We provide examples below and define this mathematically in Sec. \ref{partlydecoup}. %After initializing $\Za$, AltGDmin alternatively updates $\Zb$ and $\Za$ by using minimization for $\Zb$ (keeping $\Za$ fixed) and  GD for $\Za$ (keeping $\Zb$ fixed).

For problems for which one of the two minimizations is decoupled, and hence fast, while the other is not, AltGDmin often provides a much faster solution than the well-known Alternating Minimization (AltMin) \cite{csiszar1984information,byrne2011alternating} approach. Even if both problems are decoupled, AltGDmin still often has a communication-efficiency advantage over AltMin when used in distributed or federated settings. This is the case when the data is distributed across the nodes in such a way that the decoupled minimization over a subset of $\Zb$ also depends on the subset of data available at a node; so this can be solved locally. %, AltGDmin provides a very communication-efficient solution. %In these cases, it is often much more efficient compared to AltMin.

Federated learning is a setting where multiple distributed nodes or  entities or clients collaborate to solve a machine learning (ML) problem and where different subsets of the data are acquired at the different nodes. Each node can only communicate with a central server or service provider that we refer to as ``center'' in this article.
Communication-efficiency is a key concern with all distributed algorithms, including federated ones. Privacy is another key concern in federated learning. Both concerns dicate that the data observed or measured at each node/client be stored locally and not be shared with the center. Summaries of it can be shared with the center. The center typically aggregates the received summaries and broadcasts the aggregate to all the nodes \cite{kairouz2021advances}. %Privacy concerns in federated learning dictate that we cannot share the raw data available/acquired at the nodes with the center. Only summaries of it can be shared at each algorithm iteration. %Privacy concerns also require that the unknown signal
In this article, ``privacy'' only means the following: the nodes' raw data cannot be shared with the center; and the algorithm should be such that the center cannot reconstruct the entire unknown true signal (vector/matrix/tensor). %This point will become clearer later in Sec \ref{guarantees}.

One of the challenges in federated learning is developing algorithms that are resilient to adversarial attacks on the nodes; resilience to Byzantine attacks is especially critical.
%Most of our discussion also applies for distributed computing settings where data is available centrally, but is distributed to nodes, e.g., over the cloud, to parallelize and hence speed up the computing.
% (subsets of $\Zb$ depend only on subsets of the available data),
An important challenge in distributed computing settings (data is available centrally, but is distributed to nodes, e.g., over the cloud, to parallelize and hence speed up the computing) is to have algorithms that are resilient to stragglers (some worker nodes occasionally slowing down or failing)\cite{tandon2017gradient,ramamoorthyMG24}. As will become clear in this article, the design of both attack resilient and straggler resilient modifications of AltGDmin is also efficient. One example of Byzantine attack resilient AltGDmin is studied in \cite{lrcs_byz_icml}. %as efficient as that for standard GD. The a

\bfpara{Monograph Organization}
This monograph begins by giving some examples of partly decoupled optimization problems and their applications below. In Chapter \ref{algos}, we provide a short overview of some of the popular optimization algorithms - gradient descent (GD), block coordinate descent and AltMin, and nonlinear least squares -- and when these work well. All these are iterative algorithms that need an initialization. We describe common initialization approaches as well. Then, in Chapter \ref{altgdmin}, we precisely define a partly decoupled problem and develop and discuss the AltGDmin algorithmic framework. In the second part of this monograph, in Chapter \ref{guarantees}, we provide the AltGDmin algorithm details, including initialization, for three important LR matrix recovery problems - LR column-wise sensing, LR phase retrieval and LR matrix completion. We also state and discuss the theoretical sample and iteration complexity guarantees that we can prove for these problems. The iteration complexity helps provide total computational and communication complexity bounds.
The third part of this monograph discusses proof techniques. We first provide the general proof approach that can be used to analyze the AltGDmin in Chapter \ref{proof_ideas} and then describe the key ideas for LR problems in Chapter \ref{proof_ideas_lr}. Details are in Chapter \ref{proof_details}. Preliminaries used in these proofs are provided and explained in Chapter \ref{prelims}. This chapter provides a short overview of the most useful linear algebra and random matrix theory topics from \cite{versh_book} and \cite{spectral_init_review}. In the last part of this monograph, Chapter \ref{openques} describes open questions including other problems where AltGDmin or its generalization may be useful.

\section{Partly Decoupled Optimization Examples}
We provide a few examples of partly decoupled problems.% We define these precisely in Sec. \ref{partlydecoup}.

\bfpara{Low Rank Column-wise compressive Sensing (LRCS)}
%Low Rank Column-wise compressive Sensing (LRCS)
This problem involves recovering an $n \times q$ rank-$r$ matrix $\Xstar$, with $r \ll \min(n,q)$, from column-wise undersampled (compressive) measurements,  $\y_k:= \A_k \xstar_k$ , $k \in [q]$. The matrices $\A_k$ are dense (non-sparse) matrices that are known. Each $\y_k$ is an $m$-length vector with $m < n$. Let $\Y:=[\y_1, \y_2, \dots, \y_q]$ denote the observed data matrix. %A more practical  model involves recovery $\Xstar$ from $\y_k:= \A_k \xstar_k + \w_k$, where $\w_k$ is modeling error or noise.
We can solve this problem by considering the squared loss function. It then becomes a problem of finding a matrix $\X$ of rank at most $r$ that minimizes $\sum_{k=1}^q \|\y_k - \A_k \x_k\|_2^2.$
Suppose that $r$ or an upper bound on it is known. This problem can be converted into an unconstrained, and smaller dimensional, one by factorizing $\X$ as $\X= \U \B$, where $\U$ and $\B$ are matrices with $r$ columns and rows respectively.
Thus, the goal is to solve
\begin{equation}
\argmin_{\U,\B} f(\U,\B) := \argmin_{\U,\B} \sum_{k=1}^q  \|\y_k - \A_k \U \b_k\|_2^2.
\label{lrcs_cost}
\end{equation}
Notice that $\b_k$ appears only in the $k$-th term of the above summation. Thus, if we needed to minimize over $\B$, while keeping $\U$ fixed, the problem decouples column-wise. The opposite is not true. We refer to such a  problem as a partly decoupled problem.

In solving the above problem iteratively, there can be numerical issues because $\U \B = \U \R \R^{-1} \B$ for any $r \times r$ invertible matrix $\R$. The norm of $\U$ could keep increasing over iterations while that of $\B$ decreases or vice versa. To prevent this, either the cost function is modified to include a norm balancing term, e.g., as in \cite{rpca_gd}, or one orthonormalizes the estimate of $\U$ after each update.

Three important practical applications where the LRCS problem occurs include (i) federated sketching  \cite{one_sketch_all,sketch_1,sketch_2,hughes_icip_2012,hughes_icml_2014,lee2019neurips,cov_sketch}, (ii) accelerated (undersampled) dynamic MRI with the low rank (LR) model on the image sequence, and (iii) multi-task linear representation learning to enable few shot learning \cite{du2020few,lrcs_byz_icml,collins2021exploiting,netrapalli}. In fact, some works refer to the LRCS problem  as multi-task representation learning. (iv) The LRCS problem also occurs in for parameter estimation in multi-task linear bandits \cite{lin2024fast}.

\bfpara{Low Rank Phase Retrieval (LRPR)} This is the phaseless extension of LRCS \cite{lrpr_it,lrpr_best,lrpr_gdmin} but it was studied in detail before LRCS was studied. This involves solving
\begin{equation}
\argmin_{\U,\B} f(\U,\B) := \argmin_{\U,\B} \sum_{k=1}^q  \|\y_k - |\A_k \U \b_k| \|_2^2
\label{lrpr_cost}
\end{equation}
where $|.|$ computes the absolute value of each vector entry. LRPR finds applications in dynamic Fourier ptychography \cite{holloway,TCIgauri}.

\bfpara{LR matrix completion (LRMC)}
%Another example is LR matrix completion (LRMC).
In this case, the cost function is partly decoupled w.r.t. both $\U$ and $\B$ (keeping the other fixed). This involves recovering a LR matrix from a subset of its observed entries. Letting $\Omega$ denote the set of observed matrix entries, and letting $\P_{\Omega}$ denote the linear projection operator that returns a matrix of size $n \times q$ with the unobserved entries set to zero,  this can be expressed as a problem of learning $\Xstar$ from $\Y:= \P_{\Omega}(\Xstar)$. Letting the unknown $\X$ as $\X = \U \B$ as above, the optimization problem to solve now becomes
\begin{equation}
\argmin_{\U,\B} f(\U,\B) := \|\Y - \P_{\Omega}(\Xstar)\|_F^2 =  \sum_{k=1}^q \|\y_k - \P_{\Omega_k} (\U \b_k)\|_2^2  =   \sum_{j=1}^n \|\y^j - \P_{\Omega^j} (\u^j{}^\top \B)\|_2^2
%\sum_{j=1}^n \sum_{k=1}^q f_{\y_k}^k(\U,\b_k)), \ f_{\y_k}^k(\U,\b_k):= \|\y_k - \A_k \U \b_k\|_2^2.
\label{lrmc_cost}
\end{equation}
with $\B=[\b_1, \b_2, \dots, \b_k, \dots \b_q]$, $\U^\top = [\u_1, \u_2, \dots, \u_j, \dots, \u_n]$, $\Omega_k:=\{ j: (j,k) \in \Omega \}$ and $\Omega^j:=\{ k: (j,k) \in \Omega \}$.
Notice that the above problem is decoupled over $\B$ for a given $\U$, and vice-versa.
 LRMC finds important applications in recommender systems' design, survey data analysis, and video inpainting \cite{matcomp_candes}. LRMC also finds applications in parameter estimation for reinforcement learning; in particular for filling in the missing entries of its state transition probability matrix.

\bfpara{Other Partly Decoupled Examples}
Other examples of partly decoupled problems include non-negative matrix factorization, sparse PCA,  robust PCA and extensions (robust LRCS and robust LRMC), tensor LR slice-wise sensing and its robust extension, and LR tensor completion; and certain partly discrete problems -- clustering, shuffled or unlabeled sensing, and mixed linear regression. We describe these in the open questions chapter, Chapter \ref{openques}.
%. A different class of problems involves those in which some of the variables are discrete valued,

\subsection{Detailed description of some applications of above problems}%multi-task and few shot learning:
\bfpara{Why LR model}
Medical image sequences change slowly over time and hence these are well modeled as forming a low-rank matrix with each column of the matrix being one vectorized image \cite{dyn_mri1,lrpr_gdmin_mri_jp}. The same is often also true for similar sets of natural images and videos \cite{rpca,rrpcp_dynrpca}. The matrix of user ratings of different products, e.g., movies, is modeled as a LR matrix under the commonly used hypothesis that the ratings are explained by much fewer factors than the number of users, $q$, or products, $n$ \cite{matcomp_candes}.
In fact, many large matrices are well modeled as being LR \cite{lowrank_bigdata};  these model any image sequence or product ratings or survey dataset, in which most of the differences between the different images or ratings or survey data, $q$, are explained by only a small number $r$ of factors.
%The robust PCA or  LR plus sparse (L+S) model includes LR as its special case, while also allowing for sparse outliers, e.g., cardiac abnormalities in MRI, contrast imaging in MRI, or moving foreground objects in video or product ratings with outlier entries \cite{rpca,rpca_gd,rrpcp_icml,rrpcp_dynrpca,candes_mri}.

%If the FT is fully sampled, then just an inverse FT suffices.
\bfpara{MRI}
In MRI, which is used in medicine for cross-sectional imaging of human organs, after some pre-processing, the acquired data can be modeled as the 2D discrete Fourier transform (FT) of the cross-section being imaged. This is acquired one FT coefficient (or one row or line of coefficients) at a time \cite{cscompressible,sparsedynamicMRI_lustig}. The choice of the sampled coefficients can be random or it may be specified by carefully designed trajectories. The goal is reconstruct the image of the cross-section from this acquired data.  If we can reconstruct accurately from fewer samples, it means that the acquisition can be speeded up. This is especially useful for dynamic MRI because it can improve the temporal resolution for imaging the changes over time, e.g. the beating heart.  Accelerated dynamic MRI involves doing this to recover a sequence of $q$ images, $\xstar_k$, $k \in [q]$, say, of the beating heart or of brain function as brain neurons respond to a stimuli, or of the vocal tract (larynx) as a person speaks, from undersampled DFT measurements $\y_k$, $k \in [q]$. Here $\xstar_k$ is a vectorized image.
The matrices $\A_k$ are the partial Fourier matrices represented the 2D DFT (or sometimes the FT in case of radial sampling) computed at the specified frequencies.

\bfpara{Multi-task Learning}
Multi-task  representation learning refers to the problem of jointly estimating the model parameters for a set of related tasks. This is typically done by learning a common lower-dimensional ``representation" for all of their feature vectors. This learned representation can then be used for solving the meta-learning or learning-to-learn problem: learning model parameters in a data-scarce environment. This strategy is referred to as ``few-shot'' learning. In recent work \cite{du2020few}, a very interesting low-dimensional linear representation was introduced and the corresponding low rank matrix learning optimization problem was defined. This linear case will be solved if we can solve \eqref{lrcs_cost}. Simply said this can be understood as a problem of jointly learning the coefficients' for $q$ related linear regression problems, each with their own dataset $\A_k$, and with the regression vectors $\xstar_k$ being correlated (so that low rank is a good model on the matrix formed by these vectors, $\Xstar$). Once the ``common representation'' (the column span subspace matrix $\U$) can be estimated, we can solve a new linear regression problem that is related (correlated) with these hold ones by only learning a new $r$-dimensional vector $\b_k$ for it.

\bfpara{Federated Sketching}
For the vast amounts of data acquired on smartphones/other devices, there is a need to compress/sketch it before it can be stored or transmitted. The term ``sketch" refers to a compression approach, where the compression end is very inexpensive  \cite{one_sketch_all,sketch_1,sketch_2,hughes_icip_2012,hughes_icml_2014,lee2019neurips,cov_sketch}. A common approach to sketching, that is especially efficient in distributed settings, is to multiply each vectorized image by a different independent $m \times n$ random matrix (typically random Gaussian or Rademacher matrix) with $m < n$, and to store or transmit this sketch.

\chapter{Commonly used optimization algorithms} %and Alternating GD and Minimization (AltGDmin)
\label{algos}
%?? remove subscript $\Y$. use $\calD$ to denote data which is $\y_k,\A_k, k \in \calS_\ell$.

%?? add Byz res, privacy, straggler res to table.

\section{The optimization problem and its federated version}
Our goal is to solve
\[
\arg\min_\Z  f(\Z; \calD)
\]
Here $\calD$ is the available data.
%Here $\Y$ is the available data  and $\Z$ is the unknown variable.

\bfpara{Federation} We assume that there are a total of $\gamma$ distributed nodes, and that disjoint subsets of data are observed / sensed / measured at the different nodes. We use $\calD_\ell$ to denote the subset of the data available at node $\ell$. Thus $\cup_{\ell=1}^\gamma \calD_\ell = \calD$ is all the available data and $ \calD_\ell$s are disjoint.
We assume that $f(\Z; \calD)$ is a sum of $\gamma$ functions, each of which depends on a subset of the data $\calD_\ell$, i.e., that
\[
f(\Z; \calD) = \sum_{\ell=1}^\gamma f_\ell (\Z ; \calD_\ell)
\]
In the partly decoupled problems that we study, $f_\ell$ is a function of only a subset of $\Z$, e.g., in LRCS, $f_\ell$ depends only on $\U, \b_k$. %In this case $\calD_\ell = \{\y_k, \A_k, k \in
We clarify this later.

%If there are a total of $q$ data points,
%
%$\Y_\ell$ (typically matrix columns) is available at node $\ell$ for each $\ell \in [\gamma]$. To be precise, if $\Y$ is a matrix then
%\[
%\Y = [\Y_1, \Y_2, \dots \Y_\gamma]
%\]
%As we explain later, for a problem that is partly decoupled for one set of variables $\Zb$, this implies that $(\Zb)_\ell$ can be updated locally with using the updated $\Za$ from the previous iteration and $\Y_\ell$. %either it is observed/sensed/sketched at the node (federated setting) or it is sent to the node (distributed computing setting).

%We first describe briefly the two well-known solutions and their pros and cons in order to motivate AltGDmin.

% (direction of negative gradient)
\section{Gradient Descent (GD)}\label{gd}
The simplest, and most popular, class of solutions for solving an unconstrained (convex or non-convex) optimization problem is gradient descent (GD) \cite{cauchy1847,boyd2004convex} and its modifications, the most common being various versions of stochastic GD. GD starts with an initial guess and attempts to move little steps in the direction opposite to that of the gradient of the cost function at the previous iterate.
Pseudo-code for simple GD is as follows.
\bi
\item Initialize $\Z$ to $\hat\Z$.
\item Run the following steps $T$ times, or until a stopping criterion is reached.
\ben
\item Update $\Z$ by one GD step: $\hat\Z \leftarrow \hat\Z - \eta \nabla_\Z f(\hat\Z; \calD)$. Here $\eta$ is the GD step size, also often referred to as the learning rate.

\een
The stopping criterion usually involves checking if $\|\nabla_\Z f(\hat\Z)\|/||\hat\Z||$ is small enough.

\item Federated setting: At each algorithm iteration, each node $\ell$ computes the partial gradient $\nabla_\Z f_\ell (\hat\Z; \calD_\ell)$ and sends it to center, which computes $\nabla f = \sum_\ell \nabla f_\ell$ and the GD update.

\ei

For strongly convex cost functions, GD provably converges to the unique global minimizer of $f(\Z)$ starting from any initial guess.
For convex functions, one can prove that it converges to {\em a} global minimizer (there are multiple global minimizers in this case).
%GD actually converges to a global minimizer because, by definition, any local minimum of a convex function is also a global minimum.
For non-convex functions, convergence to a global minimizer cannot be guaranteed. One can only show that GD will converge to a  {\em a} stationary point of the cost function under mild assumptions. All above results require that the GD step size is small enough \cite{boyd2004convex}.
%reach within a small window of {\em a} stationary point of the cost function under mild assumptions and .
%
% for convex ones, it converges to a local minimizer. Which minimizer it converges to depends on the initialization used. For non-convex functions, there are no general guarantees.

In signal processing and machine learning, the goal is to learn the ``true solution'' which is one of many local minimizers of the specified cost function. Henceforth, we refer to this as the ``desired minimizer''.
For certain classes of non-convex cost functions that are ``nice'', such as those that arise in phase retrieval \cite{wf,twf} or in various LR matrix recovery problems, one can prove results of the following flavor. If the available number of data samples is large enough,  if the step size is small enough, and if the initialization is within a certain sized window of the desired minimizer, then, the GD estimate will converge to the desired minimizer, with high probability (w.h.p.), e.g., see \cite{wf,twf,rpca_gd}. Such results are also non-asymptotic and provide an order-wise bound on the iteration complexity (number of iterations needed by the algorithm to get within $\epsilon$ normalized distance of the desired minimizer).  However, in some cases, such results requires a very small step size, e.g., for phase retrieval in \cite{wf}. This, in turns, means that GD has high iteration complexity. In other cases, such as for LRCS, it is not possible to show that a GD algorithm converges at all \cite{lrpr_gdmin}. 

%by first designing an initialization that is provably close to that minimizer.
%For others, GD makes too little progress per iteration, and hence either does not provably converge to a minimizer; or needs a very small step size (and hence too many iterations) to converge.

%For some others, the gradient computation itself is not easy or possible (except numerically), or, the gradient cannot be computed at all because the cost function is not differentiable w.r.t. the entire $\Z$. For some of these problems, the AltMin solution described next is useful.

%AltGDmin

%e.g., those that occur in sparse or low-rank matrix/tensor recovery,

%We first begin by describing the AltMin algorithm in detail and explaining its pros and cons.

\section{Block Coordinate Descent and Alternating Minimization (AltMin)}\label{altmin}
Coordinate descent involves minimizing over one scalar optimization variable at a time, keeping others fixed, and repeating this sequentially for all variables. Block Coordinate Descent or BCD involves doing this for blocks (subsets) of variables, instead of one variable at a time. AltMin is a popular special case of BCD that splits the variables into two blocks $\Z = \{\Za, \Zb\}$. This is extensively used and studied theoretically because of its simplicity \cite{csiszar1984information,byrne2011alternating}.
%Recall that the goal is to solve $\arg\min_\Z  f(\Z)$.
%For some optimization problems, the AltMin solution is more efficient, or easier to prove convergence guarantees for, or both.
%AltMin splits the optimization variables, $\Z$, into two subsets $\Z = \{\Za, \Zb\}$ so that $\arg\min_\Z  f(\Z) = \arg\min_{ \{\Za,\Zb \} } f(\Za,\Zb)$.
It solves $\arg\min_\Z  f(\Z) = \arg\min_{ \{\Za,\Zb \} } f(\Za,\Zb)$ using an iterative algorithm that starts with initializing $\Za$, and then alternatively updates $\Zb$ and $\Za$ using minimization over one of them keeping the other fixed.  The following is pseudo-code for AltMin.
\bi
\item Initialize $\Za$ to $\hat\Za$.
\item Alternate between the following two steps $T$ times (or until a stopping criterion is reached).
\ben
\item Update $\Zb$ keeping $\Za$ fixed: $\hat\Zb \leftarrow \min_{\Zb} f(\hat\Za, \Zb)$
\item Update $\Za$ keeping $\Zb$ fixed: $\hat\Za \leftarrow \min_{\Za} f(\Za, \hat\Zb)$
\een
\item Federated setting: This is not easy and has to be considered on a problem-specific basis. For most problems it is not efficient. %; typically one of the two minimization problems will need to be solved using multpGD. %requires problem specific design and is usually not efficient.
\ei
BCD is a generalization of this algorithm to the case when $\Z$ needs to be split into more than two subsets for the individual optimizations to be closed form or otherwise reliably solvable.

AltMin works well, and often can be shown to provably converge, for  problems in which the two minimizations can either be solved in closed form, or involve use of a provably convergent algorithm. Bilinear problems, such as the LRCS and LRMC problems described earlier, are classic examples of settings in which each of the minimization problems is a Least Squares (LS) problem, and hence, has a closed form solution \cite{lowrank_altmin,lrpr_best}. LRPR \cite{lrpr_best} is an example problem in which one minimization is LS while the other is a standard phase retrieval problem, with many provably correct iterative algorithms, e.g., \cite{twf}.
A second class of problems where AltMin is used, while GD is not even applicable, is those for which the cost function is not differentiable w.r.t. some of  the variables. Clustering is one example of such a problem; the k-means clustering algorithm is an AltMin solution.

%For some of the problems for which both GD and AltMin converge,
As explained above in Sec. \ref{gd}, for some of the problems for which both AltMin and GD are applicable, either GD cannot be shown to converge or it requires a very small step size to provably converge, making its iteration complexity too high. For certain classes of ``nice'' problems such as the LR problems described earlier, when initialized carefully, AltMin makes more progress towards the minimizer in each iteration, and hence converges faster: it can be shown to have an iteration complexity that depends logarithmically on the final error level.
However, typically, AltMin is much slower per iteration than GD. An exception is problems in which both the minimization problems of AltMin are decoupled and hence very fast. %, i.e., have time complexity comparable to that needed just to compute a gradient. LRMC is one such example.
%The only example settings we are aware of where there is happens is when both the minimization problems are decoupled (subsets of $\Za$ depend on subsets of $\Y$, and the same holds for $\Zb$ too), e.g., this is true for  low rank matrix completion (LRMC).

Moreover, a federated or distributed modification of AltMin is almost never efficient. This is the case even when both minimizations are decoupled like for LRMC. The reason is one of the minimization steps will require using data from multiple nodes. This will either require use of multiple GD iterations to solve the minimization problem (slow and communication-inefficient) or it will require all nodes to send their raw data to the center which distributes it (communication-inefficient and not private either). %?? We discuss this point in detail after explaining AltGDmin.

%?? move to later: for LRMC: altmin, gd, altgdmin all have same total time cost, gd needs more iters but less time per iter. Comm cost per iter: same for gd and altgdmin, higher for altmin. thus total comm cost is lowest for altgdmin if we treat kappa as constant. sample comp: best for gd.

%Unless there is decoupling, when compared to gradient descent (GD), the per iteration time cost of AltMin is higher. But, the hope is that it makes more progress, per iteration, towards the desired solution, and consequently the required number of iterations to achieve a certain error level (iteration complexity) is lower. When this is sufficiently lower, the overall time complexity of AltMin is lower than that of GD.  With a good enough initialization, AltMin algorithms are also easier to analyze theoretically.

%AltMin based algorithms been extensively used for various structured signal recovery problems and phase retrieval; these include sparse recovery, low rank matrix completion, low rank matrix sensing, phase retrieval and its sparse and low-rank extensions (cite), and low-rank column-wise matrix sensing. AltMin is also popular for clustering (k-means clustering), mixed linear regression, ??.

%LRMC is the one problem setting in which even the per-iteration cost of AltMin is lower than that of GD. The reason is both minimization problems in it are decoupled. LRMS is the other extreme example of a bilinear problem where both mins are expensive. But the algorithm analysis is simpler and there are no parameters to set.

\section{Non-Linear Least Squares (NLLS)} %If the minimization over $\Zb$ is a LS problem, the
The Non-Linear Least Squares (NLLS) approach was originally developed within the telecommunications literature for frequency estimation and related problems (all low-dimensional problems) \cite{kay_book}. This also splits $\Z$ into two blocks, $\Z = \{\Za,\Zb\}$ and solves problems in which, for any value of $\Za$, the minimization over $\Zb$ is an over-determined least squares (LS) problem.
%NLLS obtains a closed form expression for $\Zb$ in terms of $\Za$, denote it by $\hat\Zb(\Za)$.
For such problems, the NLLS approach substitutes a closed form expression for $\hat\Zb$ in terms of $\Za$ into the original cost function. It then minimizes the new cost function over $\Za$ using GD or one of its modifications such as the Newton method.
% followed by minimizing the new cost function (which now only depends on $\Za$). %over $\Za$ using GD or other GD modifications.
To be precise, it uses GD (or Newton's method or any solver) to solve
\[
\min_{\Za} f(\Za,\hat\Zb(\Za)) %\text{ where  } \hat\Zb(\Za) = \argmin_{\Zb} f(\Za,\Zb)
\]
where $\hat\Zb(\Za) = \argmin_{\Zb} f(\Za,\Zb)$  is the closed form for the LS solution for $\Zb$ in terms of $\Za$.

% \tilde{f}(\Za):= \tilde{f}(\Za):=

NLLS is easy to use for problems for which it is easy to compute the gradient of  $\tilde{f}(\Za):= f(\Za,\hat\Zb(\Za))$ w.r.t. $\Za$. However, for problems like LRCS, the new cost function is $\tilde{f}(\U):= \sum_k \|\y_k - \A_k \U (\A_k \U)^\dag \y_k\|^2$, with $\M^\dag:= (\M^\top \M)^{-1} \M^\top$. The gradient of this new cost function does not have a simple expression. %Even if one could compute the gradient,
%or is very expensive to compute. %Both LRCS and LRMC are instances where this approach can be used, but is not used because the gradient computation is messy.
Moreover, it is expensive to compute, and it makes the algorithm too complicated to analyze theoretically. To our knowledge, guarantees do not exist for NLLS.
%, because it is not clear how to bound the norm of the gradient. To our knowledge, finite sample guarantees do not exist for NLLS even for problems for which it is practically used. even if the gradient can be computed (one can approximate it numerically for example),

%AltGDmin or for AltMin, GD, NLLS or
\section{Algorithm Initialization}
\label{algoinit}
All iterative algorithms require an initialization. If the cost function is strongly convex, it has a unique minimizer. In this case, any initialization will work. For all other cases, the solution that an iterative algorithm converges to depends on the initialization. % Usually, one can show that the algorithm converges to {\em a} stationary point, which may be a local minimizer, closest to the initialization.

There are a few possible ways that an algorithm can be initialized. The most common approach is random initialization.
In this case, one runs the entire algorithm with multiple random initializations and stores the final cost function value for each. The output corresponding to the initialization that results in the smallest final cost is then chosen. 

In some other settings, some prior knowledge about $\Za$ is available and that is used as the initialization. For example, in multi-modal imaging, an approximate image estimate may be available from one source and that can serve as an initialization.

For a large number of structured signal (vector, matrix, or tensor) recovery problems and for phase retrieval problems,  one can come up with a carefully designed spectral initialization:  one computes the top, or top few, singular vectors of an appropriately defined matrix, that is such that the top singular vector(s) of its expected value are equal to, or close to, the unknown quantity of interest, or to a part of it.

% When a good initialization cannot be designed, the typical approach is to use a random initialization. In this case, usually one can only hope to prove that the algorithm converges to {\em a} local minimum.  (GD, block coordinate descent, AltMin, or AltGDmin) , such as low rank and or sparse recovery problems, phase retrieval and sparse or low rank PR

%This involves using the available data to compute a matrix whose top (or top few) singular vector(s) are a good approximation, in some appropriate distance metric, to the true desired value of $\Za$. %For example for low-rank recovery, this distance is the subspace distance.
%data-based approximation of a

\chapter{AltGDmin for Partly Decoupled Optimization Problems} \label{altgdmin}
We first precisely define a partly decoupled problem and then develop the AltGDmin framework.

\section{Partly-Decoupled Optimization: Precise Definition}\label{partlydecoup}
Consider an optimization problem $\arg\min_\Z  f(\Z)$. Suppose, as before, that $\Z$ can be split into two blocks $\Z = \{ \Za,\Zb\}$, such that optimization over one keeping the other fixed can be correctly solved (has a closed form or provably correct iterative solution).
We say this problem is partly decoupled (is decoupled for $\Zb$) if
\[
f(\Z; \calD) := f(\Za,\Zb; \calD) = \sum_{\ell=1}^\gamma f_\ell (\Za, (\Zb)_\ell; \calD_\ell)
\]
with $(\Zb)_\ell$ being disjoint subsets of the variable set $\Zb$, e.g., in case of the LR problems described earlier, these are different columns of the matrix $\B$.

Thus, partial decoupling implies that the minimization over $\Zb$ (keeping $\Za$ fixed at its previous value, denoted $\hat\Za$) can be solved by solving $\gamma$ smaller dimensional optimization problems, i.e.,
\[
\min_{\Zb} f(\hat\Za,\Zb; \calD) = \sum_{\ell=1}^\gamma \min_{(\Zb)_\ell} f_\ell (\hat\Za, (\Zb)_\ell; \calD_\ell)
\]
The computation cost of most optimization problems is more than linear, and hence, the $\gamma$ smaller dimensional problems are quicker to solve, than one problem that jointly optimizes over all of $\Zb$. Moreover, notice that the minimization over $(\Zb)_\ell$ only depends on the data subset $\calD_\ell$. Thus, if all of $\calD_\ell$ is available at a node, then the minimization over $(\Zb)_\ell$ can be solved locally at the node itself. %This is what provides a communication efficiency advantage also.

%: $\min_{(\Zb)_\ell} f_\ell (\hat\Za, (\Zb)_\ell; \calD_\ell)$,
\begin{remark}
It is possible that there are partly decoupled optimization problems for which the cost function is not just a sum of simpler cost functions, but is some other composite function. We attempt to define this most general case in Appendix \ref{partly_decoup_appendix}
\end{remark}

\section{Alternating GD and Minimization (AltGDmin)}
For partly decoupled problems, in recent work \cite{lrpr_gdmin},  we introduced the following  {\em Alternating GD and Minimization (AltGDmin)}  algorithmic framework. % as faster or more communication-efficient alternative to both AltMin and GD.
\bi
\item Initialize $\Za$ to $\hat\Za$. Approaches discussed in Sec. \ref{algoinit} can be used.
\item Alternate between the following two steps $T$ times, or until a stopping criterion is reached.
\ben
\item Update $\Zb$ by minimization: $\hat\Zb \leftarrow \min_{\Zb} f(\hat\Za, \Zb)$.
 Because of the decoupling, this simplifies to
 \[
(\hat\Zb)_\ell \leftarrow \argmin_{(\Zb)_\ell} f_\ell (\hat\Za, (\Zb)_\ell) \forall \ell \in [\gamma]
\]
\ben
\item Federated setting: Each node $\ell$ solves the above problem locally. No data exchange needed.
\een

 \item Update $\Za$ by GD:
 \[
 \hat\Za  \leftarrow \hat\Z_a -  \eta  \sum_{\ell=1}^\gamma \nabla_{\Za} f_\ell (\Za, (\Zb)_\ell; \calD_\ell)
\]
 %with $\nabla_{\Za} f =  \sum_{\ell=1}^\gamma \nabla_{\Za} f_\ell (\Za, (\Zb)_\ell; \calD_\ell)$.
 Here $\eta$ is the GD step size.
\ben
\item Federated setting: Each node $\ell$ computes the partial gradient $\nabla_{\Za} f_\ell (\hat\Za, (\hat\Zb)_\ell; \calD_\ell)$ and sends it to the center, which computes $\nabla f = \sum_\ell \nabla f_\ell$ and the GD update.
\een

\een

\ei

%where both minimizations of AltMin have efficient and closed form solutions (e.g., in case of LRMC)
\bfpara{Time complexity per iteration: centralized}
%Except in a few special cases,
%If the minimization w.r.t. $\Za$ is not decoupled
The time cost of gradient computation w.r.t. $\Za$ is much lower than that of solving a full minimization w.r.t. it.  In addition, if the time cost of solving $\min_{\Zb} f(\hat\Za, \Zb)$ is comparable to that of computing the gradient w.r.t. $\Za$, then, per iteration, AltGDmin is as fast as GD, and much faster than AltMin. This is the case for the LRCS problem, for example.

%In these cases, per-iteration, AltGDmin would be much faster than AltMin.
%An exception is LRMC which is decoupled w.r.t. both $\U$ and $\B$; and hence, for it, both have the same per-iteration cost.
%Typically GD algorithms are the fastest solutions. %This is true for all the structured signal/matrix recovery problems, e.g., for many of the linear LR problems described below the gradient computation cost is only $r$-times linear-time.
 %$\nabla_{\Za} f(\hat\Za, \hat\Zb) )$, then one can argue that AltGDmin

\bfpara{Communication complexity}
If the data is federated as assumed earlier (data subset $\calD_\ell$ is at node $\ell$), then  one can develop an efficient distributed federated implementation that is also more communication-efficient per iteration than AltMin, and comparable to GD. Node $\ell$ updates $(\Zb)_\ell$ locally and computes its partial gradient which it shares with the center.
%Suppose that the data subset $\Y_\ell$ is available at node $\ell$. Then the minimization can be solved locally at each node (parallel processed). The gradient w.r.t. $\Za$ can also be computed in a parallel fashion since $\nabla_{\Za} f(\hat\Za, \hat\Zb) = \sum_{\ell \in [\gamma]} \nabla_{\Za} f( \hat\Za, (\hat\Zb)_\ell )$. Node $\ell$ computes the $\ell$-th term of this and sends to the center. The communication cost per node per iteration is only equal to the dimension of $\Za$.%, e.g., see the table in Fig. \ref{table_fig_compare}
%
The center needs to only sum these, implement GD (just a subtraction), and (if needed) process the final output, e.g., orthonormalize the columns of $\Za=\U$ in case of LR recovery problems. This step is quick, of order $nr^2$ since $\U$ is an $n \times r$ matrix.

%\bfpara{Privacy: federated setting} The above decoupling property also implies data privacy in a federated setting. The central server never sees the entire raw data $\Y$ and it cannot estimate $\Zb$. Thus it cannot recover $\Z$.

\bfpara{Iteration Complexity and Sample Complexity}
It is often possible to also prove that the AltGDmin iteration complexity is only slightly worse than that of AltMin and much better than that of GD; and this is true under sample complexity lower bounds that comparable to what AltMin or GD need. The reason is the minimization over $\Zb$ in each iteration helps ensure sufficient error decay with iteration, even with using a constant GD step size.
For example, we have proved this for LRCS, LRPR, and LRMC; see Sec. \ref{guarantees} and Table \ref{table_compare}. This claim treats the matrix condition number as a numerical constant. We should mention here  that, the GD algorithm for $\U,\B$, referred to as Factorized GD (FactGD), does not provably converge  for LRCS or LRPR; the reasons are explained in Sec. \ref{guarantees}. Hence we do not have a bound on its iteration complexity. FactGD does converge for LRMC but its iteration complexity is $r$ times worse than that of AltGDmin or AltMin.

%
%, making it overall much more efficient time or communication wise.
%For many problems, the AltGDmin iteration complexity will be only slightly worse than that of AltMin and much better than that of GD.

%The iteration and sample complexity bounds need to be derived for each prl
%Consequently, the total time cost of the entire AltGDmin algorithm will typically be lower than that for AltMin.
%We explain how to obtain iteration and sample complexity bounds in Sec. \ref{altgdmin_theory}.

\bfpara{Overall Computation and Communication Complexity and Sample Complexity}
As explained above, for many partly decoupled problems, one can prove that AltGDmin is as fast per iteration as GD, while having iteration complexity that is almost as good as that of AltMin with a  sample complexity bound that is comparable to that of AltMin. This makes it one of the fastest algorithms in terms of total time complexity. In terms of communication cost, its per iteration cost is usually comparable to that of GD, while its iteration complexity is better, making it the most communication efficient.
%https://cymath.iastate.edu/math-for-all/

%For problems for which AltGDmin is much faster per iteration than AltMin but has similar iteration complexity
As an example, for LRCS, AltGDmin is much faster and much more communication-efficient than AltMin. This is true both in terms of order-wise complexity and practically in numerical experiments.
For LRMC, which is partly decoupled for both $\Zb$ and for $\Za$, all of AltGDmin, AltMin and GD have similar order-wise time complexity. However communication-complexity of AltGDmin is the best. Consequently, in numerical simulations on federated AWS nodes, AltGDmin is overall the fastest algorithm for large problems; see \cite{lrmc_gdmin}.

\bfpara{Non-differentiable cost functions} Another useful feature of AltGDmin is that it even applies for settings for which the cost function is not differentiable w.r.t. the decoupled set of variables. Some examples include clustering  and unlabeled/shuffled sensing. We describe these in the open questions section, Sec. \ref{examples_not_fully_diff}. % and maximum linear regression problems described below in Sec. \ref{examples_not_fully_diff}. %all the variables. For example, this holds in case for the data clustering problem discussed below. For such problems, AltMin also applies, but is slower. But, PGD and AltGD cannot be used.

\bfpara{Byzantine-Resilient or Straggler-Resilient Modifications} Federated algorithms are often vulnerable to attacks by adversaries. One of the most difficult set of attacks to deal with is Byzantine attacks. %Since GD is the most commonly used algorithms, m
Because AltGDmin involves exactly one round of partial gradients exchange per iteration, designing Byzantine resilient modifications for AltGDmin is easy to do and the resulting algorithm retains its efficiency properties. In recent work \cite{altgdmin_icml}, we developed a Byzantine-resilient AltGDmin solution for federated LRCS. We postpone the discussion of these modifications to a later review.
%of well studied approaches for standard GD. The algorithm needs careful and different analyses of course. See \cite{altgdmin_icml} for an example.

In distributed computing, resilience to straggling nodes is an important practical requirement. Straggler resilient GD using the ``gradient coding'' approach has been extensively studied \cite{tandon2017gradient,ramamoorthyMG24}. These approaches are directly applicable also for AltGDmin.

\part{AltGDmin for Partly-Decoupled Low Rank (LR) Recovery Problems: Algorithms \& Guarantees}

\chapter{AltGDmin for three LR matrix recovery problems}
\label{guarantees}

%?? LRMC -- remove details, give theorem statement and only give differences.

\section{Notation}
We use $\|.\|_F$ to denote the Frobenius norm,  $\|.\|$ without a subscript to denote the (induced) $l_2$ norm, $^\top$ to denote matrix or vector transpose, and $\M^\dag := (\M^\top \M)^{-1} \M^\top$. For a tall matrix $\M$, $QR(\M)$ orthonormalizes $\M$. We use $diag(\bm{v})$ to create a diagonal matrix with entries given by entries of vector $\bm{v}$. 
% $\e_k$ to denote the $k$-th canonical basis vector ($k$-th column of $\I$), and $\M^\dag := (\M^\top \M)^{-1} \M^\top$. %The operator $orth(\U)$ orthonormalizes the columns of $\U$ (e.g., by QR decomposition).  (often called the operator norm or spectral norm)
For two $n \times r$ matrices $\U_1, \U_2$ that have orthonormal columns, we use
\[
\SE_2(\U_1, \U_2) : = \|(\I  - \U_1 \U_1^\top) \U_2\| , \ \SE_F(\U_1, \U_2) : = \|(\I  - \U_1 \U_1^\top) \U_2\|_F
\]
as two measures of Subspace Distance (SD). Clearly, $\SE_F \le \sqrt{r} \SE_2$.

We reuse the letters $c,C$ to denote different numerical constants in each use with the convention that $c < 1$ and $C \ge 1$. %We use $\sum_k$ as a shortcut for the summation over $k=1$ to $q$ and $\sum_\ik$ for the summation over $i=1$ to $m$ and $k=1$ to $q$. {\em We use whp to refer to ``with high probability'' and this means that the claim holds with probability (w.p.) at least $ 1- n^{-10}$.}

%Let $\Xstar \svdeq \Ustar (\bm{\Sigma^*}) {\V^*}:= \Ustar \Bstar$, where $\Ustar \in \Re^{n \times r}$ and has orthonormal columns and $\Vstar \in \Re^{r \times q}$ with orthonormal rows. We use $\kappa= \sigmax/\sigmin$ to denote the condition number of the diagonal $r \times r$ matrix $\bf\Sigma^*$. Here $\sigmax,\sigmin$ to denote its largest, smallest singular values.  Also, we let $\Bstar:=\bm{\Sigma^*} {\V^*}$ so that $\Xstar = \Ustar \Bstar$.

\section{LRCS, LRPR, and LRMC Problems}
In all three problems, the goal is to recover an $n \times q$ rank-$r$ matrix $\Xstar =[\xstar_1, \xstar_2, \dots, \xstar_q]$, with $r \ll \min(n,q)$, from different types of under-sampled linear or element-wise nonlinear functions of it.
Let
	$
	\Xstar \svdeq \Ustar  \bSigma \Vstar: = \Ustar \Bstar
	$
	denote its reduced (rank $r$) SVD, and $\kappa:= \sigmax/\sigmin$ the condition number of $\bSigma$. We let $\Bstar:= \bSigma \Vstar$. %We need the following assumption.

\subsection{LRCS problem}
The goal is to recover an $n \times q$, rank-$r$, matrix $\Xstar =[\xstar_1, \xstar_2, \dots, \xstar_q]$ from independent linear projections of it, i.e., from % (sketches) of each of its $q$ columns, i.e. from%projections of each of its columns, i.e., from% (linear sketches)and its column span subspace, a rank-$r$ matrix $\Xstar \in \Re^{n \times q}$, where $r \ll \min(n,q)$,
\bea
\y_k := \A_k \xstar_k, \ k  \in [q]
\label{ykvec}
\eea
where each $\y_k$ is an $m$-length vector, with $m<n$,  $[q]:=\{1,2,\dots, q\}$, and the measurement/sketching matrices $\A_k$ are mutually independent and known. %The setting of interest is  low-rank (LR), $r \ll \min(n,q)$, and undersampled measurements, $m < n$.
For obtaining theoretical guarantees, each $\A_k$ is assumed to be random-Gaussian: each entry of it is independent and identically distributed (i.i.d.) standard Gaussian.

Since no measurement $\y_\ik$ is a global function  of the entire matrix, $\Xstar$, we need the following assumption to make our problem well-posed (allow for correct interpolation across columns). This assumption is a subset of the ``incoherence assumption'' introduced for correctly solving the LR matrix completion problem \cite{matcomp_candes,optspace,lowrank_altmin}.
	
\begin{assu}[$\mu$-incoherence of right singular vectors] \label{right_incoh}
Assume that $\|\bstar_k\|^2 \le \mu^2 r \sigmax^2 / q$ for a numerical constant $\mu$.% that is greater than one.
%Since $\|\xstar_k\| = \|\bstar_k\|^2$, this also implies  that $\|\xstar_k\|^2 \le \mu^2 r \sigmax^2 / q$
\end{assu}

\subsection{LRPR problem} In LRPR, which is a generalization of LRCS, the goal is to recover $\Xstar$ from undersampled phaseless linear projections of its columns, i.e., from $\z_k:=|\y_k|$, $k \in [q]$. Here $|.|$ of a vector takes the magnitude of each element of the vector. If the vector is real-valued, then this just means that  the sign is not measured. In case of Fourier ptychography, $\y_k$ are complex-valued and in that case, one takes the absolute value of each complex number entry.
LRPR also needs Assumption \ref{right_incoh} for the same reason.

\subsection{LRMC problem}
LRMC involves recovering an $n \times q$ rank-$r$ matrix $\Xstar =[\xstar_1, \xstar_2, \dots, \xstar_q]$ from a subset of its entries. Entry $j$ of column $k$, denoted $\Xstar_\jk$, is observed, independently of all other observations, with probability $p$. Let $\xi_\jk \iidsim \mathrm{Bernoulli}(p)$ for $j \in [n], k \in [q]$. Then, the set of observed entries, denoted by $\Omega$, is
\[
\Omega := \{(j,k): \xi_\jk = 1\}
\]
By setting the unobserved entries to zero, the observed data matrix $\Y \in \Re^{n \times q}$ can be defined as
\begin{equation}
\Y_\jk := \begin{cases}
\Xstar_\jk &   \text{ if } (j,k) \in \Omega, \\ % \textrm{with probability } p,\\
0 & \text{ otherwise}.
\end{cases}
\ \ \text{or, equivalently,} \ \
\Y := \Xstar_\Omega
\label{eq:msrmnts}
\end{equation}
Here and below, $\M_{\Omega}$ refers to the matrix $\M$ with all entries whose indices are not in the set $\Omega$ are zeroed out; while the rest of the entries remain unchanged.

We use $\Omega_k := \{j \in [n] \mid  \xi_\jk = 1\}$ to denote the set of indices of the observed entries in column $k$. To easily explain the AltGDmin algorithm idea, we define a diagonal 1-0 matrix  $\S_k \in \Re^{n \times n}$ as
\[
\S_k:= \mathrm{diag}([\xi_\jk, \ j \in [n]])
\]
With this, our goal is to learn $\Xstar$ from
\[
\y_k := \S_k \xstar_k, \ k \in [q]
\]
\begin{remark}
In the above, we let the matrix $\S_k$ be an $n \times n$ diagonal matrix with entries being 1 or 0, only for ease of notation. It contains a lot of zero entries. The expected number of nonzero rows (diagonal entries only) in $\S_k$ is $pn$. 
\end{remark}
We need the following assumption on the singular vectors of $\Xstar$; this is a way to guarantee that the rows and columns of $\Xstar$ are dense (non-sparse) \cite{matcomp_candes,optspace,lowrank_altmin}. This helps ensure that one can correctly interpolate (fill in) the missing entries even with observing only a few entries of each row or column.

\begin{assu}[$\mu$-incoherence of singular vectors of $\Xstar$]  %Let $\Xstar \svdeq \U^* \bf \Sigma^* {\V^*}$, then,  we
Assume  row norm bounds on $\Ustar$: $\max_{j \in [n]}\| \ustar{}^j\| \leq \mu \sqrt{r/n}$, and %{row, not column or column norm of $\V^\top$}
column norm  bounds on $\V^*$: $\max_{k \in [q]} \|  \v^*_k \| \leq \mu \sqrt{r/q}$. %, $\forall j \in [n], \, k \in [q]$.
Since $\Bstar= \bm{\Sigma^*} {\V^*}$, this implies that $\| \bstar_k \| \le \mu \sqrt{r/q} \sigmax$.% Here, and below, $\|.\|$ denotes the $\ell_2$ norm and, for a matrix $\A$, $\a_k$ denotes its $k$-th column while $\a^j$ denotes its $j$-th row transposed.
\label{incoh}
\end{assu}

\section{Federation}
%Federation means that (i) different subsets of the data are acquired at different distributed nodes; and (ii) all nodes can only communicate with the central node or “center”.
%Communication-efficiency is a key concern with all distributed algorithms, including federated ones. Privacy of the data is another concern in a federated setting. In this work, ``privacy'' means the following: the nodes' raw data cannot be shared with the center; and the algorithm should be such that the center cannot recover any of the columns of the LR matrix.

%The goal is to recover an $n \times q$ rank-$r$ matrix $\Xstar =[\xstar_1, \xstar_2, \dots, \xstar_q]$ from  $m$ independent linear projections (sketches) of each of its $q$ columns, with $m \ll n$.
We assume that there are a total of $\gamma$ nodes, with $\gamma \le q$ and each node has access to a different subset of the columns of the observed data matrix $\Y$ and the corresponding matrices $\A_k$ (or enough information to define them). This type of federation where the columns are distributed is often referred to as ``vertical federation''.
In case of LRCS, $\Y$ is $m \times q$. In case of LRMC, $\Y$ is $n \times q$ with a lot of zero entries. All nodes can only communicate with a central node or ``center''.

We use $\calS_\ell$ to denote the subset of columns of $\Y$ available at node $\ell$. The sets $\calS_\ell$ form a partition of $[q]$. i.e., they are mutually disjoint and $\cup_{\ell=1}^\gamma \calS_\ell = [q]$. Thus, the data at node $\ell$,
\[
\calD_\ell = \{\y_k, \A_k, \ k \in \calS_\ell\}
\]
To keep notation simple, we assume $q$ is a multiple of $\gamma$ and that $|\calS_\ell| = q/\gamma$. Our discussion of complexities assumes $\gamma \ll q$ and treats $\gamma$ as a numerical constant. Thus order $|\Omega|/\gamma$ is equal to order $|\Omega|$ with $|\Omega|\ge (n+q) r$ (the number of samples needs to be larger than the number of unknowns in rank $r$ matrix).

For LRCS with $\A_k$ being random Gaussian, the storage (or communication in case of distributed computing) required is significant, it is $mn q/\gamma$ per node. For LRMC, $\S_k$ is fully specified by just the set observed indices $\cup_{k \in \calS_\ell} \Omega_k$. This is much cheaper to store or transmit with a cost of only $|\Omega|/\gamma$. The same is true for the LRCS problem for the MRI application where $\A_k$ is a partial Fourier matrix; in this case only the observed frequency locations need to be stored or transmitted. %(a row sub-matrix of the 2D discrete Fourier transform operation written as a matrix). % For LRCS with random Gaussian entries, the en
%\Omega_{(\ell:=

\section{AltGDmin for LRCS: algorithm and guarantees}\label{altgdmin_lrcs}

AltGDmin for LRCS was introduced and studied in parallel works \cite{lrpr_gdmin,collins2021exploiting,netrapalli} and follow-up work \cite{lrpr_gdmin_2}. In \cite{lrpr_gdmin,lrpr_gdmin_2}, we referred to the problem as LR column-wise compressive sensing (LRCS), while \cite{collins2021exploiting,netrapalli} referred to the same problem as multi-task linear representation learning. The initialization introduced in \cite{lrpr_gdmin,lrpr_gdmin_2} is  the best one (needs fewer samples for a certain accuracy level). The best sample complexity guarantee for AltGDmin is the one proved in our recent work \cite{lrpr_gdmin_2}. %We later learned that the algorithm itself was first developed in  \cite{cvpr-paper-cited-by-netrapalli} for a different application.

\subsection{AltGDmin-LRCS Algorithm}
We first summarize the algorithm and then explain each step. Our development follows \cite{lrpr_gdmin,lrpr_gdmin_2}.
AltGDmin for the LRCS problem involves minimizing
$f(\U,\B): = \sum_{k=1}^q \|\y_k - \U \b_k\|^2$ over $\U,\B$. Clearly, this is  decoupled for columns of $\B$ (with holding $\U$ fixed). Thus, we use  $\Za \equiv \U, \ \Zb \equiv \B$. It proceeds as follows.
\ben
\item {\em Spectral initialization:} %Initialize $\U$. %(explained below). %using a truncated spectral initialization explained below.
We initialize $\U$ by computing the top $r$ singular vectors of the following matrix
\[
\X_0 := \sum_k \A_k^\top \y_{k,trnc} \e_k^\top, \ \y_{k,trnc}:= \mathrm{trunc}(\y_k,\alpha) % (\y_k \circ \indic_{|\y_k| \le  \sqrt\alpha})
\]
Here $ \alpha:= \tilde{C} \sum_k \|\y_k\|^2 / mq$ with $\tilde{C}:=9\kappa^2\mu^2$, and the function $\mathrm{trunc}$ truncates (zeroes out) all entries of the vector $\y_k$ with magnitude greater than $\sqrt\alpha$, i.e.,  for all $j\in [n]$, $\mathrm{trunc}(\y,\alpha)_j = (\y)_j  \indic_{|\y_j| \le  \sqrt\alpha}$, with $\indic$ being the indicator function.

%Here, the notation $\indic_{\bm{z} \le \alpha}$ returns a 1-0 vector with 1 where $\bm{z}_j < \alpha$ and zero everywhere else, and $\bm{z}_1 \circ \bm{z}_2$ is the Hadamard product.% (.* operation in MATLAB).

\item At each iteration, update $\B$ and $\U$ as follows:
\ben
%for the cost function $f(\U,\B)$
\item {\em Minimization for $\B$:} keeping $\U$ fixed, update $\B$ by solving $\min_{\B} f(\U, \B)$. Due to the form of the LRCS model, this minimization decouples across columns, making it a cheap least squares problem of recovering $q$ different $r$ length vectors. It is solved as $\b_k = (\A_k \U)^\dag \y_k$ for each $k \in [q]$. %Here $\M^\dagger:= (\M^\top \M)^{-1} \M^\top$.

\item {\em GD for $\U$:} keeping $\B$ fixed, update $\U$ by a GD step, followed by orthonormalizing its columns: %projecting the output onto the space of matrices with orthonormal columns:
    $\U^+  = QR(\U - \eta \nabla_{\U} f(\U,\B))$. Here $QR(.)$ orthonormalizes the columns of its input.
\een
\een

\bfpara{Computation cost}
%AltGDmin provably converges, and that (this statement treats $\kappa$ as a numerical constant)
The use of minimization to update $\B$ at each iteration is what helps ensure that we can show exponential error decay with a constant step size. At the same time, due to the column-wise decoupled nature of LRCS, the time complexity for this step is only as much as that of computing one gradient w.r.t. $\U$. Both steps need time\footnote{The LS step time is $\max(q \cdot mnr, q \cdot mr^2)=mqnr$ (maximum of the time  needed for computing $\A_k \U$ for all $k$, and that for obtaining $\b_k$ for all $k$)  while the GD step time is $\max( q \cdot mnr, nr^2) =mqnr$ (maximum of the time needed for computing the gradient w.r.t. $\U$, and time for the QR step).}of order $mqnr$. This is only $r$ times more than ``linear time'' (time needed to read the algorithm inputs, here $\y_k,\A_k$'s). To our knowledge, $r$-times linear-time is the best known time complexity for any algorithm for any LR matrix recovery problem. Moreover, due to the use of the $\X=\U\B$ factorization, AltGDmin is also communication-efficient. Each node needs to only send $nr$ scalars (gradients w.r.t $\U$) at each iteration.%
 %$r$ times linear-time

\bfpara{Understanding the Initialization step}
To understand the initialization step, note the following. It can be shown that $\E[\X_0] = \Xstar \D(\alpha)$ where $\D(\alpha)$ is a diagonal $q \times q $ matrix with $\sigmamin(\D) \ge 0.9$ with high probability (w.h.p.) \cite{lrpr_gdmin_2,lrpr_gdmin}. Thus, $\E[\X_0]$ is a rank $r$ matrix with column-span equal to that of $\Ustar$ (or $\Xstar$). Furthermore, it is easy to see that
\[
\X_0 = \sum_{k=1}^q \sum_{i=1}^m \a_\ik (\a_\ik^\top \xstar_k) \indic_{ (\a_\ik^\top \xstar_k)^2 \le \alpha }
\]
Using concentration bounds and linear algebra results\footnote{Using sub-exponential Bernstein inequality to lower and upper bound $\alpha$; and using the sub-Gaussian Hoeffding inequality and an easy epsilon-net argument  \cite{versh_book} to bound $\|\X_0 - \E[\X_0]\|$, one can argue that, w.h.p., $\X_0$ is close to its expected value if $mq$ is large enough. This, along with using the Wedin $\sin \theta$ theeorem \cite{spectral_init_review}, and lower bounding the smallest entry of $\D(\alpha)$, helps bound subspace distance (SD) between $\U_0$ and $\Ustar$.}, it can be shown that, w.h.p., $\X_0$ is a good approximation of its expected value and hence, in terms of subspace distance, $\U_0$ is a good approximation of $\Ustar$ (column span of $\Xstar$).

Sample-splitting is assumed, i.e., each new update of $\U$ and $\B$ uses a new independent set of measurements and measurement matrices, $\y_k, \A_k$.

\subsection{Federated implementation}
%
%%node $\tl$ stores the sketches $\y_k$ for all the signals (columns) that it contains. is a signal at a geographically distributed node $k$ and there are a total of $q$ nodes.  (except those for Byzantine-robustness)
%%
%Privacy constraints dictate that we cannot share the $\y_k$s with the central server; although summaries computed using the $\y_k$s can be shared at each algorithm iteration. %The center can use these to compute a global aggregate that can, in turn, be shared with all the nodes. The corresponding matrix column estimates $\xhat_k$ should also be updated locally at the node itself.
%%
%This will be done as follows.
%
Consider the GDmin steps. Update of $\b_k$s and $\x_k$s is done locally at the node that stores the corresponding $\y_k$.  For gradient w.r.t. $\U$ computation, the partial sums over $k \in \mathcal{S}_\tl$ are computed at node $\tl$ and transmitted to the center which adds all the partial sums to obtain $\nabla_\U f(\U,\B)$. GD step and QR are done at the center. The updated $\U$ is then broadcast to all the nodes for use in the next iteration.
%The three steps inside the ``for $k=1$ to $q$'' loop are implemented locally at the federated node that stores the corresponding measurement vector $\y_k$.
%Node $\tl$ also computes the sum of its $q_\tl$ individual gradients $\nabla_U f_k$ and sends this partial sum, which is an $n \times r$ matrix, to the center. The center sums the $L$ partial sums and uses this to compute $\Uhat^+$. It then obtains $\U^+$ from  $\Uhat^+$ by QR decomposition.
%%The node then computes the sum of its individual gradients $\nabla_U f_k$ and sends this partial sum, which is an $n \times r$ matrix, to the center. The center sums the $L$ partial sums computes and $\Uhat^+$, and then $\U_+$ from  $\Uhat^+$ by QR decomposition.
% %
%%ach and the node needs to only transmit $nr$ scalars (entries of $\nabla_\U f_k$) at each iteration.
%%Since most of the work is done at the nodes in parallel, this results in a roughly $(q/\max_\tl q_\tl) $ times faster algorithm compared to its centralized version.
%Recall that $q_\ell$ is the number of $\y_k$s at node $\ell$.
The per node time complexity is thus $mnr q_\tl$ at each iteration. The center only performs additions and a QR decomposition, which is an order $nr^2$ operation, in each iteration. The communication cost is order $nr$ per node per iteration. %Thus, the time complexity of the federated solution is only $mnr(\max_\tl q_\tl)T$ per node. Here $q_\tl = q/\gamma$. The communi
%This makes it $q/(\max_\tl q_\tl)$ times faster than the centralized approach. % instead of $mnrq T$ for the centralized solution.
%
%Thus, if $q_\tl = q/L$ for all $\tl$, then the federated solution is $L$-times faster. Here $T$ is the total number of iterations.

The initialization step can be federated by using the Power Method \cite{golub89,npm_hardt} to compute the top $r$ eigenvectors of $\Xhat_0 \Xhat_0{}^\top$. Power method starts with a random initialization and runs the iteration $\Uhat_0 \leftarrow QR(\Xhat_0 \Xhat_0{}^\top \Uhat_0)$.
 %Starting with a random initialization $(\U_0)_0$, for each $\tau = 1,2,\dots C$, we compute $(\Uhat_0)_\tau \leftarrow \sum_k \frac{1}{m}( (\sum_i \a_\ik \y_\ik \indic ) ( \sum_i \a_\ik \y_\ik \indic )^\top \U_0)_{\tau-1}$. Here $\indic$ is short for the indicator function used in \eqref{newinit}. The partial sums for this summation are computed locally at the nodes, followed by aggregation at the center. The center also computes $(\U_0)_\tau$ by QR, on $(\Uhat_0)_\tau$ and shares it with the nodes. %This process is repeated for $\tau_{max} = C$ iterations.%$(\Uhat_0)_\tau \qreq (\U_0)_\tau (\R_0)_\tau$ at the center and share it with the nodes.
%
%\subsection{Communication cost and privacy}
 Any power method guarantee, e.g., \cite{npm_hardt} can be used to guarantee that its output is within a subspace distance $\delta_0$ to the span of the top $r$ singular vectors of $\X_0$ within order $\log(1/\delta_0)$ iterations. %eigenvectors of $\Xhat_0 \Xhat_0{}^\top$ after a sufficient number of iterations.
The communication complexity is thus just $nr$ per node per iteration. The number of iterations needed is only order $\log r$ because $\U_0$ only needs to be order $1/r$ accurate. % (need to share $\sum_{k \in \calS_\ell} (\X_0)_\ell (\X_0)_\ell^\top \U$).

\bfpara{Communication cost} The total communication cost is order $\max(nr \log r, nr \cdot T)$ where $T$ is the total number AltGDmin iterations needed to achieve $\eps$ accuracy. We show below that $T = C \kappa^2 \log(1/\eps)$ suffices. For accurate solutions $\eps < \exp(-r)$ and hence the total communication cost is order $nr T = \kappa^2 nr \log(1/\eps)$.

\bfpara{Privacy} Observe from above that the information shared with the center is not sufficient to recover $\Xstar$ centrally. It is only sufficient to estimate $\Span(\Ustar)$.
The recovery of the columns of $\B$, $\tb_k$, is done locally at the node where the corresponding $\y_k$ is stored, thus ensuring privacy. %Thus, with the above algorithm, the signals $\xstar_k = \Ustar \tb_k$ measured at node $\tl$ can be recovered only at node $\tl$, thus ensuring privacy.
%The center and all nodes have access to the estimated column span, but different column sub-matrices of the matrix are recovered at the respective nodes where those columns were measured.

\subsection{Theoretical guarantees}
We provide below the best known guarantee for LRCS; this is taken from \cite{lrpr_gdmin_2}. We state the noise-free case result here for simplicity.
Let $m_0$ denote the total number of samples per column needed for initialization and let $m_1$ denote this number for each GDmin iteration. Then, the total sample complexity per column is $m = m_0  + m_1 T$.
Our guarantee given next provides the required minimum value of $m$.
\begin{theorem} [AltGDmin-LRCS \cite{lrpr_gdmin_2}]
Assume that Assumption \ref{right_incoh} holds. Set $\eta = 0.4 / m \sigmax^2 $ and $T = C \kappa^2 \log( 1 /\eps)$.
If
\[
m q \ge C \kappa^4 \mu^2 (n+q) r ( \kappa^4 r + \log (1/\eps) ) %\frac{1}{\eps}
\]
and $m \ge C \max(\log n, \log q, r) \log (1/\eps)$, then, with probability (w.p.) at least $1 - n^{-10}$,% for all $k \in [q]$,
\[
\SE_2(\U,\Ustar) \le \eps \text{ and } \|\x_k - \xstar_k\| \le \eps \|\xstar_k\| \ \text{for all $k \in [q]$}. %= (1 - 0.6 c_\eta / \kappa^2 )^t  \frac{0.02}{\sqrt{r} \kappa^2}
\]
%$\|\x_k - \xstar_k\| \le \eps \|\xstar_k\|$. %All bounds of Theorem \ref{Blemma} also hold with $\delta_t = \eps$.
%
The time complexity is $mqnr \cdot T =   mqnr \cdot \kappa^2 \log( 1 /\eps)$. The communication complexity is $nr\cdot T = nr  \cdot \kappa^2 \log( 1 /\eps)$ per node.% per iteration.
\label{main_res}
\end{theorem}
%; the latter two we learned about very recently since they were published in different venues and with a completely different paper title and application
%AltGDmin was introduced and studied in parallel works \cite{lrpr_gdmin,collins2021exploiting,netrapalli}. The best guarantee for it is the one stated above and this appeared in our recent work \cite{lrpr_gdmin_2}. We later learned that the algorithm itself was also introduced in \cite{cvpr-paper-cited-by-netrapalli}.

\begin{remark}\label{gen_eta}
More generally, for any $\eta = c_\eta /(m \sigmax^2)$ with $c_\eta \le 0.8$, one can show that $\SE_2(\U^+, \Ustar) \le (1 - \frac{c c_\eta}{\kappa^2}) \SE_2(\U, \Ustar)$. In short, the above result applies, with only changes to numerical constants.
\end{remark}

%?? later replace by noisy case

%The GD algorithm here refers to use of GD for the variables $\U,\B$ as done in \cite{rpca_gd} for LR matrix completion and robust PCA.
\subsection{Discussion}
Existing approaches for LRCS include the AltMin solution studied in our work on LR phase retrieval (LRCS is a special case of LRPR) \cite{lrpr_best,lrpr_it,lrpr_icml} and the convex relaxation studied in \cite{lee2019neurips}. For reasons explained in detail in \cite{lrpr_gdmin}, for LRCS, there does not seem to be a way to guarantee convergence of either of the GD algorithms that have been studied for LRMC and robust PCA -- Factorized GD (FactGD) and Projected GD  (PGD) \cite{rpca_gd,rmc_gd}.
Factorized GD is GD for $\U,\B$ for the cost function $f(\U,\B) + \lambda \|\U^\top \U - \B \B^\top \|_F$ (the second term is a norm balancing term). PGD is GD for $\X$, with each GD step followed by projection onto the set of rank $r$ matrices (by SVD). The reason is: to show convergence, we need to bound the norm of the gradient w.r.t. $\U$ or $\X$ of $f(\U,\B)$ or $f(\X)$, and show that it decays with iterations, under the desired roughly $nr^2$ sample complexity. To obtain this bound, one needs a tight bound on the column-wise recovery error $\max_k \|\x_k - \xstar_k\|$. This is not possible to get for either FactGD or PGD because, for both, the estimates of $\x_k$ are coupled (PGD) or coupled given $\U$ (FactGD)
\footnote{
Consider FactGD. The gradient w.r.t $\U$ of $f(\U,\B)$ is $\sum_{k=1}^q \A_k^\top \A_k (\xstar -  \x_k) \b_k^\top$. To bound the norm of its deviation from its expected value, we need a small enough bound on the sub-exponential norm of each summand \cite[Chap 2]{versh_book}; this requires a small enough bound on the column-wise error $\max_k \|\x_k - \xstar_k\|$, here $\x_k = \U \b_k$. It is not possible to get a tight bound on this quantity for FactGD because its estimates of the different $\b_k$s are coupled, due to the gradient term coming from the second norm balancing term. Consider PGD. The gradient w.r.t. $\X$ is $\sum_{k=1}^q \A_k^\top \A_k (\xstar -  \x_k)$; bounding it again requires a bound on $\max_k \|\x_k - \xstar_k\|$. The estimates $\x_k$ are coupled for different $k$ because of the rank $r$ projection step.
}.%

AltGDmin is the fastest and most communication-efficient compared to both of AltMin and convex relaxtion. Convex relaxation (mixed norm minimization) is known to be much slower.  Its time complexity is not discussed in the paper, however, it is well known that solvers for convex programs are much slower when compared to direct iterative algorithms: they either require number of iterations proportional to $1/\sqrt{\eps}$ or the per-iteration cost has cubic dependence on the problem size, here $(nr)^3$.
AltMin is also slower than AltGDmin, both in terms of theoretical complexity and experimentally, because, for updating both $\U$ and $\B$, it requires solving a minimization problem keeping the other variable fixed. The minimization step for $\U$ is the slow one. %This is  because the minimization w.r.t. $\U$ depends on all the measurements (depends on $\y_k$ for all $k \in [q]$).
The same is true for its communication cost. The minimization step for updating $\U$ needs to use multiple GD iterations instead of just one in case of AltGDmin, or it needs to share matrices of size $nr \times nr$ (even more inefficient). This is why both the time and communication cost of AltMin depend on $\log^2(1/\eps)$ instead of just $\log(1/\eps)$ is case of AltGDmin.

% Detailed comparisons are provided in cite-ankit-spl.
%
%?? copy-paste rest of discussion from ??
%%We will compare and contrast the three different versions of AltGDmin developed in the three first works cited above (most of the algorithm is the same, but the initialization is different and consequently the sample complexity requirement is different), with GD, AltMin and with the nonlinear-LS approach \cite{lrpr_best,collins2021exploiting,lrpr_gdmin,lrpr_gdmin_2,kay_book}. Time, communication, and sample complexity, and privacy,  will be compared.
%
%Mention federated setting. Maybe also mention for distributed setting.
%
%Two tables: one for LRCS and one for LRMC. time per iter, communic per iter, iteration comp, sample comp, private?
%
%Algo's: altgdmin-best-init (our two papers), altgdmin-bad-init, altmin (copy guarantee from LRPR paper), GD (factored GD), nonlinear-LS
%
%(state best guarantees for each algo modification).

\renewcommand{\c}{\bm{c}}
\section{AltGDmin for LRPR: algorithm and guarantees}
To explain the ideas simply here, we consider the real-valued case. This means we do need to worry about complex conjugation.
The phaseless measurements $\z_k$ can be rewritten as
\[
\z_k = diag( \c_k^*) \y_k = diag(\c_k^*) \A_k \Ustar \bstar_k
\]
where $\c_k^*$ is a vector of signs/phases of $\y_k$ and $diag$ converts this into a diagonal matrix. Thus, the cost function to minimize now becomes
\[
f(\U,\B,\{\c_k, k \in [q]\}): = \sum_{k=1}^q \|\z_k - diag(\c_k) \A_k \U \b_k\|^2
\]
Clearly this problem is again decoupled with $\Za = \U$ and $\Zb = \{\c_k, \b_k, k \in [q]\}$. Notice also that when $\U$ is fixed, solving for $\{\b_k, \c_k\}$ is a standard $r$-dimensional PR problem with many fast and provably correct solutions, e.g., \cite{wf,twf}. The cost of $r$-dimensional PR is order $mr \log(1/\eps)$ and the cost of computing $\A_k \U$ is $mnr$,  per column. Thus, the total cost of standard PR for all columns is just $q \max(mnr, mr \log(1/\eps))$. Typically the first term dominates. Gradient computation cost is still $mnrq$. Thus the total cost of AltGDmin iterations is $mqnr \cdot T$ with $T$ bounded in the result below.

The initialization in this case also needs to be different. We initialize $\U$ by computing the top $r$ singular vectors of \cite{lrpr_best}
\[
\M = \sum_k \A_k^\top \z_{k,trunc} \z_{k,trunc}^\top \A_k
\]
with the truncation done exactly as explained above for LRCS (truncation only uses magnitudes of observations).
We can prove the following \cite{lrpr_gdmin}.
\begin{theorem}[AltGDmin-LRPR \cite{lrpr_gdmin}]
Assume that Assumption \ref{right_incoh} holds. Set $\eta = 0.4 / m \sigmax^2 $ and $T = C \kappa^2 \log( 1 /\eps)$.
If
\[
m q \ge C \kappa^6 \mu^2 (n+q) r^2 ( \kappa^4 r + \log (1/\eps) ) %\frac{1}{\eps}
\]
and $m \ge C \max(\log n, \log q, r) \log (1/\eps)$, then,  the conclusions of Theorem \ref{main_res} hold with $\SE_2$ replaced by $\SE_F$.
The time complexity is $mqnr \cdot T =   \max (mqnr, mq r \log(1/\eps)) \cdot \kappa^2 \log( 1 /\eps)$. The communication complexity is $nr\cdot T = nr  \cdot \kappa^2 \log( 1 /\eps)$ per node.% per iteration.
\label{main_res_pr}
\end{theorem}
\begin{remark}
We can use any  $\eta = c_\eta /(m \sigmax^2)$ with $c_\eta \le 0.8$, see Remark \ref{gen_eta}. %, one can show that $\SE_2(\U^+, \Ustar) \le (1 - \frac{c c_\eta}{\kappa^2}) \SE_2(\U, \Ustar)$. In short, the above result applies, with only changes to numerical constants.
\end{remark}

Notice that the only change in the above result compared to LRCS is an extra factor of $r$ in the sample complexity. This trend is well-known from other work on structured phase retrieval \cite{cai,lrpr_it,lrpr_best}.
The rest of the discussion is the same as in case of LRCS. AltGDmin is much faster than AltMin. FactGD or PGD do not provably converge for the same reasons. The proof strategy for this case involves interpreting the gradient w.r.t. $\U$ as a noisy version of the LRCS case. The overall idea for handling this case is provided in Sec. \ref{noisygrad_proof}.  %In fact, the overall proof approach for this case can be explained

%Thus, in this case the algorithm consists of AltGDmin alternating between updating $\U$ as done above

%In this there are three sets of unknowns, $\Ustar, \Bstar, \C^*$ where $\C^*$ denotes the phase vectors of all columns.

\section{AltGDmin for LRMC: algorithm and guarantees}
AltGDmin for LRMC was studied in \cite{lrmc_gdmin}. %We briefly summarize the new ideas over LRCS. %algorithm and guarantees next.
There are two differences between LRMC and LRCS. The first is that LRMC measurements are row-wise and column-wise local while those for LRCS are global functions of each column. This is why LRMC needs incoherence of left and right singular vectors of $\Xstar$, and needs to prove this for each estimate $\X=\U\B$ at each iteration. LRCS needs this only for right singular vectors. The second is that the measurements are bounded and this is why the initialization does not need a truncation step.

The goal is to minimize
\begin{equation}
\min_{\substack{\check\B,\, \check\U  \colon \check\U^\top \check\U = \I  }}  f(\check\U,\check\B), \ f(\check\U,\check\B):= \| (\Y - \check\U \check\B)_{\Omega}  \|_F^2 \label{eq:obj}
\end{equation}
As before, we impose the orthornormal columns constraint on $\check\U$ as one way to ensure that the norm  of $\U$  does not keep increasing or decreasing continuously with algorithm iterations, while that of $\B$ decreases or increases.

%It is not difficult to see that 
As explained earlier in \eqref{lrmc_cost}, this cost function is partly decoupled for $\B$ as well as for $\U$. This means that we could pick either of the two to serve as $\Zb$; the choice depends on how the data is federated. In fact, since the LRMC problem is symmetric w.r.t. rows and columns, one can always assume vertical federation as stated earlier and, if needed, transpose the matrices to satisfy the assumption.

Conceptually, the only difference for the AltGDmin algorithm in this case is in the initialization step. However, its efficient implementation requires some careful work. The analysis to derive the theoretical guarantees needs significant extra work as well. Most importantly, it requires showing incoherence of $\U$ at each iteration including the initialization. For the iterations, this can be proved; we explain the main in ideas in Sec. \ref{proof_ideas_lr} and \ref{proof_details}. For the initialization of $\U$, we need to ensure this by construction. We do this by adapting the idea of \cite{rpca_gd}. We first compute the top $r$ singular vectors of $\Y$; denote the matrix formed by these singular vectors by $\U_{00}$. We then project $\U_{00}$ onto the space of row incoherent matrices,  ${\calU}: = \{ \check\U: \|{\check\u}^j \| \le \mu \sqrt{r/n} \}$ to obtain $\Pi_{\calU}(\U_{00})$. We finally obtain $\U_0$ by orthonormalizing it by QR. Here,
\bea
[\Pi_{\calU}(\M)]^j = \m^j \cdot \min\left(1, \frac{\mu \sqrt{r/n} }{\| \m^j \|} \right), \ \forall  j \in [n]
%\Pi_{\calU}(\M) = \sum_j \e_j \left( \e_j^\top \M \cdot \frac{\min(\|\e_j^\top \M\|,\mu \sqrt{r/n}) }{\|\e_j^\top \M\|} \right)
\label{proj_incoh}
\eea
In words,  if a row of $\M$ has $\ell_2$ norm that is more than the threshold $\mu \sqrt{r/n}$, then one renormalizes the row so that its norm equals the threshold. If the norm is less than this threshold, then we do not change it. Clearly this is an order $nr$ time operation.
In summary,  $\U_0:= QR(\Pi_{\calU}(\U_{00}))$ with $\U_{00}$ being the top $r$ left singular vectors of $\Y$.
%The last step is computing the QR decomposition of $\Pi_{\calU}(\U_{00})$.  The entire initialization needs time of order $nr^2$.

The rest of the AltGDmin algorithm is conceptually similar to that for LRCS, we use $\Za=\U$ and $\Zb = \B$. However, its efficient implementation is very different and hence, so is its time and communication complexity. Briefly, the reason is that $\S_k$ is just a row selection matrix. Thus, for example, $\S_k \U$ is actually implemented by sub-selecting the rows of $\U$ and not by matrix multiplication. A lot of other steps use similar ideas for efficient implementation. We summarize the complexities in Table \ref{table_compare}. We can prove the following for AltGDmin-LRMC

\begin{theorem}[AltGDmin-LRMC \cite{lrmc_gdmin}]
\label{thrm}
%Let $\deltat := \SE_F(\U^{(t)},\U^*)$.
Pick an $\eps < 1$.
Assume that Assumption \ref{incoh} holds, and that, entries of $\Xstar$ are observed independently of other entries with probability $p$. Set $\eta = 0.5/(p \sigmax^2)$ and $T = C\kappa^2 \log(1/\eps)$.
If  $n q p > C  \kappa^6 \mu^2 \max(n,q) r^2 \log \max(n,q) \log({1}/{\epsilon})$, then, with probability (w.p.) at least $1 - 4T/\min(n,q)^3$,
\begin{equation}
\SE_F(\U^{(T)},\Ustar)  \leq \eps \text{ and } \|\X^{(T)}  - \Xstar\|_F \le \eps \|\Xstar\|.
\end{equation}
%(recall that $\X^{(T)} = \U^{(T)}\B^{(T)}$).
%
%Using the per iteration computation and communication costs from above, and the value of $T$ given here,
The total per-node computation complexity of federated AltGDmin is $C\kappa^2 \log(1/\eps) \cdot \max(n,|\Omega|)r^2 \cdot \frac{1}{\gamma} $ and its total per-node communication complexity is  $C\kappa^2 \log(1/\eps) \cdot nr$. %\min(n, (|\Omega| / \gamma))
% r \cdot \max(n, (|\Omega| / \gamma) )= nr$
\end{theorem}
Observe that $nq\cdot p = \E[|\Omega|]$, i.e., it is the expected value of the sample complexity. We often just use the phrase ``sample complexity'' when referring to it in our writing.% Treating $\kappa,\mu$ as numerical constants, the above result says that as long as we observe order $nr^2 \log q  \log (1/\eps)$ matrix entries, and we set the GD step size, $\eta$, and the total number of iterations, $T$, as stated, then  w.h.p., we can fill in the rest of the entries accurately: the normalized Frobenius norm of the error in this estimation is at most $\eps$. Also, we can estimate the column span of $\Xstar$ with $\eps$ accuracy. The number of iterations needed is order $\log(1/\eps)$.  We can achieve any $\eps$ accuracy as long as we increase the sample and iteration complexity accordingly.
% we set the step size $\eta$ for GD of $\U$ as stated,

\begin{remark}
More generally, we can use any $\eta =  c_\eta /(p \sigmax^2)$ with $c_\eta \le 0.8$; see Remark \ref{gen_eta}.
%for any $\eta = c_\eta /(p \sigmax^2)$ with $c_\eta \le 0.5$, one can show that $\SE_F(\U^+, \Ustar) \le (1 - \frac{c c_\eta}{\kappa^2}) \SE_F(\U, \Ustar)$. This is true for both LRCS and LRMC.  %. Then, if at each iteration, $nq \cdot p > C  \kappa^6 \mu^2 q r^2 \log q$, then  with probability (w.p.) at least $1 - 4/n^3$,
\end{remark}

%(GD for $\U,\B$ with a norm balancing term added to the cost function)
\subsection{Discussion}
Recall that for LRCS, AltGDmin is much faster than AltMin because the minimization step w.r.t. $\U$ is coupled and hence expensive.
However, in case of LRMC, the recovery problem is decoupled for both $\U$ and for $\B$. Consequently, AltMin is the fastest centralized solution and order-wise (ignoring dependence on $\kappa,\mu$), all of AltMin, AltGDmin and FactGD are equally fast. PGD is much slower. %Practically, both are a little slower than AltMin.
In a federated setting, when considering communication cost, AltGDmin is the most communication efficient compared with both AltMin and GD (FactGD). Compared with FactGD, AltGDmin iteration complexity is better by a factor of $r$. Compared with AltMin, its per-iteration cost is lower. It requires sharing just the gradient w.r.t. $\U$, which is at most $nr$ entries, in each iteration. AltMin, on the other hand, requires sharing all the observed entries and this has a communication cost of order $|\Omega| > nr$. The required $|\Omega|$ is order $nr^2$ at least (see sample complexity).  %Alternatively, one updates $\U$ using GD (private version of AltMin), but that then requires many more data exchanges of size $nr$.
We provide a summary of comparisons of guarantees for both LRCS and LRMC in Table \ref{table_compare}.

\begin{table*}[t!]
	\begin{center}
\resizebox{0.99\linewidth}{!}
		{
			\begin{tabular}{lllll} 
\toprule
				LRCS &  Computation  	& Communic.  & Sample  &  Resilient   \\
& Complexity &  Complexity &  Complexity & Modific \\
				 \midrule
				AltGDmin \cite{lrpr_gdmin_2}  & $m\frac{q}{\gamma} nr \cdot \log(1/\eps)$  &  $nr \log(1/\eps)$  & $nr \max(r,\log(1/\eps))$  & Efficient  \\
				&&&& \\
\hline
			GD (FactGD)   & $m\frac{q}{\gamma} nr \cdot T $ & $nr \cdot T$ & $\text{(cannot bound)}$ & Efficient   \\
				& $\text{(cannot bound $T$)}$  && & \\
\hline
				AltMin \cite{lrpr_best}   &   $m \frac{q}{\gamma} nr \cdot \log^2(1/\eps)$   &  $nr \log^2(1/\eps)$ & $nr^2\log(1/\eps)$ & Not Efficient   \\
				&&&& \\
\hline
				Convex \cite{lee2019neurips}  &  $mqnr \cdot \min(\frac{1}{\sqrt{\eps}}, n^3 r^3)$ &  &   $\frac{nr}{\eps^4}$   & Not Efficient \\
				%		%				ProjGD \cite{fastmc} & 0  & $\frac{|\Omega|}{\gamma} r^2$    & $ \sum_{k \in \calS_\ell} |\Omega_k|$ & $\frac{|\Omega|}{\gamma} $ & $\log (\frac{q}{\epsilon}) $ & no \\
(mixed norm min)			&&&  & \\
%				\hline
\bottomrule
%			\end{tabular}
\toprule
				LRPR &  Computation  	& Communic.  & Sample  &  Resilient   \\
& Complexity &  Complexity &  Complexity & Modific \\
				 \midrule
				AltGDmin \cite{lrpr_gdmin}  & $\max( m\frac{q}{\gamma} nr, m\frac{q}{\gamma}r \log(1/\eps))  \cdot \log(1/\eps)$  &  $nr \log(1/\eps)$  & $nr^2 \max(r,\log(1/\eps))$  & Efficient  \\
				&&&& \\
\hline
			GD (FactGD)   & $m\frac{q}{\gamma} nr \cdot T $ & $nr \cdot T$ & $\text{(cannot bound)}$ & Efficient   \\
				& $\text{(cannot bound $T$)}$  && & \\
\hline
				AltMin \cite{lrpr_best}   &   $m \frac{q}{\gamma} nr \cdot \log^2(1/\eps)$   &  $nr \log^2(1/\eps)$ & $nr^2\log(1/\eps)$ & Not  \\
				&&&& \\
\hline
%				Convex \cite{lee2019neurips}  &  $mqnr \cdot \min(\frac{1}{\sqrt{\eps}}, n^3 r^3)$ &  &   $\frac{nr}{\eps^4}$   & Not \\
%				%		%				ProjGD \cite{fastmc} & 0  & $\frac{|\Omega|}{\gamma} r^2$    & $ \sum_{k \in \calS_\ell} |\Omega_k|$ & $\frac{|\Omega|}{\gamma} $ & $\log (\frac{q}{\epsilon}) $ & no \\
%(mixed norm min)			&&&  & \\
%%				\hline
\bottomrule
%\\
%\vspace{0.2in}
%			\begin{tabular}{llll}
\toprule
%%				LRMC & Compute.  	& Communic.  & Sample     \\
			LRMC &  Computation  	& Communic.  & Sample  &  Resilient   \\
& Complexity &  Complexity &  Complexity & Modific \\
				\midrule
				AltGDmin \cite{lrmc_gdmin}  & $\frac{|\Omega|}{\gamma}r^2 \log (\frac{1}{\epsilon})$   & $ nr \log (\frac{1}{\eps})$  & $nr^2\log n \log(\frac{1}{\eps})$  & Efficient  \\
				&&& & \\
\hline
				GD (FactGD)   \cite{rpca_gd,lafferty_lrmc} & $ \frac{|\Omega|}{\gamma}r^2 \log (\frac{1}{\eps})$    & $ nr^2 \log (\frac{1}{\eps})$   &$nr^2\log n$   & Efficient    \\
				&&&  & \\
\hline
				AltMin  \cite{lowrank_altmin}  &  $ \frac{|\Omega|}{\gamma}r \log^2(\frac{1}{\eps})$   &  $nr \log^2(\frac{1}{\eps})$  & $ nr^{4.5}\log n\log(\frac{1}{\eps})$  & Not  \\
(use GD for updating $\U$)				&&&  & \\
\hline
%				AltMin  \cite{lowrank_altmin}  & $\frac{|\Omega|}{\gamma}r^2  \log (\frac{1}{\eps})$ & $ \frac{|\Omega|}{\gamma}  \log (\frac{1}{\eps})$  & $nr^{4.5}\log n \log(\frac{1}{\eps})$   \\
%(use closed form for updating $\U$)				&&& \\
%\hline
AltMin  \cite{lowrank_altmin,lrmc_gdmin}  & $\frac{|\Omega|}{\gamma}r^2  \log (\frac{1}{\eps})$ & $ \frac{|\Omega|}{\gamma}  \log (\frac{1}{\eps})$  & $nr^2 \log n \log(\frac{1}{\eps})$   & Not \\
(use closed form for updating $\U$)				&&& & \\
\hline
Convex \cite{matcomp_candes}  &  $|\Omega|r \cdot \min(\frac{1}{\sqrt{\eps}}, n^3 r^3)$ &  &   $n^{1.2} r \log^2 n $    & \\
(nuclear norm min)				&&&  & Not \\
				\bottomrule
			\end{tabular} %$\text{(unknown)}$
		}
		\vspace{-0.05in}
		\caption{\small\sl  Comparing AltGDmin with AltMin and GD (FactGD) for recovering an $n \times q$ rank $r$ matrix from a subset of $m$ linear projections of its columns (LRCS), $m$ phaseless linear projections (LRPR),  or from a subset of its entries, when each entry is observed with probability $p$ independent of all others (LRMC).
 $\gamma$ is the total number of federated nodes. $\Omega$ is the set of observed entries for LRMC.
			The table assumes $n \approx q$, $\gamma$ is a numerical constant, $\kappa,\mu$ are numerical constants, $\max(\log(1/\eps), r) = \log(1/\eps)$, and $|\Omega| \ge nr$ (necessary).  Here  $\text{Communic Comp} = T \cdot \max(\text{Communic.(node), Communic.(center)}) $. 	Similarly for the computation cost. }
%\cite{allrtnPpr} Comparisons for LRCS (top) and LRMC (bottom): 
		%
		\label{tab1}
		\label{table_compare}
	\end{center}
\end{table*}

\part{AltGDmin Analysis -- Overall Proof Technique and Details for LR Problems}

\chapter{General Proof Approach for any Problem} \label{proof_ideas} % -- Theoretical analysis of AltGDmin based algorithms}

%As noted earlier, this will need extra assumptions and many more new ideas. Also, it may not work for all problems in Problem Class \ref{probclass}.

%\section{Overall approach: top-level ideas} \label{proof_ideas_gen}
The following approach generalizes the ideas used for LRCS \cite{lrpr_gdmin_2} and the other LR problems.
Let $\mathrm{NormDist}_a$ be the relevant measure of normalized distance for $\Za$ and $\mathrm{NormDist}_b$ for $\Zb$. The distance metric used can be different for $\Za$ and $\Zb$, e.g., for the LR matrix recovery problems discussed above, we used the subspace distance for $\Za$, and normalized Euclidean norm distance for $\Zb$.
%, appropriately so it lies between zero and one. %e.g., if $\Za$ represents a low-dimensional subspace then this is the subspace distance (or SD fdivided by $\sqrt{r}$ if it is the Frobenius SD).
The following is the overall approach that can be considered to analyze AltGDmin for solving a problem. This generalizes the ideas used for the above guarantees for LRCS and LRMC.
\bi
\item Analyze the initialization step to try to show that $\mathrm{NormDist}_a( \hat\Za, \Za^*) \le \delta_0$  with a certain probability. % for a $\delta_0$ small enough.

\bi
\item Typically, the initialization is a spectral initialization for which existing approaches (if any) can be used. %For example, this is the case for LRMC and robust LRMC. On the other hand, for robust LRCS, this step also needs to be carefully designed and analyzed.

\item  In many cases, $\delta_0$ being a small numerical constant suffices. For certain problems, it may even be possible to prove results with random initialization; in this case, $\delta_0$ is very close to 1.
\ei

\item At iteration $t$, suppose that we are given an estimate $\hat\Za$ satisfying $\mathrm{NormDist}_a( \hat\Za, \Za^*) \le \delta_{t-1}$ with $\delta_{t-1}$ ``small enough''.

\bi
\item Analyze the minimization step to show that $\mathrm{NormDist}_b(\hat\Zb, \Zb^*) \lesssim \delta_{t-1}$ with a certain probability. % for a small numerical constant $C$, e.g., $C=2$ or 3. %here $\lesssim$ means it can be $C$ times $\delta_{t-1}$ for a numerical constant $C \ge 1$.
\bi

\item This analysis will typically be the easier one, because this step is often a well studied problem, e.g., in case of LRCS or LRMC, it is the standard least squares (LS) problem. For LRPR, it is a standard phase retrieval problem. % while for robust LRCS/LRMC, it is a robust linear regression problem. %(in case of L+S models).
%This proof should require assuming that this step involves minimizing a strongly convex cost function. In many of the examples described above this is in fact a well-posed LS problem and hence satisfies this.
\ei

\item  Analyze the GD step to try to show that, for the updated $\Za$ estimate, $\hat\Za^+$, $\mathrm{NormDist}_a( \hat\Za^+, \Za^*) \le \delta_t:= c_1(\eta,\delta_0,\eps_1) \delta_{t-1} $ with a certain probability. This bound would hold under an upper bound on the step size $\eta$ and the initialization error $\delta_0$.
    %for a $c_1 < 1$.    % Here $c_1$ depends on $\eta$.

% Analyze the GD step and try to set the step size $\eta$ to show that, for the updated $\Za$ estimate, $\hat\Za{}^+$, $\mathrm{NormDist}_a( \hat\Za^+, \Za^*) \le c_1(\eta) \delta_{t-1} : = \delta_t$ for a $c_1 < 1$.     Here $c_1$ depends on $\eta$.

\ei

%\vsm \item We try to set $\eta$ to ensure that the above GD step error bound holds (requires $\eta$ to be small enough) and $c_1$ is as set above (set $\eta$ to the largest allowed value for the bound to hold).
%\vsm
%\item If we are able to show this with a step size $\eta$ that is constant, it will help guarantee that the iteration complexity grows logarithmically with $1/\eps$. We expect this to be possible, because each iteration of AltGDmin does a full min over $\Zb$. This should help ensure sufficient error decay at each GDmin iteration, even while using a constant GD step size.  As an example, we were able to do this for the LRCS problem described earlier. From simulation experiments, we conjecture that a similar analysis should be possible also for LRMC; see Fig. \ref{table_fig_compare}.

\ei
We set $\delta_0$, $\eta$ and $\eps_1$ to ensure that $c_1(\eta,\delta_0,\eps_1) \le c $ for a  $c<1$ (exponential error decay). If we can set $\eta$ to be a constant (w.r.t. $n,q,r$), it will help guarantee that the iteration complexity grows logarithmically with $1/\eps$ (fast convergence).
%$\delta_t = c \delta_{t-1}$  with $c<1$ (exponential error decay).
All the above steps should hold with a certain probability that depends on the problem dimensions ($n,q,r$ in case of LR problems), sample complexity ($m$ or $n p$ in case of LRCS and LRMC), and the values of $\delta_0$ and $\eps_1$.
We use our values of $\delta_0$ and $\eps_1$ to find a lower bound on the sample complexity in terms of $n,q,r$ and $\kappa,\mu$, in order to guarantee that all the above steps hold with a high enough probability. %the probability goes to zero if $n$ or $q$ goes to infinity.
%If this can be done, we can show that $\delta_t \le c_1^t \delta_0$ and this allows us get the iteration complexity $T$ to guarantee $\delta_T \le \eps$.

 Sample-splitting is assumed across iterations in order to make the analysis easier (a common technique for analyzing iterative algorithms that we learned about in \cite{lowrank_altmin} and follow-up works). This helps guarantee that the estimates $\hat\Za, \hat\Zb$ used in a given step are independent of the data used in that step. Using this assumption, (i) the expected value of terms can be computed more easily; and (ii) the summands in a given term are independent conditioned on past data making it possible to bound the deviation from the expected value using concentration bounds for sums of independent random variables/vectors/matrices \cite{versh_book}.

In the AltGDmin analysis, analyzing the GD step is the most challenging part. The reason is AltGDmin is not a GD or projected GD algorithm (both of which are well studied) for any variable(s). This means that the gradient at $\Za = \Za^*$, $\nabla_{\Za} f(\Za^*, \Zb)$  is not zero. %However, it should be possible to show that this quantity is small. One way to do this is to try to show that its expected value is either zero or is very small as done in our preliminary work for LRCS and LR phase retrieval \cite{lrpr_gdmin}. %. This was the case for LRCS and respectively If the GD step involves projected GD, e.g., in case $\Za$ denotes a subspace, there are other difficulties too because the distance metric being used is subspace distance.??

%We will first use the above approach to do a problem specific analysis to obtain sample and iteration complexity bounds for the three problems that are most similar to LRCS: robust LRCS, LRMC and robust LRMC.
%Next, we will try to generalize our approach, and also consider developing provable guarantees for tensor LRCS and for the clustering problem.
%%
%%Open questions for all problems include the following:
%%The first step will be to prove guarantees in the easiest setting: using a carefully designed initialization and assuming sample splitting. %We explain the ideas we will explore below.
%%

We emphasise here that because one of the steps in AltGDmin is a minimization step, there does not seem to be a way to prove guarantees without sample-splitting. All guarantees for AltMin require sample splitting \cite{lowrank_altmin,pr_altmin,lrpr_best}.  On the other hand, for  GD / factorized GD, it is possible to prove guarantees without sample splitting as done in \cite{rpca_gd,twf}.

\chapter{AltGDmin for LR Problems: Overall Proof Ideas} \label{proof_ideas_lr}
In this chapter, we explain the main ideas that can be used for analyzing AltGDmin for solving an LR recovery problem. We begin below by specifying AltGDmin for any LR problem. Next in Sec. \ref{cleangrad_proof}, we provide the main ideas for analyzing the noise-free attack-free linear measurements problems - LRCS and LRMC. In Sec. \ref{noisygrad_proof}, we explain how to analyze the noisy, attack-resilient or nonlinear measurement settings.
More details for both sections are provided in Chapter \ref{proof_details}.
%\subsubsection{Mathematical tools used}
%%A common set of tools used for analyzing AltGDmin are
The mathematical tools used in this analysis (linear algebra, probability and random matrix theory ideas) are summarized in Sec. \ref{prelims}. Many of these are from  \cite{spectral_init_review} and  \cite{versh_book} and include: (i)  singular value bounds, (ii) results such as the Davis-Kahan or the Wedin $\sin \Theta$ theorem, that can be used to obtain a deterministic bound the subspace distance between the estimate of the column-span of the unknown LR matrix $\Xstar$ and the true one \cite{spectral_init_review}; and (iii) matrix concentration bounds (or scalar ones combined with an appropriate epsilon-net argument) such as sub-Gaussian Hoeffding, sub-exponential Bernstein or ythe matrix Bernstein inequality \cite{versh_book}. In addition some basic linear algebra tricks are needed as well.

\section{AltGDmin for any LR matrix recovery problem}
Consider the problem of recovering $\X = \U \B$ from $\Y:= \mathcalA(\U\B) $ where $\mathcalA$ is a linear operator. We consider the squared loss function
\[
f(\U,\B):= \nu  \|\Y -  \mathcalA(\U\B) \|_F^2
\]
where $\nu$ is a quantity that does not depend on $\U,\B$ and that is used to normalize the loss function so that
\[
\E[\nabla_\U f(\U,\B)]  = (\X - \Xstar) \B^\top, \ \ \X := \U \B
\]
when $\U,\B$ are independent of the data $\{\Y, \mathcalA\}$.
For example, for LRCS, with $\A_k$ containing i.i.d. standard Gaussian entries, $\nu = 1/m$; while for LRMC, $\nu = 1/p$.

The AltGDmin algorithm proceeds as follows.
\bi
\item Initialize $\U$: use a carefully designed spectral initialization approach to get $\U_0$
\item Repeat the following for all $t=1$ to $T$:
\ben
\item Update $\B$ by minimization: obtain
$$\B_t: = \arg \min_{\B} f(\U_{t-1},\B)$$
(this statement assumes that the minimizer is unique and this fact is proved in the algorithm analysis). % to keep the algorithm itself general one could pick $\B_t$ from the set of minimizers $\arg \min_{\B} f(\U_{t-1},\B)$.
     %This assumes that the minimizer is unique; this is true if $\mathcal{A}$ is linear and
This step is efficient if it decouples column-wise, as in the case of LRCS, LRPR and LRMC.

\item Update $\U$ by GD followed by orthnormalization:
$$\U_t: = QR(\U_{t-1} - \eta \nabla_\U f(\U,\B_t) )$$
\een

\item Sample splitting is assumed as noted earlier.
\ei

\section{Proof approach: clean and noise-free case}\label{cleangrad_proof}
To explain the main ideas of our proof approach, we use the simplest setting: noise-free and attack-free LRCS and LRMC. Depending on the problem, we use $\SE_2$ or $\SE_F$. For any problem, $\SE_F$ can be used. For attack-free LRCS, use of $\SE_2$ gives in a better result (sample complexity lower by a factor of $r$).   When $\SE_2$ is used, all the norms below are $\|.\|$. When $\SE_F$ is used, all the numerator term norms are $\|.\|_F$.

Let $\U$ be the estimate at the $t$-th iteration.
Define
\begin{align*}
\g_k &  := \U^\top \xstar_k, k \in [q], \text{ and } \G:= \U^\top \Xstar, \\
\Pperp &  := \I - \Ustar \Ustar^\top, \\
\gradU &  :=  \nabla_\U f(\U,\B) \\  %= \sum_k \A_k^\top (\A_k \U \b_k - \y_k) \b_k ^\top  \\ %& =\sum_\ik (\y_\ik - \a_\ik{}^\top \U \b_k) \a_\ik \b_k{}^\top
\delta_t & := \SE(\U,\Ustar) = \|\Pperp \U\|
\end{align*}

Under the sample-splitting assumption, it can be shown that
\bea
\E[\gradU] = \E[\nabla_\U f(\U,\B)] = (\X - \Xstar) \B^\top
\label{EgradU}
\eea
%the value of $ $ depends upon how the cost function is scaled. For LRCS, $ =m$, for LRMC, it equals $p$.

Recall the Projected GD step for $\U$:
\begin{eqnarray}
\Utilde^+ &= \U - \eta \gradU \text{ and } \Utilde^+  \qreq \U^+\R^+
\label{Utildeeq}
\end{eqnarray}
Since $ \U^+= \Utilde^+ (\R^+)^{-1}$ and since $\|(\R^+)^{-1}\| =1/ \sigmamin(\R^+) = 1/ \sigmamin(\Utilde^+)$, thus, $\SE(\U^+, \Ustar)  = \| \Pperp \U^+\| $ can be bounded as
\begin{eqnarray}
\delta_{t+1}:= \SE(\U^+, \Ustar) & \le \dfrac{\|\Pperp \Utilde^+\|}{\sigmamin(\Utilde^+)}
\le \dfrac{\|\Pperp \Utilde^+\|}{\sigmamin(\U) - \eta \|\gradU\|}
= \dfrac{\|\Pperp \Utilde^+\|}{1 - \eta \|\gradU\|}
\label{SDeq}
\end{eqnarray}
This follows by Weyl's inequality and $\sigmamin(\U)=1$.
%Using Theorem \ref{Blemma}, we can show that = \eta   (\X - \Xstar) \B^\top$
Consider the numerator.
Using \eqref{Utildeeq}, adding/subtracting $\eta \E[\gradU]$, and using \eqref{EgradU} which implies that  $ \Pperp \E[\gradU] = \Pperp  (\X - \Xstar) \B^\top =   \Pperp \X \B^\top =   \Pperp \U \B \B^\top$, we get % (since $\Pperp \Xstar = \Pperp \Ustar \Bstar = 0$ and $\X = \U \B$),
\begin{align*}
%\tilde\U^+ & = \U - \eta  \E[\gradU] + \eta  (\E[\gradU] - \gradU ), \text{ thus,} \\
 \Pperp \tilde\U^+  % & = \Pperp \U - \eta m \Pperp (\X - \Xstar) \B^\top + \eta \Pperp (\E[\gradU] - \gradU) \\
 & = \Pperp \U - \eta   \Pperp \U \B \B^\top + \eta \Pperp ((\E[\gradU] - \gradU)) \\
 & = \Pperp \U (\I - \eta   \B \B^\top) +  \eta \Pperp (\E[\gradU] - \gradU)
\end{align*}
%The last row used $\Pperp \Xstar=0$. %$\E[\gradU] = m (\X - \Xstar) \B^\top$ with $\X = \U \B$, and
Thus,  using \eqref{SDeq}, %\eqref{EgradU},
%\begin{align}
%\|\Pperp \Utilde^+ \| \le \|\Pperp \U\| \|\I -  \eta   \B \B^\top\| + \eta \|\E[\gradU] - \gradU\| \label{eq:PU+}
%\end{align}
%and so
\begin{align*}
\SE(\U^+, \Ustar)
%& \le \dfrac{\|\Pperp \Utilde^+\|}{\sigmamin(\Utilde^+)} \\
%& \le \dfrac{\|\Pperp \Utilde^+\|}{1 - \eta \|\gradU\|} \\
& = \dfrac{\|\Pperp \U (\I - \eta   \B \B^\top) + \eta \Pperp (\E[\gradU] - \gradU)\|}{1 - \eta \| \E[\gradU] + \gradU - \E[\gradU]\|} \\
%& \le \dfrac{\|\Pperp \U\| \cdot \|\I -  \eta   \B \B^\top\| + \eta \|\E[\gradU] - \gradU\|\|}{1 - \eta \|\gradU\|} \\
& \le \dfrac{\|\Pperp \U\| \cdot \|\I -  \eta   \B \B^\top\| + \eta \|\E[\gradU] - \gradU\|  }{1 - \eta \|\E[\gradU]\| - \eta \|\E[\gradU] - \gradU \|}
\label{SDeq}
\end{align*}
%For an appropriate upper bound on the step size $\eta$,
Notice that
\[
\lambda_{\min}(\I -  \eta   \B \B^\top) = 1 - \eta   \|\B\|^2.
\]
Thus, if $\eta < 0.9/\|\B\|^2$, then $1 - \eta   \|\B\|^2 > 0.1 > 0$, i.e., the matrix $(\I -  \eta   \B \B^\top)$ is positive semi-definite (p.s.d.). This means that
\[
\|\I -  \eta   \B \B^\top\| = \lambda_{\max}(\I -  \eta   \B \B^\top) = 1 - \eta   \sigma_r(\B)^2
\]
%(We note here that the required upper bound on $\eta$ will decrease if the problem is noisy or otherwise
Thus, if $\eta \le 0.9/ \|\B\|^2$, then
\begin{align}
 \notag
\delta_{t+1} & := \SE(\U^+, \Ustar) \\  \notag
%& \le \dfrac{\|\Pperp \Utilde^+\|}{\sigmamin(\Utilde^+)} \\
& \le \dfrac{\|\Pperp \Utilde^+\|}{1 - \eta \|\gradU\|} \\  \notag
& = \dfrac{\|\Pperp \U (\I - \eta   \B \B^\top) + \eta \Pperp (\E[\gradU] - \gradU)\|}{1 - \eta \| \E[\gradU] + \gradU - \E[\gradU]\|} \\
%& \le \dfrac{\|\Pperp \U\| \cdot \|\I -  \eta   \B \B^\top\| + \eta \|\E[\gradU] - \gradU\|  }{1 - \eta \|\E[\gradU]\| - \eta \|\E[\gradU] - \gradU \|} \\  \notag
%
& \le \dfrac{\|\Pperp \U\| \cdot \|\I -  \eta   \B \B^\top\| + \eta \|\E[\gradU] - \gradU\| }{1 - \eta \|\E[\gradU]\| - \eta \|\E[\gradU] - \gradU \|} \\  \notag
& \le  \dfrac{\delta_t (1 - \eta   \sigma_r(\B)^2) + \eta \|\E[\gradU] - \gradU\| }{ 1 - \eta \|\E[\gradU]\| - \eta \|\E[\gradU] - \gradU \|} \\  \notag
& = \dfrac{\delta_t \left(1 - \eta   \left(\sigma_r(\B)^2 - \frac{\|\E[\gradU] - \gradU\|}{\delta_t}\right) \right)} {1 - \eta \|\E[\gradU]\| - \eta \|\E[\gradU] - \gradU \|} \\  \notag
& = \delta_t \left(1 - \eta   \left(\sigma_r(\B)^2 - \frac{\|\E[\gradU] - \gradU\|}{\delta_t}\right) \right) \left( 1 + 2\eta \|\E[\gradU]\| + 2 \eta \|\E[\gradU] - \gradU \| \right) \\  \notag
& \le \delta_t \left(1 - \eta   \left(\sigma_r(\B)^2 - \frac{\|\E[\gradU] - \gradU\|}{\delta_t} - 2\|\E[\gradU]\| - 2 \|\E[\gradU] - \gradU \|\right) \right) \\
& \le \delta_t \left(1 - \eta   \left(\sigma_r(\B)^2 - 3\frac{\|\E[\gradU] - \gradU\|}{\delta_t} - 2\|\E[\gradU]\| \right) \right)
\label{deltatbnds}
\end{align}
using $(1-x)^{-1} \le (1 + 2x)$ for $x<0.5$; $(1-z)(1+2x) = 1- z + 2x - 2xz < 1 - z + 2x = 1 - (z - 2x)$; and $\delta_t < 0.5$ (this allows us to replace $ \frac{\|\E[\gradU] - \gradU\|}{\delta_t} + 2 \|\E[\gradU] - \gradU \|$ by  3 times the first term).

The next step is to upper bound the expected gradient norm $\|\E[\gradU]\|$ and the gradient deviation norm $\|\E[\gradU] - \gradU \|$ and to lower bound $\sigma_r(\B)$. We need tight enough bounds in order to be able to show that for $\eta$ small enough,
\[
\delta_{t+1} \le (1 - c_1/\kappa^2) \delta_t
\]
To bound the terms, we use matrix concentration bounds from Sec. \ref{prelims} and the incoherence assumptions on $\Bstar$ (LRCS) or on $\Bstar$ and $\Ustar$ (LRMC).
The overall approach is as follows. We provide more details in Sec. \ref{proof_details}. At each iteration,
\bi

\item The first step is to analyze the minimization step to bound $\|\B - \G\|$: for LRCS, we can bound $\max_k \|\b_k - \g_k\|$ and use it to bound $\|\B - \G\|$; for LRMC, we can only bound the matrix error.
\bi
\item This is used to bound $\|\X - \Xstar\|$ using triangle inequality
\item  The above is used to upper bound $\sigmamax(\B)$ and lower bound $\sigma_r(\B)$ using tricks from Sec. \ref{prelims}.
\ei

\item We use the minimization step bounds to bound $\|\E[\gradU]\| \le \|\X - \Xstar\| \sigmamax(\B)$.

\item We bound $\|\E[\gradU] - \gradU \|$ using matrix concentration inequalities from Sec. \ref{prelims} and  the minimization step bounds.

\item Most above results also use incoherence of $\B$ (for LRCS) and of $\B$ and of $\U$ (for LRMC), which needs to be proved.
\bi
\item Incoherence of $\B$ is easy to show. %For LRCS it is a direct consequence of the $\b_k - \g_k$ bound. For LRMC
\item Incoherence of $\U$ for LRMC requires an inductive argument and a concentration bound on $\max_j \|\e_j^\top (\gradU - \E[\gradU])\|$. We need a bound on it that contains a factor of $\sqrt{r/n}$; but it need not contain a factor of $\delta_t$.
\ei

\ei
We provide details of the above steps in Sec. \ref{proof_details}.

\section{Proof approach: Noisy Gradient approach to deal with Nonlinear or Noisy or Attack-prone cases}\label{noisygrad_proof}
Above we explained the proof strategy for the simple noise-free case. Here, we explain how to modify it to deal with various modifications of the basic LRCS or LRMC problems. One simple example where this occurs is the LR phase retrieval (LRPR) problem which is the phaseless measurements' generalization of LRCS. A second example is noise-corrupted LRCS or LRMC. A third setting is dealing with attacks, such as the Byzantine attack, by malicious nodes. The algorithm itself may remain the same (noisy case) or may change (LRPR or attack setting).
%?? move to LR problems section. For LRPR, one modifies basic AltGDmin as follows: replace $\y_k$ in the gradient computation by $\hat\y_k = phase(\A_k \U \b_k) \z_k$; and replace the LS step to estimate $\b_k$  by any standard PR algorithm such as \cite{twf} or \cite{rwf}. The initialization is also different.

In all the above cases, the gradient expression will be different in the GD step to update $\U$. To bound $\SE(\U^+, \Ustar)$, we need to define and bound an extra term that we refer to as $\Err$ for ``Error term''. Let
\[
\Err:= \gradU_{cln} - \gradU
\]
where $\gradU_{cln}$ is the gradient from the noise-free case section that satisfies
\[
\E[\gradU_{cln}] =   (\X - \Xstar) \B^\top
\]
We proceed exactly as in the noise-free case, with the following modification: we add/subtract  $\E[\gradU_{cln}] =   (\X - \Xstar) \B^\top$ and we add/subtract $\gradU_{cln}$. This gives the following:  if $\eta \le 0.9/ \|\B\|^2$, then
\begin{align}
\delta_{t+1} & := \SE(\U^+, \Ustar) \\  \notag
%& \le \dfrac{\|\Pperp \Utilde^+\|}{\sigmamin(\Utilde^+)} \\
& \le \dfrac{\|\Pperp \Utilde^+\|}{1 - \eta \|\gradU\|} \\  \notag
& = \dfrac{\|\Pperp \U (\I - \eta   \B \B^\top) + \eta \Pperp (\E[\gradU_{cln}] - \gradU_{cln})  + \eta \Pperp \Err \|}{1 - \eta \| \gradU - \gradU_{cln} + \gradU_{cln}  - \E[\gradU_{cln}] - \E[\gradU_{cln}]\|} \\  \notag
& \le \dfrac{\|\Pperp \U\| \cdot \|\I -  \eta   \B \B^\top\| + \eta \|\E[\gradU_{cln}] - \gradU_{cln}\| + \eta \|\Err\| }{1 - \eta \|\E[\gradU_{cln}]\| - \eta \|\E[\gradU_{cln}] - \gradU_{cln} \| -  \eta \|\Err\| } \\  \notag
& \le \dfrac{\delta_t (1 -  \eta  \sigma_r(\B)^2)  + \eta \|\E[\gradU_{cln}] - \gradU_{cln}\| + \eta \|\Err\| }{1 - \eta \|\E[\gradU_{cln}]\| - \eta \|\E[\gradU_{cln}] - \gradU_{cln} \| -  \eta \|\Err\| }
\label{deltatbnds_noisy_1}
\end{align}
Using  $(1-x)^{-1} \le (1 + 2x)$ for $x<0.5$; $(1-z)(1+2x) <  1 - (z - 2x)$; $\delta_t < 0.5$; and using $ \|\E[\gradU_{cln}]\|  < 1$, $\|\E[\gradU_{cln}] - \gradU_{cln} \|<1$, and $\|\Err\| < 1$,
\begin{align}
\delta_{t+1} & := \SE(\U^+, \Ustar) \\  \notag
& \le \dfrac{\delta_t (1 -  \eta  \sigma_r(\B)^2)  + \eta \|\E[\gradU_{cln}] - \gradU_{cln}\| + \eta \|\Err\| }{1 - \eta \|\E[\gradU_{cln}]\| - \eta \|\E[\gradU_{cln}] - \gradU_{cln} \| -  \eta \|\Err\| } \\  \notag
& \le \dfrac{\delta_t (1 -  \eta  (\sigma_r(\B)^2 - \eta \frac{\|\E[\gradU_{cln}] - \gradU_{cln}\|}{\delta_t})) + \eta \|\Err\| }{1 - \eta \|\E[\gradU_{cln}]\| - \eta \|\E[\gradU_{cln}] - \gradU_{cln} \| -  \eta \|\Err\| } \\  \notag
%& \le  \left( \delta_t (1 -  \eta  (\sigma_r(\B)^2 - \eta \frac{\|\E[\gradU_{cln}] - \gradU_{cln}\|}{\delta_t})) + \eta \|\Err\| \right) {(1 + 2 \eta \|\E[\gradU_{cln}]\| + 2 \eta \|\E[\gradU_{cln}] - \gradU_{cln} \| + 2  \eta \|\Err\| ) } \\  \notag
%
& \le  \delta_t \left( 1 -  \eta  (\sigma_r(\B)^2 - \frac{\|\E[\gradU_{cln}] - \gradU_{cln}\|}{\delta_t}) \right) (1 + 2 \eta \|\E[\gradU_{cln}]\| + 2 \eta \|\E[\gradU_{cln}] - \gradU_{cln} \| + 2  \eta \|\Err\| ) \\  \notag
& + \eta \|\Err\| (1 + 2 \eta \|\E[\gradU_{cln}]\| + 2 \eta \|\E[\gradU_{cln}] - \gradU_{cln} \| + 2  \eta \|\Err\| )    \\  \notag
& \le  \delta_t \left(1 -  \eta  (\sigma_r(\B)^2 -  \frac{\|\E[\gradU_{cln}] - \gradU_{cln}\|}{\delta_t} -  2 \|\E[\gradU_{cln}]\| -  2  \|\Err\| \right) \\  \notag
& + \eta \|\Err\| (1 + 2 \eta \|\E[\gradU_{cln}]\| + 2 \eta \|\E[\gradU_{cln}] - \gradU_{cln} \| + 2  \eta \|\Err\| )    \\  \notag
\label{deltatbnds_noisy}
\end{align}

The next steps are similar to those in the noise-free case. What result we can finally prove depends on how small $\Err$ is.
\ben
\item If $\|\Err\|$ is of the same order as  (or smaller than) the gradient deviation term, then the noise-free case analysis extends without much change. This requires $\|\Err\|$ to decay as $c \delta_t$ for a $c<1$, which can be shown for the LR phase retrieval problem.

\item If $\|\Err\|$ is not as small, but is of order $\delta_0$ or smaller, then the final bound will contain two terms: the first decays with $t$, and the second is a constant term that is of the order of $\max_t \|\Err_t\|$. %Even to prove this bound, we do need $\|\Err\|$ to be sufficiently smaller than the lower bound on $\sigma_r(\B)$.
   % at least of order $\delta_0$
\een
We provide details for both steps in Sec. \ref{proof_details}.

\chapter{AltGDmin for LR Problems: Proof Details} \label{proof_details}

\section{Key Results Used}\label{keyres}

%The epsilon net approach \cite[Chap 4]{versh_book} that is used in our proofs can be summarized in the result below. We explain what an epsilon net is Sec. \ref{prelims}.
%
%\begin{theorem}%[Bounding $\|\M\|$]
%For an $n \times r$ matrix $\M$ and fixed vectors $\w, \z$ with, $\w \in \calS_n$ and $\z \in \calS_r$, suppose that
%\[
%|\w^\top \M \z | \le b_0 \text{ w.p. at least } 1-p_0 .
%\]
%where $b_0$ does not depend on $\w,\z$. Then,
%\[
%\|\M\| \le 1.4 b_0  \text{ w.p. at least } 1- \exp( (\log 17) (n+r) ) \cdot p_0
%\]
%\label{epsnet_Mwz}
%\end{theorem}
%
%\begin{proof}
%Denote $\eps_0$-nets covering $\calS_{n-1}$ and $\calS_{r-1}$ by $\bar\calS_{n-1}$ and $\bar\calS_{r-1}$.
%Using union bound and an approach similar to that of the proof of  Lemma 4.4.1 of \cite{versh_book}, we can show the following:  w.p. at least $1 - (1+2/\eps_0)^{n+r} p_0$,
%\bi
%\item $\max_{\w \in \bar\calS_{n-1}, \z \in \bar\calS_{r-1}} |\w^\top \M \z | \le b_0$ and
%\item $\|\M\| := \max_{\w \in \calS_{n-1}, \z \in \calS_{r-1}} \w^\top \M \z  \le \max_{\w \in \calS_{n-1}, \z \in \calS_{r-1}} |\w^\top \M \z | \le \frac{1}{1 - 2\eps_0 - \eps_0^2} b_0$.
%\ei
%The proof of the first item follows by union bound; that of second item above follows that of Lemma 4.4.1 of \cite{versh_book}.
%Using $\eps_0=1/8$ gives the theorem's  conclusion.
%\end{proof}

By combining Theorem \ref{epsnet_Mwz} given in Chapter \ref{prelims} with the scalar sub-exponential Bernstein or sub-Gaussian Hoeffding inequalities, we obtain the following two results which have been widely used in the LR recovery and phase retrieval literature. These study sums of rank-one matrices which are outer products of specific types of random vectors (r.vec). The last result below is the matrix Bernstein inequality.

\begin{corollary}[Sum of rank-one matrices that are outer products of two sub-Gaussian r.vecs.]\label{subexpo_epsnet}
Consider a sum of $m$ zero-mean independent rank-one $n \times r$ random matrices $\x_i \z_i^\top$ with $\x_i, \z_i$ being sub-Gaussian random vectors with sub-Gaussian norms $K_{x,i}, K_{z,i}$ respectively.
%that are such that, for unit vectors, $\w,\z$, $\w^\top \X_i \z$  are sub-exponential r.v.s with sub-exponential norm $K_{e,i}$ that does not depend on $\w,\z$. An example  is if each $\X_i$ is an outer product of two sub-Gaussian random vectors.
For a $t \ge 0$,
\[
\|\sum_{i=1}^m \x_i \z_i^\top\| \le 1.4 t
\]
with probability at least
\[
1- \exp\left( (\log 17) (n+r) - c \min \left(  \frac{t^2}{\sum_i (K_{x,i}, K_{z,i})^2 } , \frac{t}{\max_i (K_{x,i}, K_{z,i}) }  \right)  \right)
\]
\end{corollary}

By combining Theorem \ref{epsnet_Mwz} with the scalar sub-Gaussian Hoeffding inequality, we conclude the following.

\begin{corollary}[Sum of rank-one matrices that are outer products of a sub-Gaussian r. vec. and a bounded r.vec.]\label{subG_epsnet}
Consider a sum of $m$ zero-mean independent rank-one $n \times r$ random matrices $\x_i \z_i^\top$ with $\x_i$ being sub-Gaussian random vector with sub-Gaussian norms $K_{x,i}$ and $\z_i$ being a bounded random vector with $\|\z_i\| \le L_i$.
Then, clearly, for any $\w, \w'$, $\w^\top \x_i \z_i^\top \w'$ is a sub-Gaussian r.v. with sub-Gassian norm $K_{x,i} L_i$.
Thus, for a $t \ge 0$,
\[
\|\sum_{i=1}^m \x_i \z_i^\top\| \le 1.4 t
\]
with probability at least
\[
1- \exp\left( (\log 17) (n+r) - c  \frac{t^2}{\sum_i (K_{x,i} L_i)^2 } \right)
\]
\end{corollary}

For bounded matrices, the following matrix Bernstein result gives a much tighter bound than what would be obtained by combining scalar bounded Bernstein and Theorem \ref{epsnet_Mwz}. See Sec. \ref{prelims} for details.

\begin{theorem}[Matrix Bernstein]\label{mat_bern}
    Let $\X_1, \X_2, \dots \X_m$ be independent, zero-mean, $n \times r$ matrices with $\|\X_i\| \le L$ for all $i=1,2,...m$. Define the ``variance parameter'' of the sum
\[
 v := \max\left(  \|\sum_i \E[\X_i \X_i^\top]\|,  \|\sum_i \E[ \X_i^\top \X_i] \| \right).
\]
Then,
\[
\| \sum_{i=1}^m \X_i \| \le t
\]
with probability at least
\[
%1 - 2 \exp\left( - c \frac{t^2}{v + L t /3} \right) \ge
1 - 2 \exp\left(\log \max(n,r)  -  c \min\left(\frac{t^2}{v}, \frac{t}{L} \right)  \right)
\]
\end{theorem}

\section{Analyzing the Initialization step}

The following is the overall approach to analyze spectral initialization.

\bfpara{LRCS initialization}
%We explain the overall proof strategy here for LRCS.
Recall the LRCS initialization from Sec. \ref{altgdmin_lrcs}.
% that, for LRCS, we initialize $\U$ by computing the top $r$ singular vectors of the following matrix
%\[
%\X_0 := \sum_k \A_k^\top \y_{k,trnc} \e_k^\top, \ \y_{k,trnc}:= \mathrm{trunc}(\y_k,\alpha) % (\y_k \circ \indic_{|\y_k| \le  \sqrt\alpha})
%\]
%Here $ \alpha:= \tilde{C} \sum_k \|\y_k\|^2 / mq$ with $\tilde{C}:=9\kappa^2\mu^2$ and the function $\mathrm{trunc}$ truncates (zeroes out) all entries of the vector $\y_k$ with magnitude greater than $\sqrt\alpha$, i.e.,  for all $j\in [n]$, $\mathrm{trunc}(\y,\alpha)_j = (\y)_j  \indic_{|\y_j| \le  \sqrt\alpha}$, with $\indic$ being the indicator function.
%
It is not hard to show that
$$\E[\X_0] = \Xstar \D(\alpha)$$
where $\D(\alpha)$ is a diagonal $q \times q $ matrix with all non-zero entries and that $\sigmamin(\D) = \min_k (\D)_{k,k} \ge 0.9$ with high probability (w.h.p.) \cite{lrpr_gdmin_2,lrpr_gdmin}.
Thus, $\E[\X_0]$ is a rank $r$ matrix with column-span equal to that of $\Ustar$ (or of $\Xstar$). Using Wedin $\sin \theta$ theorem,
\[
\SE_F(\U_0,\Ustar) \le \dfrac{\sqrt{r}\|\X_0 - \E[\X_0]\|}{\sigma_r(\E[\X_0]) - 0 - \|\X_0 - \E[\X_0]\|}
\]
Using the bounds from Sec. \ref{minsingval}, it is not hard to see that
\[
\sigma_r ( \E[\X_0] )
= \sigma_r(\Xstar \D(\alpha))
\ge \sigma_r(\Xstar) \sigmamin( \D(\alpha)^\top ) = \sigmin \sigmamin(\D(\alpha)) \ge 0.9 \sigmin
\]
To upper bound $\|\X_0 - \E[\X_0]\|$ observe that $\X_0$ can be rewritten as
\[
\X_0 = \sum_{k=1}^q \sum_{i=1}^m \a_\ik (\a_\ik^\top \xstar_k) \indic_{ (\a_\ik^\top \xstar_k)^2 \le \alpha }
\]
where $\a_\ik$ is a an $\mathcal{N}(\bm{0},\I)$ (standard Gaussian) vector).
We first use the sub-exponential Bernstein inequality to lower and upper bound $\alpha$ by a constant times $\|\Xstar\|_F^2/q$ w.h.p.
Next, we use Corollary \ref{subG_epsnet} (sub-Gaussian Hoeffding inequality + epsilon-netting)  to bound $\|\X_0 - \E[\X_0]\|$ under a lower bound on the sample complexity $mq$. This, along with using the Wedin $\sin \theta$ theorem (Theorem \ref{Wedin_sintheta}) \cite{spectral_init_review}, and lower bounding the smallest entry of $\D(\alpha)$, helps bound the subspace distance between $\U_0$ and $\Ustar$.

\bfpara{LRMC initialization}
For LRMC, $\X_0 = \Y$ and it is easy to see that $\E[\Y] = \Xstar$.
Thus, we can again use Wedin and a different matrix concentration bound (matrix Bernstein) to show that $\U_{00}$ is a good subspace estimate of $\Ustar$. Analyzing the second step (projection onto row incoherent matrices) uses the fact that the set of row incoherent matrices is convex, this argument is borrowed from \cite{rpca_gd}). Finally we analyze the orthonormalization
%where $\Y$ is a matrix of the same size as $\Xstar$ with the unobserved entries set to zero.

%The second step involves computing $\U_0$ by projecting $\U_{00}$ onto the set of row incoherent matrices using \eqref{proj_incoh}, and orthonormalizing the output.  Using matrix Bernstein and Wedin's $\sin \Theta$ theorem as done in \cite[Theorem ?]{spectral_init_review}, one can argue that $\U_{00}$ is close to $\Ustar$ in $\SE_F$ subspace distance. We further show that the same is true for the projection by using the fact that the projection is onto a convex set (the set of row incoherent matrices is convex); this argument is borrowed from \cite{rpca_gd}) and analyzing the orthonormalization.

\section{Clean Noise-free case}
Recall equation \eqref{deltatbnds} from the previous chapter. We make the following assumption temporarily just to show the flow of our proof. Later, we explain how to show that this assumption holds w.h.p. under just a sample complexity lower bound and incoherence Assumption \ref{right_incoh} or \ref{incoh}. The final guarantee only needs to make assumptions on the true $\Xstar$ or on $\Y$ (data).
 %the multiplier of $\eta$ is $\ge c_1$ for a $c_1 > 0$, i.e., this term is strictly positive.
%
\begin{tempassu}
Suppose that
\ben
\item $\sigma_r(\B) \ge 0.9 \sigma_r(\Bstar) = 0.9 \sigmin $ %, $\sigmamax(\B) \le 1.1 \sigmax$,
\item $\|\E[\gradU] - \gradU\| \le c_1 \delta_t \sigmin^2 $ for a $c_1 < 0.2$,
\item $\|\E[\gradU]\| \le C_2 \kappa^2 \delta_t \sigmin^2 $,
\item with $\delta_t$ such that  $c_1 +  (2C_2 \kappa^2 +  2c_1) \delta_t  <0.1$ for all $t \ge 0$, and
\item $\eta = c_\eta / \sigmax^2$ with $c_\eta \le 0.9 / (1.1   \sigmax^2)$, and
\item equation \eqref{EgradU} holds.
\een
We note here that the big $C_?$ numbered constant $C_2$ can depend on $r$:  if $\SE_2$ is used as the subspace distance measure then, $C_2 = C \sqrt{r}$ while if $\SE_F$ is used then $C_2 = C$.
\label{assu1}
\end{tempassu}
Using \eqref{deltatbnds} and the bounds from the above Temporary Assumption \ref{assu1},
\begin{align*}
\delta_{t+1}:= \SE(\U^+, \Ustar)
& \le \delta_t \left(1 - \eta   \left(0.9\sigmin^2 - c_1 \sigmin^2 - 2C_2 \delta_t \sigmin^2 - 2c_1 \delta_t \sigmin^2 \right) \right) \\
& \le \delta_t \left(1 - \eta  \sigmin^2 \left(0.9 - c_1 - (2 C_2 \kappa^2 + 2c_1)\delta_t  \right) \right) \\
& \le \delta_t (1 - \eta   (0.9 - c_1 - 0.1) \sigmin^2 ) = \delta_t \left(1 - (0.8-c_1)\frac{c_\eta}{\kappa^2} \right)
\end{align*}

\bfpara{Proving the bounds of Temporary Assumption \ref{assu1}}
%\subsubsection{Use of concentration bounds to obtain high probability guarantees}
We use the matrix concentration bounds from Sec. \ref{prelims} to bound the following with high probability
\ben
\item $\max_k \|\b_k - \g_k\|$ and use it to bound $\|\B - \G\|_F$ (for LRCS) or directly bound $\|\B - \G\|_F$  (for LRMC)
\item $\|\E[\gradU] - \gradU\|$ (in case of LRCS) or $\|\E[\gradU] - \gradU\|_F$  (for LRMC)
\item $(\S_k \U)^\top (\S_k \U)$ and $(\S_k \Ustar)^\top (\S_k \U)$ (for LRMC)
%$\|\b_k\|$  (in case of LRMC); LRCS also needs a bound on $\|\b_k\|$ (incoherence of $\B$), but this follows directly from the bound on $\max_k \|\b_k - \g_k\|$.
\item and $\|\e_j^\top (\E[\gradU] - \gradU)\|$  (for LRMC).
\een
The LRCS proofs use Corollary \ref{subexpo_epsnet} (sub-exponential Bernstein inequality followed by the epsilon-net argument).  The LRMC proofs use the Theorem \ref{mat_bern} (matrix Bernstein inequality). % given in  from  Sec. \ref{prelims}.
%The proof for LRCS initialization step uses the sub-Gaussian Hoeffding followed by epsilon net Corollary \ref{subG_epsnet} while that for LRMC uses matrix Bernstein.

The $\B-\G$ bound  assumes incoherence of $\U$ (in case of LRMC).
The  gradient deviation  $\|\E[\gradU] - \gradU\|$ bound assumes the $\B-\G$ bound and incoherence of $\B$ and of $\U$ (in case of LRMC).
%The $\B$ incoherence bound assumes incoherence of $\U$ (in case of LRMC).  To obtain this bound, one proceeds with obtaining an exact expression for $\b_k$ and simplifying it.
The bound on $\|\e_j^\top (\E[\gradU] - \gradU)\|$ (used to show incoherence of the updated $\U^+$ in case of LRMC) uses incoherence of $\B$.

\bi
\item  \bfpara{Bounding $\|\E[\gradU] \| = C \|(\X - \Xstar) \B^\top\|$ and $\sigma_r(\B)$}
The bound on $\|\B-\G\|$ is used to upper bound  $\|\X- \Xstar\|$ and $\|\B\|=\sigmamax(\B)$ and to lower bound  $\sigma_r(\B)$. The first two bounds are used to bound  $\|\E[\gradU]\| \le \|\X- \Xstar\| \|\B\|$.
The $\sigmamax$ bound is straightforward. The $\sigmamin$ bound follows using the bounds given next; these follow using the preliminaries from Sec. \ref{minsingval}.
\[
\sigma_r(\B) \ge \sigma_r(\G) - \|\B - \G\|,
\]
\[
\sigma_r(\G)= \sigmamin(\G^\top) = \sigmamin(\Bstar{}^\top \Ustar^\top \U) \ge \sigmin \sigmamin(\Ustar^\top \U)
\]
\begin{align*}
\sigmamin^2(\Ustar^\top \U)
& = \lambda_{\min}(\U^\top \Ustar \Ustar^\top \U) \\
& = \lambda_{\min}(\U^\top (\I - \Pperp) \U) \\
& = \lambda_{\min}(\I - \U^\top \Pperp \U) \\
& = \lambda_{\min}(\I - \U^\top \Pperp^2 \U) = 1 - \|\Pperp U\|^2 = 1 - \SE(\Ustar,\U)^2
\end{align*}

%\bfpara{Delete begin}
%\[
%\sigma_r(\G) = \sigma_r(\U^\top \Ustar \Bstar) \ge \sigmamin(\U^\top \Ustar) \sigma_r(\Bstar) = \sigmamin(\U^\top \Ustar)\sigmin,
%%(\Bstar{}^\top \Uatar{}^\top \U)
%\]
%\begin{align*}
%\sigmamin^2(\U^\top \Ustar) & = \sigma_r^2(\U^\top \Ustar) =  \sigma_r^2(\Ustar{}^\top \U) \\
% & = \lambda_r(\U^\top \Ustar \Ustar^\top \U) = \lambda_{\min}(\U^\top \Ustar \Ustar^\top \U) = \lambda_{\min}(\U^\top (\I - \Pperp) \U) \\
% & = \lambda_{\min}(\I - \U^\top \Pperp \U) = \lambda_{\min}(\I - \U^\top \Pperp^2 \U) = 1 - \|\Pperp U\|^2 = 1 - \SE(\Ustar,\U)^2
%\end{align*}
%
%\bfpara{Delete end}

\item \bfpara{Showing Incoherence of updated $\B$}
LRCS only needs incoherence of $\B$. LRMC needs that of both.
%bounds are needed for LRCS and LRMC. %In both cases, the update of $\B$ is decouple
\bi
\item Incoherence of $\b_k$ (columns of $\B$) for LRCS: Since we can bound $\|\b_k - \g_k\|$ for each $k$, the bound on $\|\b_k\|$ follows directly from this bound and the fact that $\|\g_k\| \le \|\bstar_k\|$.

\item  Incoherence of $\b_k$ (columns of $\B$) for LRMC: We use the exact expression for updating $\b_k$ and concentration bounds on $(\S_k \U)^\top (\S_k \U)$ and $(\S_k \Ustar)^\top (\S_k \U)$ to get a bound that is a constant times $\|\bstar_k\|$. %The incoherence assumption then $\|\bstar_k\| \le \mu \sqrt{r/q}\sigmax$.  %incoherence of $\U$, to bound $\|\b_k\|$ while realizing that the bound does not need to contain $\delta_t$, but instead it does need a factor of $\sqrt{r/q}$.
%This requires using  $|x-z|\le 2 \max(|x|,|z|)$ and bounding both $|x|,|z|$ so that the bound is a constant  times $\|\bstar_k\|$.%
\ei

\item \bfpara{Showing Incoherence of updated $\U$: only for LRMC}
 This is shown in two steps:
\bi
\item First we bound the deviation of the $j$-th row of $\gradU$ from its expected value, $\max_j \|\e_j^\top (\E[\gradU] - \gradU)\| $. We need a bound on it that contains a factor of $\sqrt{r/n}$; but it need not contain a factor of $\delta_t$. To get such a bound we use $|x-z|\le 2 \max(|x|,|z|)$ and obtain a bound of the form $c_4  \max(\|\e_j^\top \Ustar\|,\|\e_j^\top \U\|) \sigmax^2$.
    % $\max(\|\ustar^j\|,\|\u^j\| \sigmax^2$.

\item Next, we use \eqref{Utildeeq} and proceed in a fashion similar to \eqref{deltatbnds} to bound $\|\e_j^\top \Utilde^+\|$. We show that  $\|\e_j^\top \Utilde^+\| \le (1 - (0.9-0.1) \eta ) \|\e_j^\top \U\| + \|\e_j^\top \Ustar\| +  c_4 \max(\|\e_j^\top \Ustar\|,\|\e_j^\top \U\|)$  followed by simplifying this expression using the denominator expression. We can simplify this bound by using $\max(x,z) \le x+z$.
    We get  $\|\e_j^\top \U^+\| \le  (1 - c/\kappa^2) \|\e_j^\top \U\| +  2 \|\e_j^\top \Ustar\|$.
    %We get an expression that does not grow with iteration number, $t$.
\ei
Finally, recursively applying the above expression we get a bound on $\|\e_j^\top \U_t\|$ in terms of $\|\e_j^\top \U_0\|$ and $\|\e_j^\top \Ustar\|$ that does not grow with $t$. We show that
\[
\|\e_j^\top \U_t\| \le  (1 - c/\kappa^2)^t \|\e_j^\top \U_0\| + \sum_{\tau=0}^{t-1} (1 - c/\kappa^2)^\tau 2 \|\e_j^\top \Ustar\|  \le \|\e_j^\top \U_0\| + C \kappa^2 \|\e_j^\top \Ustar\|
\]
\

\item \bfpara{Ensuring $(2C_2 \kappa^2 + 2c_1) \delta_t  < 0.1$} %Guaranteeing $\delta_t$ upper bound}
%
%This requires $\delta_t < (0.1 - 2c_1)/ (2C_2 \kappa^2)$.
Note from above that we are showing exponential decay for $\delta_t$ with $t$. A corollary of this is that $\delta_t < \delta_0$. Thus, our bound on $\delta_t$ holds if $\delta_0$ satisfies the same bound, i.e., if the initialization is good enough i.e. if $\delta_0 < 0.1/ (2C_2 \kappa^2 + 2c_1)$. The proof approach given above shows exponential error decay and hence $\delta_{t+1} < \delta_{t} < \delta_0$ for each $t$.
%We note here the $ , C_2$ can depend on $r$.
(We note here the above numbered constants $c_1,C_2$ etc can depend on $n,q,r,\kappa^2$. For example, in case of LRMC, $C_2= \sqrt{r}$.)

\ei

%All proofs also need careful linear algebra tricks.

%Concentration bounds (such as sub-exponential Bernstein inequality along with an epsilon-net argument or matrix Bernstein inequality) are needed to bound $\|\B - \G\|_F$, the bounding on $\|\E[\gradU] - \gradU\|$, and the bounding of $\|\e_j^\top (\E[\gradU] - \gradU)\| $ for use in showing incoherence of $\U$ for LRMC

\section{Nonlinear or Noisy or attack-prone or outlier corrupted settings}
%Above we explained the proof strategy for the simple noise-free case. In this section we explain how to modify the proof approach to deal with various modifications of the basic LRCS or LRMC problems. One simple example is the LR phase retrieval (LRPR) problem which is the phaseless measurements' generalization of LRCS. A second example is noise-corrupted LRCS or LRMC. A third setting is dealing with model-poisoning attacks, such as the Byzantine attack, by malicious nodes. LRPR and dealing with attacks requires a modified algorithm, where as for the noisy case, the algorithm remains the same.

Using equation \eqref{deltatbnds_noisy_1} from the previous chapter and Temporary Assumption \ref{assu1},
\begin{align}
\notag
\delta_{t+1} & := \SE(\U^+, \Ustar) \\  \notag
& \le  \dfrac{\delta_t (1 - \eta   0.9 \sigmin^2)  + \eta c_1 \delta_t \sigmin^2 + \eta \|\Err\|}{1 - \eta C_2 \kappa^2 \delta_t \sigmin^2  - \eta c_1 \delta_t \sigmin^2  -  \eta \|\Err\| } \\ %  \notag
& \le  \dfrac{\delta_t  \left( 1 - \frac{c_\eta}{\kappa^2}  (0.9 - c_1) \right) + \frac{c_\eta}{\kappa^2} \|\Err\|}{1 - \frac{c_\eta}{\kappa^2}(C_2 \kappa^2 + c_1) \delta_t  -  \frac{c_\eta}{\kappa^2}\|\Err\| }  %  \notag
%
%& \le  \dfrac{\delta_t (1 -  (0.9 - c_1) \frac{c_\eta}{\kappa^2}  + \frac{c_\eta}{\kappa^2} \|\Err\|}{1 -  \delta_t (C_2 \kappa^2   + c_1)\frac{c_\eta}{\kappa^2}  -  \frac{c_\eta}{\kappa^2}\|\Err\| } \\  \notag
%& \le  \delta_t \left( 1 - \frac{c_\eta}{\kappa^2}  (0.9 - c_1) + \frac{c_\eta}{\sigmax^2} \|\Err\| \right) \left( 1 +  (2C_2 \kappa^2   + 2c_1) \delta_t \frac{c_\eta}{\kappa^2}   + 2\frac{c_\eta}{\sigmax^2} \|\Err\| \right) %\\
%& \le  \delta_t \left( 1 - \frac{c_\eta}{\kappa^2}  (0.9 - c_1) + \frac{c_\eta}{\sigmax^2} \|\Err\| \right) \left( 1 +  (2C_2 \kappa^2   + 2c_1) \delta_t \frac{c_\eta}{\kappa^2}   + 2\frac{c_\eta}{\sigmax^2} \|\Err\| \right) \\
%& \le  \delta_t \left( 1 - \frac{c_\eta}{\kappa^2}  \left( 0.9 - c_1 - (2C_2 \kappa^2   + 2c_1) \delta_t + \frac{2\|\Err\|}{\sigmin^2} \right)  \right)
%+ \left(1 +  \frac{c_\eta}{\kappa^2} \left( (2C_2 \kappa^2   + 2c_1) \delta_t    +  \frac{2\|\Err\|}{\sigmin^2} \right) \right) \frac{c_\eta}{\sigmax^2} \|\Err\| \\
\label{deltatbnds_noisy_2}
\end{align}
%
%We consider two types of bounds on the $\Err$ term.
\ben
\item First, if we can show that the $\Err$ term is of the same order as the gradient deviation term, i.e., both decay at the same rate as $\delta_t$ w.h.p., under the desired sample complexity bound, i.e., if %for a %with $c_5 < (0.9 - c_3)$
\[
\|\Err\| \le c_4 \delta_t \sigmin^2 , \ \text{ with a } c_4 < (0.9-c_1),
\]
then, the old analysis applies without change. One can still show exponential decay of $\delta_t$ with $t$.
This is the case for example for LRPR for which the $\Err$ term bound is taken from \cite{lrpr_best} since this term also occurs when studying AltMin for LRPR.

This would also be the case for noisy LRCS or LRMC if one assumes a small enough bound on the noise (noise to signal ratio smaller than $c$ times the final desired error $\eps$).

\item If we want to obtain error bounds without making any assumptions on the noise, then the $\Err$ term does not satisfy the above bound. This is also the case for Byzantine-resilient AltGDmin for the vertically federated LRCS and (any) federated LRMC setting. Both cases are instances of heterogeneous gradients.
%
%In some problems, for example, when considering the LR recovery problems with additive noise (and not assuming a bound on the noise to signal ratio) or when studying Byzantine-resilient AltGDmin for vertically federated LRCS or for LRMC (ongoing work), the term $\Err$ does not satisfy the above assumption: its norm does not decay with $\delta_t$.
%
In these cases, suppose a much looser bound on $\|\Err_t\|$ holds: suppose that it is of the order of the initial error $\delta_0$, i.e., suppose that
\[
\max_t \|\Err_t\| \le 0.1 \delta_0 \sigmin^2
\]
%Consider the second case from above.
Substituting this bound only into the denominator term of \eqref{deltatbnds_noisy_2} (in the numerator we leave $\|\Err_t\|$ as is, this allows for a tighter bound in cases where the error is much smaller than its upper bound)
\begin{align*}
\delta_{t+1}
& \le  \delta_t \left( 1 - \frac{c_\eta}{\kappa^2}  \left( 0.9 - c_1 - 0.1 -\frac{2\|\Err\|}{\sigmin^2} \right)  \right)
+ \left(1 +  \frac{c_\eta}{\kappa^2} \left( 0.1    +  \frac{2\|\Err\|}{\sigmin^2} \right) \right) \frac{c_\eta}{\sigmax^2} \|\Err\| \\
& \le  \delta_t \left( 1 - \frac{c_\eta}{\kappa^2}  \left( 0.8 - c_1 - \frac{2\|\Err\|}{\sigmin^2} \right)  \right)
+ \left(1 +  \frac{c_\eta}{\kappa^2} \left(0.1  + \frac{2\|\Err\|}{\sigmin^2} \right) \right) \frac{c_\eta}{\sigmax^2} \|\Err\| \\
& \le  \delta_t \left( 1 - \frac{c_\eta}{\kappa^2}  \left( 0.8 - c_1  - 0.2 \delta_0 \right)  \right)
+ \left(1 +  \frac{c_\eta}{\kappa^2} \left( 0.1   +  0.2\delta_0 \right) \right) \frac{c_\eta}{\sigmax^2} \|\Err\|
%
%
%
%%& \le \delta_t (1 - \eta (  0.9 - c_1 ) + \left( 1 + 2\eta \|\E[\gradU]\| + 2 \eta \|\E[\gradU] - \gradU \| + 2 \eta \|\Err_t\| \right) \|\Err_t\| \\
%& \le \delta_t (1 - \eta   \sigmin^2 ( 0.9 - c_1 - 2(C_2 \kappa^2   + c_1) \delta_t) ) + \left( 1 + 2\eta (C_2 \kappa^2 + c_1) \delta_t \sigmin^2 + 2 \eta \|\Err_t\| \right) \|\Err_t\| \\
%% 2 \eta \delta_t \sigmin^2
%%& \le \delta_t \left(1 - \frac{c_\eta (0.9 - c_3)}{\kappa^2} \right) + ( 1 + 2 (c_\eta (C_2 + c_1/\kappa^2)/ ) \delta_t ) \|\Err_t\| \\
%%& \le \delta_t \left(1 - \frac{c_\eta (0.9 - c_3)}{\kappa^2} \right) + 1.1 \|\Err_t\|
\end{align*}
Using the upper bound on $\delta_0$ from Assumption \ref{assu1} (it assumes a bound on $\delta_t$ for any $t\ge 0$), we get  %and that  $(C_2 \kappa^2   + c_1) \delta_0 < 0.1$, then,
\begin{align}
\delta_{t+1}
 %& \le \delta_t (1 - \eta   \sigmin^2 ( 0.9 - c_1 - 0.01) ) + \left( 1 + 0.02 \eta \sigmin^2 + 2 \eta 0.01 \right) \eta \|\Err_t\| \\
% & \le \delta_t (1 - \eta   \sigmin^2 ( 0.9 - c_1 - 0.01) ) + \left( 1 + 0.02 c_\eta/\kappa^2 \right) \eta \|\Err_t\| \\
%& \le \delta_t (1 - \eta   \sigmin^2 ( 0.9 - c_1 - 0.01) ) + 1.1 \eta \|\Err_t\| \\
& \le \delta_t \left(1 - \frac{c_\eta}{\kappa^2} \left( 0.8 - c_1  - 0.02 \right) \right) + 1.12c_\eta \frac{\|\Err_t\|}{\sigmax^2}
%& = \delta_t (1 - \frac{c_\eta}{\kappa^2} ( 0.9 - c_1 - 0.01 - 0.1)
\label{recurs}
\end{align}
Using this inequality, and the bound on $\|\Err_t\|$, we can argue that $\delta_t \le \delta_0$ for all $t$. To see this, suppose $\delta_t \le \delta_0$. Using above, and $c_\eta <1$,  $\delta_{t+1} \le \delta_0 (1-0.78+c_1) + 0.112 \delta_0 = (0.22 + c_1 + 0.112)\delta_0 < \delta_0$ since $c_1 < 0.2$.
%The last two equations help show that $\delta_{t+1} \le \delta_t$ (and hence $\delta_{t+1} \le \delta_0$) if $c_1$ is small enough.
%
%
%with $\delta_0 = c/\kappa^2$, this suffices to simplify the above expression. We also use $\delta_t < \delta_0$. Then, we roughly get
%\begin{align*}
%\delta_{t+1}
%& \le \delta_t \left(1 - \frac{c_\eta (0.9 - c_3)}{\kappa^2} \right) + 1.1 \|\Err_t\|
%\end{align*}
The recursion in \eqref{recurs} can be simplified to get
\begin{align}
\delta_t
& \le \left(1 - \frac{c_\eta}{\kappa^2} (0.78-c_1) \right)^t \delta_0 + \sum_{\tau=1}^t \left(1 - \frac{c_\eta}{\kappa^2} (0.78-c_1) \right)^{t-\tau} 1.12 c_\eta \kappa^2 \frac{\|\Err_\tau\|}{\sigmin^2} \notag \\
& \le \left(1 - \frac{c_\eta}{\kappa^2} (0.78-c_1) \right)^t \delta_0 + \frac{1.12}{c_\eta(0.78-c_1)}\kappa^4 \max_{\tau \in [t]} \frac{\|\Err_\tau\|}{\sigmin^2}
\label{simplify_recurs}
\end{align}

In summary, if all the bounds from the noise-free case holds and if $\|\Err_t\| \le 0.1 \delta_0 \sigmin^2$, then \eqref{simplify_recurs} holds.
\een

\chapter{Linear Algebra and Random Matrix Theory Preliminaries}\label{prelims}
Most of the below review is taken from \cite{versh_book}.

\section{Linear algebra: maximum and minimum singular value and the induced 2-norm}\label{minsingval}
We denote the hyper-sphere in $\Re^n$ by $\calS_{n-1}$; thus $\calS_{n}:=\{\x \in \Re^n: \|\x\| = 1\}$.

For any matrix $\M$ of size $n_1 \times n_2$,
\[
\|\M\| = \sigmamax(\M) = \sigma_1(\M) = \max_{\x \in \calS_{n_2-1} } \|\M \x \|  = \max_{\x \in \calS_{n_2-1},\y \in \calS_{n_1-1} } \y^\top \M \x  = \sqrt{\lambda_{\max}(\M^\top \M)}
\]
If $\M$ is a symmetric matrix, then
\[
\|\M\| =  \max_{\x \in \calS_{n-1} } |\x^\top \M \x |
\]
For an $n_1 \times n_2$ matrix $\M$,
\[
\sigmamin(\M):= \sigma_{n_2}(\M):= \min_{\x \in \calS_{n_2-1} } \|\M \x \| = \sqrt{\lambda_{\min}(\M^\top \M)}
\]
i.e., it is the $n_2$-the singular value. Thus, for a rectangular matrix, $\sigmamin(\M) \neq \sigmamin(\M^\top)$ but $\sigmamax(\M) = \sigmamax(\M^\top)$ and, more generally,
\[
\sigma_i(\M^\top) = \sigma_i(\M)  %= \lambda_i(\M^\top \M) = \lambda_i(\M \M^\top)
\]
Weyl's inequality for singular values implies that
\[
\sigma_i(\M) - \|\A\|  \le \sigma_i(\M + \A) \le \sigma_i(\M) + \|\A\|
%\sigmamin(\M + \A) \ge \sigmamin(\M) - \|\A\|
\]
For an $n_1 \times n_2$ matrix $\M$ and an $n_2 \times n_3$ matrix $\A$,  if $\A$ has rank $n_3$, then,
\[
\sigmamin(\M \A) = \min_{\x \in \calS_{n_3-1}} \|\M \A \x\| \cdot \frac{\|\A\x\|}{\|\A\x\|} \ge \sigmamin(\M) \min_{\x \in \calS_{n_3-1}} \|\A\x\| = \sigmamin(\M) \sigmamin(\A)
\]

%\[
%\sigma_i(\M \A) \ge \sigmamin(\M) \sigma_i(\A) ,?? not \ correct i think
%\]
%%\[
%%\sigma_i(\M \A) = \sigma_i(\A^\top \M^\top) \ge  \sigmamin(\A^\top) \sigma_i(\M^\top) = \sigmamin(\A^\top) \sigma_i(\M)
%%\]
%%

%\section{Using these bounds}
%For the algebra lemma, we need to lower bound  $\lambda_{\min}(\B \B^\top) = \sigmamin^2(\B^\top)$.
%
%Using the bounds from Sec. \ref{minsingval},
%\[
%\sigmamin(\B^\top) \ge \sigmamin(\G^\top) - \|\B - \G\|
%\]
%and
%\[
%\sigmamin(\G^\top) = \sigma_r (\G^\top) =\sigma_r(\G) =  \sigma_r(\U^\top \Ustar \Bstar) \ge \sigmamin(\U^\top \Ustar) \sigma_r(\Bstar) = \sigmamin(\U^\top \Ustar)\sigmin
%%(\Bstar{}^\top \Uatar{}^\top \U)
%\]

%Using the bounds from Sec. \ref{minsingval},
%\[
%\sigma_r ( \E[\X_0] )
%= \sigma_r(\Xstar \D(\alpha))
%\ge \sigma_r(\Xstar) \sigmamin( \D(\alpha)^\top ) = \sigmin \sigmamin(\D(\alpha)) = \sigmin \min_k \beta_k(\alpha)
%\]

\section{Linear algebra: Wedin and Davis-Kahan $\sin \Theta$ theorems} \label{wedin_daviskahan}

%subspace distance equals sine of largest principal angle between the subspaces
%\[ \SE(\U,\hat\U):=\|(\I - \Uhat \Uhat^\top) \U\| \]

The following results are used to bound the subspace distance between the top $r$ singular or eigen vectors of a matrix and its estimate.

\begin{theorem}[Davis-Kahan for eigenvectors of symmetric matriaces]
For symmetric matrices $\S, \hat\S$, let  $\U,\hat\U$ denote the matrices of their top $r$ eigenvectors respectively. Then.
\[
\SE(\U,\hat\U)
 \le \frac{\|\S - \hat\S\|}{\lambda_r(\S) - \lambda_{r+1}(\hat\S)}
 \le \frac{\|\S - \hat\S\|}{\lambda_r(\S) - \lambda_{r+1}(\S) - \|\S - \hat\S\|}
\]
Also, let $\u_i$ denote the $i$-th eigenvector. Then, we have the following bound.
\[
\sin \theta(\u_i, \hat\u_i) \le \frac{\|\S - \hat\S\|}{\min_{j \neq i} |\lambda_j(\S) - \lambda_{i}(\S)}
\]
Here $\sin \theta(\u_i, \hat\u_i) = \sqrt{1 - (\u_i^T \hat\u_i)^2} $
\end{theorem}

\begin{theorem}[Wedin $\sin \Theta$ theorem for Frobenius norm subspace distance \cite{wedin,spectral_init_review}[Theorem 2.3.1]]
For two $n_1 \times n_2$ matrices $\M^*$, $\M$, let $\Ustar, \U$ denote the matrices containing their top $r$ left singular vectors and let $\Vstar^\top, \V^\top$ be the matrices of their top $r$ right singular vectors (recall from problem definition that we defined SVD with the right matrix transposed). Let $\sigma^*_r, \sigma^*_{r+1}$ denote the $r$-th and $(r+1)$-th singular values of $\M^*$. %Also let $\E:= \M - \M^*$.
If $\|\M - \M^*\| \le \sigma^*_r - \sigma^*_{r+1}$, then
\begin{align*}
&\SE_F(\U, \Ustar) \\
& \le \frac{\sqrt{2} \max(\|(\M - \M^*)^\top \Ustar\|_F, \|(\M - \M^*)^\top \Vstar^\top\|_F    )}{\sigma^*_r - \sigma^*_{r+1} - \|\M - \M^*\|}
\end{align*}
\begin{align*}
&\SE_2(\U, \Ustar) \\
& \le \frac{\sqrt{2} \max(\|(\M - \M^*)^\top \Ustar\|, \|(\M - \M^*)^\top \Vstar^\top\| )}{\sigma^*_r - \sigma^*_{r+1} - \|\M - \M^*\|}
\end{align*}
%
%$\M \svdeq \U \Sigma \V^\top + \U_\perp \Sigma_\perp \V_\perp^\top$  and $\hat\M \svdeq \hat\U \hat\Sigma \hat\V^\top + \hat\U_\perp \hat\Sigma_\perp \hat\V_\perp^\top$ with $\Sigma, \hat\Sigma$ being $r \times r$ diagonal matrices with entries $\sigma_j, \hat\sigma_j$, $j=1,2,\dots r$ respectively. Thus $\U, \hat\U$ are $n_1 \times r$ matrices with orthonormal columns. %and $\Sigma_\perp$ an $(n_1-r) \times (n_2 - r)$ diagonal matrix with entries $\sigma_j$, $j = (r+1),(r+2), \dots, \min(n_1,min_
\label{Wedin_sintheta}
\end{theorem}

\section{Probability results: Markov's inequality and its use to prove concentration bounds}

All the concentration bounds stated below use the Markov inequality, which itself is an  easy application of the integral identity
\begin{theorem}[Markov's inequality]
For a non-negative random variable (r.v.) $Z$,
\[
\Pr(Z > s) \le \frac{\E[Z]}{s}
\]
\end{theorem}
This result forms the basis of the entire set of results on non-asymptotic random scalar, vector, and matrix theory.
We obtain the Chebyshev inequality by applying Markov's inequality to $Z = |X- \mu|$ with $\mu = \E[X]$.
For all the other inequalities we use the Chernoff bounding technique explained below. This requires using an upper bound on the moment generating function (MGF) of the r.v.'s. This, in turn, requires assuming that the r.v.'s belong to a certain class of ``nice enough'' probability distributions (bounded, sub-Gaussian, or sub-exponential). With making one of these assumptions, the probability bound obtained is much tighter (decays exponentially) than what Chebyshev provides. However, the Chebyshev bound is the most general since it does not assume any distribution on the r.v.s.

\section{Probability results: Chernoff bounding idea}
The MGF of a random variable (r.v.) $X$ is defined as
\[
M_X(\lambda):= \E[\exp(\lambda X)]
\]
Chernoff bounding involves applying the Markov inequality to $Z = e^{tX}$ for any $t \ge 0$. Notice $e^{tX}$ is always non-negative.
\[
\Pr(X > s ) =  \Pr(e^{tX} > e^{ts} ) \le  e^{-ts} \E[e^{tX}]  =   e^{-ts} M_X(t)
\]
Since this bound holds for all $t \ge 0$, we can take a $\min_{t\ge 0}$ of the RHS or we can substitute in any convenient value of $t$.

If $S = \sum_{i=1}^m X_i$ with $X_i$'s independent, then $M_X( \lambda) = \prod_i M_{X_i}( \lambda)$.
The next step involves either using an exact expression for MGF or a bound on the MGF for a class of distributions, e.g.,  Hoeffding's lemma.
This is often followed by a scalar inequality such as  $1+x \le e^x$ or using $cosh(x) \le e^{x^2/2}$ (or other bounds) to simplify the expressions to try to get a summation over $i$ in the exponent.
\[
\Pr(S > s) = \Pr(\sum_i X_i > s ) =  \Pr(e^{t\sum_i X_i} > e^{t s } ) \le  e^{-ts} \E[e^{ t\sum_i X_i }]
%\le \min_{ \lambda \ge 0} e^{- \lambda s} M_{\sum_i X_i}( -\lambda) =
\le \min_{ \lambda \ge 0} e^{- \lambda s} \prod_i M_{X_i}(\lambda)
\]
The final step is to minimize over $ \lambda \ge 0$ by differentiating the expression and setting it to zero, or picking a convenient value of $ \lambda \ge 0$ to substitute.

A similar approach is then used to bound $\Pr(\sum_i X_i < - s )$. The only difference is we use  $Z = e^{-t\sum_i X_i}$ for $t \ge 0$ and so,
\[
\Pr(S <-s) = \Pr(\sum_i X_i <- s ) =  \Pr(e^{-t\sum_i X_i} > e^{-t \cdot (-s) } ) \le  e^{-ts} \E[e^{-t\sum_i X_i }]
%\le \min_{ \lambda \ge 0} e^{- \lambda s} M_{\sum_i X_i}( -\lambda) =
\le \min_{ \lambda \ge 0} e^{- \lambda s} \prod_i  M_{X_i}(-\lambda)
\]
Combine both of the above bounds to bound
\[
\Pr( |\sum_i X_i | > s ) = \Pr( \sum_i X_i  > s ) + \Pr( \sum_i X_i  < - s )
\]

\section{Probability results: bounds on sums of independent scalar r.v.s (scalar concentration bounds)}
The results summarized below are taken from \cite[Chap 2]{versh_book}. We state these results for sums of zero mean r.v.s. However, usually these are applied to show concentration of sums of nonzero mean r.v.s around their means. Given a set of nonzero mean r.v.s $Z_i$,
\[
X_i = Z_i - \E[Z_i]
\]
is zero mean.

The first result below, Chebyshev inequality, requires no assumptions on the r.v.s except that they have a finite second moment. But its probability bound is also the weakest. The three results below it are for sums of bounded, sub-Guassian, and sub-exponential r.v.s.

\begin{theorem}[Chebyshev's inequality]
Let $X_i$, $i=1,2,\dots,n$ be independent r.v.s  with $\E[\X_i^2] < \infty$. Then,
\[
\Pr(|\sum_i X_i| > t ) \le  \frac{1}{t^2} \sum_i \E[\X_i^2]
\]
\end{theorem}

\begin{theorem}[Bounded Bernstein inequality]
Let $X_i$, $i=1,2,\dots,n$ be independent zero-mean bounded r.v.s with $\Pr( -M_i \le X_i \le M_i) = 1$. Let $\sigma_i^2:= \max(\E[X_i^2])$.
Then
\[
\Pr(|\sum_i X_i) | \ge  t ) \le 2 \exp\left(- \frac{ 0.5 t^2}{\sum_i \sigma_i^2 + 0.33 (max_i M_i) t   } \right)
\]
\end{theorem}

\begin{definition}[Sub-Gaussian and Sub-exponential r.v.]
We say a r.v. $X$ is sub-Gaussian with sub-Guassian norm $K$ if $\Pr(|X| > t) \le 2 \exp(-t^2/K^2 )$. Equivalently, $K = C \sup_{p \ge 1} \frac{1}{\sqrt{p}} \E[ |X|^p]^{1/p}$.

We say a r.v. $X$ is sub-exponential with sub-exponential norm $K$ if $\Pr(|X| > t) \le 2 \exp(-t/K )$. Equivalently, $K = C \sup_{p \ge 1} \frac{1}{p} \E[ |X|^p]^{1/p}$.
\end{definition}

\begin{fact}[Product of sub-Gaussians is sub-exponential]
For two sub-Gaussian r.v.s $X,Y$ with sub-Gaussian norms $K_X, K_Y$, the r.v. $Z:= XY$ is sub-exponential with sub-exponential norm $K_X K_Y$.
\end{fact}

\begin{theorem}[Sub-Gaussian Hoeffding inequality]
Let $X_1, X_2, \dots X_n$ be independent zero-mean sub-Gaussian r.v.s with sub-Gaussian norm $K_i$. Then, for every $t \ge 0$,
\[
\Pr( | \sum_i X_i | \ge t ) \le 2 \exp\left( - c \frac{ t^2}{\sum_i K_i^2 } \right)
\]
\end{theorem}

\begin{theorem}[Sub-exponential Bernstein inequality]
Let $X_1, X_2, \dots X_n$ be independent zero-mean sub-exponential r.v.s with sub-exponential norm $K_i$. Then, for every $t \ge 0$,
\[
\Pr( | \sum_i X_i | \ge t ) \le 2 \exp\left( - c \min \left(  \frac{t^2}{\sum_i K_i^2 } , \frac{t}{\max_i K_i }  \right)  \right)
\]
\end{theorem}

\section{Probability results: Epsilon netting argument used for extending union bound to uncountable but compact sets}\label{epsnet}
The following discussion is taken from \cite[Chap 4]{versh_book}.
An ``epsilon net'' is a finite set of points that is used to ``cover'' a compact set by balls of radius $\eps$. More precisely, it is a finite set of points that are such that any point on the compact set is within an $\eps$ distance of some point in the epsilon-net. We use the bounds on the size of the smallest epsilon-net that covers a hyper-sphere to convert a scalar concentration bound into a bound on the minimum and maximum singular values of a large random matrix.
%such as the sub-exponential Bernstein or sub-Gaussian Hoeffding

%Recall the linear algebra basics summarized earlier in Sec. \ref{maxsingval} and \ref{minsingval}.
%The epsilon-net arguments used below use an epsilon-net to cover a hyper-sphere in $\Re^n$, $\calS_{n-1}$. So we define it precisely for this special case.
\begin{definition}[Epsilon net on a sphere]
We say $\mathcal{N}_\eps$ is an $\eps$-net covering $\calS_{n-1}$ in Euclidean distance if $\mathcal{N}_\epsilon \subset \S^{n-1}$  and if, for any $\x \in \S_{n-1}$, there exists a $\bar{\x} \in \mathcal{N}_\eps$ s.t. $\|\x - \bar{\x}\| \le \eps$.
\end{definition}
It can be shown, using volume arguments \cite{versh_book}, that there exists an epsilon-net, $\mathcal{N}_\eps$, covering $\calS_{n-1}$ whose cardinality can be bounded as
\[
|\mathcal{N}_\eps| \le (1 + 2/\eps)^n
\]
Using this bound, the following result can be proved for obtaining a high probability bound on the l2-norm of a matrix $\M$ with random entries.

%All the high probability bounds for initialization use sub-Gaussian Hoeffding inequality, while those for GD lemmas use the sub-exponential Bernstein inequality, both are from \cite{versh_book}. In addition, these lemmas also use the following results to ``epsilon-net'' extend a bound holding for a fixed unit norm $\W$ (or $\w$) to all unit norm $\W$s (or $\w$s)

%The epsilon net approach \cite[Chap 4]{versh_book} that is used in our proofs can be summarized in the result below. We explain what an epsilon net is Sec. \ref{prelims}.

\begin{theorem}[Bounding $\|\M\|$]
For an $n \times r$ matrix $\M$ and fixed vectors $\w, \z$ with, $\w \in \calS_n$ and $\z \in \calS_r$, suppose that
\[
|\w^\top \M \z | \le b_0 \text{ w.p. at least } 1-p_0 .
\]
where $b_0$ does not depend on $\w,\z$. Then,
\[
\|\M\| \le 1.4 b_0  \text{ w.p. at least } 1- \exp( (\log 17) (n+r) ) \cdot p_0
\]
\label{epsnet_Mwz}
\end{theorem}

\begin{proof}
Denote $\eps_0$-nets covering $\calS_{n-1}$ and $\calS_{r-1}$ by $\bar\calS_{n-1}$ and $\bar\calS_{r-1}$.
Using union bound,  w.p. at least $1 - (1+2/\eps_0)^{n+r} p_0$,
\bi
\item $\max_{\w \in \bar\calS_{n-1}, \z \in \bar\calS_{r-1}} |\w^\top \M \z | \le b_0$ and
\item $\|\M\| := \max_{\w \in \calS_{n-1}, \z \in \calS_{r-1}} \w^\top \M \z  \le \max_{\w \in \calS_{n-1}, \z \in \calS_{r-1}} |\w^\top \M \z | \le \frac{1}{1 - 2\eps_0 - \eps_0^2} b_0$.
\ei
The proof of the second item above follows that of Lemma 4.4.1 of \cite{versh_book}.

Using $\eps_0=1/8$ gives the final conclusion.
\end{proof}

\section{Probability results: bounding sums of independent matrix r.v.s (matrix concentration bounds)}

By combining Theorem \ref{epsnet_Mwz} given above with the scalar sub-exponential Bernstein or sub-Gaussian Hoeffding inequalities, we obtain the following two results which have been widely used in the LR recovery and phase retrieval literature. These study sums of rank-one matrices which are outer products of specific types of random vectors (r.vec). The last result below is the matrix Bernstein inequality.

\begin{corollary}[Sum of rank-one matrices that are outer products of two sub-Gaussian r.vecs. (repeated from Sec \ref{keyres})]\label{subexpo_epsnet_repeat}
Consider a sum of $m$ zero-mean independent rank-one $n \times r$ random matrices $\x_i \z_i^\top$ with $\x_i, \z_i$ being sub-Gaussian random vectors with sub-Gaussian norms $K_{x,i}, K_{z,i}$ respectively.
%that are such that, for unit vectors, $\w,\z$, $\w^\top \X_i \z$  are sub-exponential r.v.s with sub-exponential norm $K_{e,i}$ that does not depend on $\w,\z$. An example  is if each $\X_i$ is an outer product of two sub-Gaussian random vectors.
For a $t \ge 0$,
\[
\|\sum_{i=1}^m \x_i \z_i^\top\| \le 1.4 t
\]
with probability at least
\[
1- \exp\left( (\log 17) (n+r) - c \min \left(  \frac{t^2}{\sum_i (K_{x,i}, K_{z,i})^2 } , \frac{t}{\max_i (K_{x,i}, K_{z,i}) }  \right)  \right)
\]
\end{corollary}

By combining Theorem \ref{epsnet_Mwz} with the scalar sub-Gaussian Hoeffding inequality, we conclude the following.

\begin{corollary}[Sum of rank-one matrices that are outer products of a sub-Gaussian r. vec. and a bounded r.vec. (repeated from Sec \ref{keyres})]\label{subG_epsnet_repeat}
Consider a sum of $m$ zero-mean independent rank-one $n \times r$ random matrices $\x_i \z_i^\top$ with $\x_i$ being sub-Gaussian random vector with sub-Gaussian norms $K_{x,i}$ and $\z_i$ being a bounded random vector with $\|\z_i\| \le L_i$.
Then, clearly, for any $\w, \w'$, $\w^\top \x_i \z_i^\top \w'$ is a sub-Gaussian r.v. with sub-Gassian norm $K_{x,i} L_i$.
Thus, for a $t \ge 0$,
\[
\|\sum_{i=1}^m \x_i \z_i^\top\| \le 1.4 t
\]
with probability at least
\[
1- \exp\left( (\log 17) (n+r) - c  \frac{t^2}{\sum_i (K_{x,i} L_i)^2 } \right)
\]
\end{corollary}

For bounded matrices, the following matrix Bernstein result gives a much tighter bound than what would be obtained by combining scalar bounded Bernstein and Theorem \ref{epsnet_Mwz}. See Sec. \ref{prelims} for details.

\begin{theorem}[Matrix Bernstein (repeated from Sec \ref{keyres})]\label{mat_bern_repeat}
    Let $\X_1, \X_2, \dots \X_m$ be independent, zero-mean, $n \times r$ matrices with $\|\X_i\| \le L$ for all $i=1,2,...m$. Define the ``variance parameter'' of the sum
\[
 v := \max\left(  \|\sum_i \E[\X_i \X_i^\top]\|,  \|\sum_i \E[ \X_i^\top \X_i] \| \right).
\]
Then,
\[
\| \sum_{i=1}^m \X_i \| \le t
\]
with probability at least
\[
%1 - 2 \exp\left( - c \frac{t^2}{v + L t /3} \right) \ge
1 - 2 \exp\left(\log \max(n,r)  -  c \min\left(\frac{t^2}{v}, \frac{t}{L} \right)  \right)
\]
\end{theorem}

%By combining Theorem \ref{epsnet_Mwz} with the scalar sub-exponential Bernstein inequality, we obtain Corollaries \ref{subexpo_epsnet} and \ref{subG_epsnet}
%conclude the following. The first few results below consider sums of rank-one matrices which are outer products of specific types of random vectors (r.vec). The last result below is the matrix Bernstein inequality.

We can also obtain a corollary for the bounded Bernstein inequality by combining it with Theorem \ref{epsnet_Mwz}, but as we explain below, that is not useful. It is not as tight as directly using the matrix Bernstein inequality. We state this inequality as a remark next just to explain why it is not useful and why the matrix Bernstein inequality should be used instead for sums of bounded matrices.

\begin{remark}[Sum of outer products of bounded random vectors (Not Useful)]
Consider a sum of $m$ zero-mean independent bounded $n \times r$ random matrices $\X_i$ with $\|\X_i \| \le \M_i$. Let $\sigma_i^2 : = \E[\|\X_i\|^2]$
For a $t \ge 0$,
%let
%\[
%p_0(t) = 2 \exp\left( - c \min \left(  \frac{t^2}{\sum_i K_i^2 } , \frac{t}{\max_i K_i }  \right)  \right)
%\]
%Then,
\[
\|\sum_{i=1}^m \X_i\| \le 1.4 t
\]
with probability at least
\[
%1- \exp\left( (\log 17) (n+r)  - \frac{ 0.5 t^2}{\sum_i \E[(w^\top \X_i \z)^2] + 0.33 (max_i |\w^\top \X_i \z|) t   } \right) \ge
1- \exp\left( (\log 17) (n+r)  - c \min\left( \frac{t^2}{\sum_i \E[\|\X_i\|^2]},  \frac{t}{max_i \|\X_i\|} \right) \right)
\]
\label{not_useful}
\end{remark}
%This result is not as tight as the next one given below called matrix Bernstein inequality.

The matrix Bernstein bound given in Theorem \ref{mat_bern_repeat} is always better than the result of the above Remark \ref{not_useful}. The reason is that the positive term in the exponent is $\max(n,r)$ in the above case and $\log \max(n,r)$ in matrix Bernstein. Consider the negative term in the exponent. $L$ is the same in both cases, but variance parameter $v$ of matrix Bernstein is  upper bounded by the one used in the above remark.

\part{Open Questions: AltGDmin and Generalized-AltGDmin for other Partly-Decoupled Problems}

\newcommand{\Smat}{\bm{S}}
\newcommand{\Smatstar}{\Smat^*}

\chapter{Open Questions}\label{openques}
We describe open questions next.
%Examples of partly decoupled problems that have not been studied yet are described below. For these problems, it is clear that AltGDmin will be fast and communication-efficient per iteration. However its total iteration complexity under a desired small enough sample complexity assumption needs to be studied before one can decide that it is indeed a useful solution for it.

\section{Guarantees for a general optimization problem} %AltGDmin and for Generalized AltGDmin}
An open question is: what assumptions do we need on an optimization problem to show that AltGDmin for it will converge. And can we bound its iteration complexity.
Moreover, can we prove results similar to those for GD for AltGDmin and under what assumptions.  Finally, when can these be extended to analyze Stochastic-AltGDmin which replaces the GD step of AltGDmin by Stochastic GD.

%Generalizing the block coordinate descent idea and altGDmin
\section{Generalized AltGDmin}
AltGDmin has been established to be a useful (faster, more communication-efficient, or both) modification of AltMin for certain partly decoupled problems. Its overall idea is to replace the slower of the two minimization steps of AltMin by a single GD step. %AltMin is s special case of the block coordinate descent (BCD) algorithm for problems in which the unknown is best split into two subsets each of which can be correctly and efficiently solved.
For certain problems, the natural split up of the unknown variable set consists of three or more subsets, an example is robust PCA and extensions -- robust matrix completion and robust LRCS -- described below. For these problems, one can generalize the AltGDmin idea as follows. Consider the block coordinate descent (BCD) algorithm which is a generalization of AltMin for multiple blocls.
 Replace minimization for the slowest variable set in BCD with GD. We explain the idea in detail below using robust PCA as an example. %In federated settings, communication time often dominates computation time. To design a communication-efficient federated BCD algorithm, one should use GD for the set of variables that depend on data that is distributed across various nodes, while using full minimization for the set of variables that do not require data exchange across nodes.

%Some of the examples given below are in fact instances of AltGDmin ideas used to modify BCD.??

\section{Robust PCA and Extensions: A Partly Decoupled Example Problem for Generalized AltGDmin}
The modern definition of Robust PCA \cite{rpca,rpca2,robpca_nonconvex,rpca_gd,rmc_gd} is the following: recover  a LR matrix $\Xstar=\Ustar \Bstar$ and a sparse matrix $\Sstar$ from observed data which is their sum, i.e., from
\[
\Y:= \Ustar \Bstar + \S^*
\]
%or more generally from $\Y:= \Ustar \Bstar + \S^* + \E^*$ where $\E^*$ is noise or modeling error.
This problem occurs in foreground background separation for videos as well as in making sense of survey data with some outlier entries.  Thus, the optimization problem to be solved is
\[
\min_{\U,\B, \S: \U^\top \U = \I} f(\U,\B,\S):= \sum_k \|\y_k -  \s_k - \U \b_k\|_2^2
\]
In the above, clearly, $\Za=\U$ and $\Zb= \{\B,\S\}$.
The first approach to try would be to use AltGDmin with this split-up. However, this is a bad idea for two reasons. First there is no good approach to jointly solve for $\B,\S$.
Second, and more importantly, this approach is ignoring the important fact that the sparse component is the one with no bound on its entries' magnitudes,  hence it is the one that needs to be initialized and updated first. The entries of $\U$ are bounded because of the unit norm column constraint, while those of $\B$ are bounded because of the incoherence assumption and bounded condition number.
Our recommended  approach in this case is to consider the following generalization of AltGDmin (Gen-AltGDmin). %with $\Zb_1 = \S$, $\Za = \U$, $\Zb_2 = \B$. This would proceed as follows.
\bi
\item Initialize $\hat\S$ and then $\hat\U$: Initialize $\hat\S$ using thresholding with threshold proportional to $\alpha= \sigmax r/n$.  Obtain $\hat\U$ as the top $r$ singular vectors of $\X_0 := \Y - \hat\S$ %\sum_k \A_k^\top (\y_k - \hat\s_k) \e_k^\top $.

\item  Run $T$ iterations that alternate between the following three steps:
\bi
\item Obtain $\hat\b_k = \U^\top (\y_k - \A_k \hat\s_k)$ for all $k \in [q]$.
\item  Obtain $\hat\s_k$ by thresholding  $(\y_k - \A_k \hat\U \hat\b_k)$ using a threshold proportional to $0.5 \alpha$,  for all $k \in [q]$ ($\alpha$ decreases with iteration).%
\item Obtain $\hat\U$ by a GD step followed by orthonormalization. %: $\hat\U \leftarrow QR(\hat\U - \eta \nabla_\U g(\hat\U,\hat\B,\hat\S))$.%
\ei
% GD-$\hat\U$:
\ei
An open question is when does the above algorithm work? Can we bound its iteration complexity under a reasonable bound on the fraction of outliers (sparse entries) in each row and column?

For Robust LRCS, the goal is to recover $\Ustar \Bstar$ from $\Y:= \A_k \Ustar \Bstar + \A_k \S^*$. The overall idea above will generalize. Thresholding will get replaced by solving a compressive sensing problem. Update of $\b_k$ will solve an LS problem.  Robust LRMC would be handled similarly with some changes because we can only hope to estimate the sparse component of the observed entries. %one difference: we only want to recover $\S^*$ for the observed entries

\newcommand{\tL}{\bm{\mathcal{L}}}
\renewcommand{\tZ}{\bm{\mathcal{Z}}}
\renewcommand{\tS}{\bm{\mathcal{S}}}
\renewcommand{\tY}{\bm{\mathcal{Y}}}
\newcommand{\tE}{\bm{\mathcal{E}}}

\section{Partly Decoupled Tensor LR: Tensor LR  slicewise sensing}
The goal is to learn / recover / estimate a $(J+1)$-th order tensor from provided data. The class of tensor problems that we focus on treats the first $J$ dimensions of the tensor differently than the last one \cite{gdmin_tensor}. This is used to model tensor time or user sequences, e.g., multidimensional image sequences, or multiple product ratings by users. % or the products' dimensions vs users.
For the dynamic MRI and federated sketching problems, this models the fact that we have time sequence of 2D or 3D (or higher-dimensional) images. If we consider videos or single-slice dynamic MRI, then $J=2$. If we consider dynamic multi-slice MRI, then $J=3$ and so on.
For the recommender systems one, this models the fact that we are trying to design a system for $J$ products and the last dimension is the different users. Thus for a Netflix movie and shows recommendation system, we would use $J=2$, while if the products are movies, shows, and documentaries, we would use $J=3$ and so on. in case MRI data for different types of imaging are combined
%Existing work on tensor models ignores this.

%For tensor recovery, for simplicity, we only state the problem for the non-robust setting.
For a $(J+1)$-th order tensor, $\tZ$, we use $\tZ_k$ to denote the $k$-th frontal slice, e.g.,  if $J=2$, then $\tZ_k=\tZ(:,:,k)$ is a matrix for each $k$. %In the most general case, we assume that $\tZ = \tL + \tE$ where $\tL$ is a LR tensor,  and $\tE$ is modeling error in this model. %In the matrix setting, this LR+S model, and its hierarchical LR only extension, are one of the simplest and most commonly used models for image sequences, e.g., \cite{lin_fessler,lrpr_gdmin_mri_jp}.
There are many ways to define rank and the notion of LR for tensors; we assume a Tucker LR model \cite{kolda2009tensor} on $\tL$, since it is the most relevant model for our asymmetric setting. We define this next.
%with ranks $r_1, r_2, \dots r_J$ along the
%
%and $\tP$ denotes an unknown permutation of the observed data, e.g., due to record linkage or data synchronization errors

We need to learn a $(J+1)$-th order tensors $\tL$ of size $n_{1}\times n_{2} \times \dots \times n_J \times q$ from data $\tY$ of size $m_{1}\times m_{2}  \times \dots \times m_J \times q$ that satisfy
\begin{eqnarray}
    \tY_k:&=&  \tA_k(\tL_k + \tE_k) \ \forall \ k \in [q] \\  %+ \bm{\mathcal{W}}_k
    %\tL&=&\tG \times_{1} \U^{(1)} \times_2 \U^{(2)} \times_3 \dots \times_J \U^{(J)} \times_{J+1} \I \Leftrightarrow
    \tL_k&=& \tG_k \times_{1} \U^{(1)} \times_2 \U^{(2)} \times_3 \dots \times_J \U^{(J)} \ \forall \ k \in [q]
    \label{obsmod}
\end{eqnarray}
when $m_j \ll \min(n_j, q)$, $\tE$ is the modeling error or noise,  $\tL$ is an unknown LR tensor with Tucker ranks $r_1, r_2, \dots r_J,q$ with $\tG$ being a $r_{1}\times r_{2} \dots \times r_J \times q$ core tensor, $\U^{(j)}$ being $n_j \times r_j$ matrices denoting the subspace bases along the various dimensions. Here  $\times_j$ denotes the $j$-th mode product \cite{kolda2009tensor}.  The function $\tA_k(.)$ is a {\em known} dense linear function or is element-wise nonlinear, e.g., in case of phase retrieval. %In the slice-wise sensing problem, $\tA_k(.)$ is ``global'' for each $k$: each scalar entry of $\tY_k$ depends on all entries of its argument. %, e.g., when $J=2$, then $\tA_k(\tZ_k) = \tZ_k \times_1 \Ph_k^{(1)} \times_2 \Ph_k^{(2)} $ with $\Ph_k^{(1)},\Ph_k^{(2)}$ being i.i.d. random Gaussian or random Fourier matrices.
%In the completion problem, $\tA_k$ models observing a subset of entries.
%
To solve the above problem, we need to minimize
\begin{align}\label{mostgen_f}
f(\U^{(j)},j \in [J], \tG, \tS) & := \sum_k  \| \tY_k - \tA_k(\tS_k) - \tA_k(\tG_k \times_{1} \U^{(1)} \times_2 \U^{(2)} \times_3 \dots \times_J \U^{(J)})\|_F^2
\end{align}
%f_k, \ f_k:=
% that appear in all terms in the above summations,
Notice that this problem is partly decoupled with $\Za = \{\U^{(j)},j=1,2,\dots,J\}$ being the coupled variables and $\Zb = \tG$ being the decoupled one.

\subsubsection{Special case: 3D Tensor LRCS (LR Tensor Slice-wise Sensing)}\label{tensor_lrcs}
We define below the simplest tensor LR extension of LRCS for a third order tensor, thus $J=2$. We studied this numerically in  \cite{gdmin_tensor}.
Our goal is to learn a LR tensor $\tL$ of size $n_{1}\times n_{2}\times q$ from available third-order tensor measurements $\tY$ of size $m_{1}\times m_{2} \times q$ with $m_1 \ll n_1, m_2 \ll n_2$. In this case some of the tensor models simplify as given next.  %with% that satisfy %that are obtained as follows.
\begin{eqnarray}
    \tY_k &=& \tL_k \times_1 \Ph_k \times_2 \Ps_k = \Ph_k \L_k \Ps_k^\top, \ \forall \ k \in [q], \nn \\
\tL_k & = & \tG_k \times_{1}\U \times_{2} \V = \U \G_k \V^\top   \ \forall \ k \in [q] \nn
    \label{obsmod_J2}
\end{eqnarray}
The first equation above models a linear tensor function.
In this $J=2$ special case, it simplifies as a product of the three  matrices. The same is true for the LR model in this case. Here $\U$ is $n_1 \times r_1$, and $\V$ is $n_2  \times r_2$.
We need to solve $\min_{\U, \V, \tG: \ \U^\top \U = \I, \V^\top \V = \I} f(\U, \V, \tG):= \sum_{k=1}^q \|\Y_k  -  \Ph_k  \U \G_k \V^\top \Ps_k^\top \|_F^2$. This is clearly a partly decoupled problem with $\Za= \{\U, \V \}$ and $\Zb = \tG$.
 %subject to the constraint that it exhibits 2D low rank with ranks $r_1$ and $r_2$, which implies $rank(\tL_{(1)})=r_1$ and $rank(\tL_{(2)})=r_2$.
%Hence we can write $\tL=\tG \times_{1}\U \times_{2} \V$ where $\tG$ is $r_1 \times r_2 \times q$, $\U$ is $n_1 \times r_1$ and $\V$ is $n_2 \times r_2$, where $r_1 << \min(n_1, n_2q)$ and $r_2<< \min(n_2, n_1q)$.
%
%We use the alternating GD and minimization (altGDmin) approach to do this, since it provides the best speed and memory/communication complexity tradeoff. By modifying ideas from matrix LR literature, we start with a carefully designed spectral initialization to initialize $\U$ and $\V$. After this

AltGDmin in this case proceeds as follows \cite{gdmin_tensor}.
(1) We initialize using a modification of our matrix case idea. Define the initial tensor $\hat\tL_0$ as follows:  $(\hat\tL_0)_k = \Ph_k^\top \Y_k \Ps_k,  \ k \in [q] $. We initialize $\U$ as the top $r_1$ singular vectors of the unfolded matrix $(\tL_0)_{(1)}$: this means unfold $\tL_0$ along the first dimension to get a matrix of size $n_1 \times n_2 q$; and $\V$ as the top $r_2$ singular vectors of $(\tL_0)_{(2)}$. %: this means unfold $(\tL_0)$ along the second dimension to get a matrix of size $n_2 \times n_1 q$.
(2) We alternatively update $\tG$ and $\{\U,\V\}$ as follows.
(a) Given $\{\U,\V\}$, update $\tG$ by minimizing the above cost function over it. %Due to the form of our cost function, this decouples into a matrix-wise minimization for each $r_1 \times r_2$ frontal slice matrix $\G_k$ individually.
Since this decouples, this step consists of $q$ inexpensive least squares (LS) problems. For this $J=2$ special case, these are obtained as $\G_k = (\Ph_k \U)^\dag  \Y_k  ((\Ps_k \V)^\dag)^\top$ for each $k \in [q]$. Here $\M^\dag:= (\M^\top \M)^{-1} \M^\top$. %See line 8 of Algorithm \ref{gdmin_prac2}.
(b) Given $\tG$, we update $\{\U,\V\}$ using GD followed by projecting the output onto the space of matrices with orthonormal columns. This is done using the QR decomposition. %This projection helps ensure that the norms of both $\{\U,\V\}$ and of $\tG$ remain bounded: the former remain unit norm. %Without this projection step, the norm of one of $U,V,\tG$ could keep increasing while that of the other keeps decreasing.
    The QR decomposition is very cheap, it is of order $n_1 r_1^2$ and  $n_2 r_2^2$ respectively.
It can be easily derived that the time complexity of the above algorithm is much lower than that of a matrix LRCS algorithm that vectorizes the first $J=2$ dimensions.  The same is true for the communication cost for a distributed implementation. %We provide a summary in Table \ref{table_compare}.
%of the observed tensor and the recovered one.

While the above algorithm converges numerically as shown in  \cite{gdmin_tensor}, its theoretical analysis is an open question.

%We evaluated this  for a simple video sketching problem using the LR model. From preliminary experiments, our algorithm converges, and provides comparable accuracy results to those obtained when using the vectorized LR (matrix LR) formulation. This is with simple choices of the ranks and other parameters. % if the ranks and other parameters were carefully learned, we would expect this model to be significantly better.
%However time-wise the algorithm is much faster than the matrix LR one. Also memory requirements are much lower.

\section{Partly Decoupled Not-Differentiable Problems} \label{examples_not_fully_diff}
%

%\bfpara{Phase retrieval (PR)} The standard PR problem, recover $\x$ from $\y:=|\A \x|$, has been extensively studied. This of course requires oversampling, i.e., $m \ge n$. The first iterative algorithm for this was an AltMin algorithm that alternated between estimating the phase $\c_j:= phase( (\A \x)_j)$ and $\x$. Given $\x$, it is easy to obtain the phase. Given, the phase, $\x$ can be recovered by LS \cite{pr_altmin}.
%A faster GD algorithm was studied in later work called TWF \cite{twf}. This needs $m \ge C n$ samples and $mn \log(1/\eps)$ time for $\eps$-accurate recovery. While the order-wise time complexity of TWF cannot be improved, the gradient computation can be simplified, if we instead use AltGDmin. This would alternate between estimating the phase (minimization) and updating $\x$ by GD for the LS problem with $\y$ replaced by $\y \otimes \c$.
%%%The best (nearly linear-time and order-optimal sample complexity) solution is a GD algorithm with a truncated initialization
%%An altGDmin algorithm can be designed for this problem by
%%standard PR: altgdmin: gd for the LS step, min for the phase.
%
%% comm-efficient soln to standard PR problem also. AltMin solution (Netrapalli, ??, 2013 paper): alt between solving for phase, and then LS for x.  the LS for x can be replaced by GD.  this can be useful if different entries of y (or sub-vectors of y) are sensed at different distrib nodes.

\newcommand{\ccstar}{c^*}
\bfpara{Clustering}
%
%A different example is Gaussian mixture model (GMM) based clustering with a small number of classes, but high-dimensional data vectors. Suppose that the covariance matrix has some structural assumptions, say diagonal or LR. Then, $\Za$ is the model parameters while $\Zb$ is the class labels. The class labels can be updated privately at the nodes. Assuming good initialization, the model update will not require sharing all data with the central server.
%Notice that, in this case, the cost function is not differentiable w.r.t. the class labels and that is allowed. (mean, covariance matrix, mixture probabilities)
%
Data clustering is another set of problems where an AltGDmin type approach can be very beneficial. %We explain our main idea by explaining how to modify the k-nearest-neighbor algorithm but we will use the same overall idea to also modify other more challenging cluster problems such as those that arise if there is missing data or the data is assumed to like a union-of-subspaces model.
%Consider first the simplest distance-based clustering problem.
We have data points $\xstar_1, \xstar_2, \dots, \xstar_q$ that belong to one of $\rho$ classes with $\rho$ being a small number. Each $\xstar_k$ is an $n$-length vector (or matrix/tensor) and $n$ is typically large.  We do not know any of the class labels which we denote by $\ccstar_1, \ccstar_2, \dots \ccstar_q$ and would like to find them.
We assume that each of the $\rho$ classes is represented by a ``low-dimensional model'', $\mathcal{M}_j$, $j \in [\rho]$. For example, we could assume a Gaussian mixture model (GMM) where data points from class $j$ occur with probability $p_j$ and have mean and a covariance matrix $\bm\mu_j$ and $ \C_j$ respectively. Without any assumptions on $\C_j$, there are too many ($n^2 \rho$) unknowns. Often a diagonal assumption is made but this may not be valid if different entries of $\xstar_k$ are correlated. A milder assumption is to assume that each $\C_j$ is LR with rank $r \ll n$, i.e., $\C_j \svdeq \U_j \bSigma_j \U_j^\top$ with $\U_j$ being $n \times r$.

Then AltGDmin with  $\Za \equiv \{p_j, \bm\mu_j, \U_j, \bSigma_j, j \in [\rho] \}$ (model parameters for all the classes) and $\Zb \equiv \{\ccstar_k, k \in [q]\}$ (class labels) will be faster than the traditional AltMin based solution (k-means clustering) \footnote{AltMin requires a full model parameters' estimate at each iteration; this is expensive when $\C_j$s are not assumed to be diagonal.}.

\bfpara{Maximum or Mixed linear regression}
This \cite{max_aff_reg,mixed_lin_reg} is another problem that can be solved more efficiently using AltGDmin instead of the AltMin algorithm \cite{max_aff_reg}. This assumes that $y = \max_{j \in [\rho]} \beta^*_j{}^\top \x + w$. The goal is to learn the $\rho$ regression vectors $\beta^*_j, j \in [\rho]$ using training data pairs $\{\xstar_k,y_k\}, k \in [q]$ with $q > \rho n$. Let  $\c_k^*, k \in [q]$ denote the index of the best model for the $k$-the training data pair.  AltMin \cite{max_aff_reg}  alternates between estimating $\c^*_k$s given $\beta_j$'s ($q$ decoupled scalar maximizations) and vice versa: use the now labeled data to learn the $\beta_j$'s by solving an LS problem. AltGDmin would replace this LS by a single GD step.% making it much faster
%with  $\Za = \{\beta_j, j \in [\rho]\}$ and $\Za =\c_k^*, k \in [q]$   can be replaced by AltGDmin with

%interpreted as data clustering with the ``data'' being $\{\xstar_k,y_k\}, k \in [q]$ and each   %can, in fact, also be interpreted as an instance of the clustering problem

\bfpara{Unlabeled or Shuffled sensing} This involves recovering $\x$ from $\y:= \bm\Pi \A \x$ when only $\y,\A$ are given  \cite{uls_sparse,uls_rlocal}.
Here $\x$ is the unknown signal, and $\bm\Pi$ is an unknown $m \times m$ permutation matrix. %This is a linear regression problem where the order of the entries of $\y$ is either fully unknown or subsets of it have been shuffled.
This problem occurs in simultaneous location and mapping for robots, multi-target tracking, and record linkage.
Polynomial time solutions for it are possible with some assumptions on $\bm\Pi$, e.g., $\bm\Pi$ is often assumed to be $s$-sparse: i.e., only $s$ or less of its rows are permuted. In this case, one common solution is to initialize $\bm\Pi$ to $\I$ and use AltMin to update $\x$ and $\bm\Pi$ alternatively. The update of $\x$ is a robust regression problem (LS with sparse outliers), while that of updating $\bm\Pi$ is a well-studied discrete optimization problem called linear assignment problem (LAP). AltMin can be speeded up if we replace it by AltGDmin: instead of updating $\x$ by fully solving the robust regression problem in each iteration, we replace it by one GD step. %This reduces the cost of this step from order $mn \log(1/\eps)$ to order $mn$. The required number of iterations may increase of course.
Another assumption on $\bm\Pi$ is the $s$-local assumption \cite{uls_rlocal}. Under this assumption, the LAP would be decoupled making it faster.
%(which is often solved using modifications of compressive sensing solutions)
%
%In applications such as record linkage,  the $s$-local assumption is more useful: this means that the permutation is only within each sub-block of $\bm\Pi$ size $s$. In both cases, an AltMin algorithm with a carefully designed initialization can be developed: given $\x$, $\bm\Pi$ can be updated by solving a  Given $\bm\Pi$, recovery of $\x$ is a LS problem. %In the  former case, $\bm\Pi$ can be initialized to $\I$. In the latter case, one can left multiply both sides by a matrix in such a way that $\bm\Pi$ disappears
%
%If $\bm\Pi$ is $s$-local, the LAP problem decouples and its complexity then becomes only $(n/s) s^3 = n s^2$ rather than $n^3$. In this case, AltGDmin would replace solving of the LS problem at each iteration by a single GD step. This would speed up the algorithm quite a bit; and also enable an efficient federated design.
%
%%with complexity $m n \log (1/\eps)$ which has complexity $n^3$

\appendix

\chapter{Partly Decoupled Optimization Problem: Most General definition} \label{partly_decoup_appendix}
%We say that the problem is decoupled if it can be solved by solving smaller dimensional problems over disjoint subsets of $\Z$. We say it is also ``data-decoupled'' if it is decoupled and if these smaller problems depend on disjoint subsets of the data. %To be precise, %we assume here that the cost is a sum of terms with term depending on disjoint subsets of the unknown $\Z$.
%Consider optimization problems for which the set of unknowns $\Z$ can be split into two parts $\Z = \{ \Za,\Zb\}$ so that optimization over one keeping the other fixed is ``easy'' (closed form or provably correct or faster). We say that such a problem is partly decoupled w.r.t $\Zb$ if $\min_{\Zb} g(\{\Za,\Zb\})$ is decoupled.  We give the precise mathematical definition in Sec. \ref{partlydecoupled_define}. We give two examples below which will clarify the idea.

Consider an optimization problem $\arg\min_\Z  g(\Z)$.  We say the problem is decoupled if it can be solved by solving smaller dimensional problems over disjoint subsets of $\Z$. %We say it is also ``data-decoupled'' if it is decoupled and if these smaller problems depend on disjoint subsets of the data. Most decoupled optimization problems are also data-decoupled and henceforth, we just refer to such problem as ``decoupled''.
To define this precisely, observe that any function $g(\Z)$ can be expressed as a composition of $\gamma$ functions, for a $\gamma \ge 1$,
\[
g(\Z) = h(f^1(\Z), f^2(\Z), \dots f^\gamma(\Z)),
\]
Here $h(.,.,..)$ is a function of $\gamma$ inputs. This is true always since we can trivially let $\gamma=1$, $h(\Z) = \Z$ and $f^1(\Z) = g(\Z)$.

We say that the optimization problem is decoupled if, for a $\gamma > 1$,  $\Z$ can be split into $\gamma$ disjoint subsets
\[
\Z = [\Z_1, \Z_2, \dots \Z_\gamma]
\]
so that
\[
\argmin_\Z  g(\Z) =  [ \argmin_{\Z_1}   f^1(\Z_1), \argmin_{\Z_2} f^2(\Z_2), \dots, \argmin_{\Z_\ell} f^\ell(\Z_\ell), \dots  \argmin_{\Z_\gamma} f^\gamma(\Z_\gamma) ) ]
\]
Observe that, in general, $\argmin$ is a set and the notation $[\calS_1, \calS_2, \dots \calS_\gamma]$ is short for their Cartesian product $\calS_1 \times \calS_2 \times \dots \calS_\gamma$. In words, the set $\argmin_\Z  f(\Z) = \{ [\hat\Z_1, \hat\Z_2, \dots \hat\Z_\gamma]: \hat\Z_1 \in \argmin_{\Z_1}   f^1(\Z_1), \hat\Z_2 \in \argmin_{\Z_2}   f^2(\Z_2), \dots, \hat\Z_\gamma \in \argmin_{\Z_\gamma} f^\gamma(\Z_\gamma) \}$.
%for functions $f^k$ such that $g(\Z) = h(f^1(\Z), f^2(\Z), \dots f^\gamma(\Z))$ for some function $h(.,.,..)$ of $\gamma$ inputs. Also

%
If $g(\Z)$ is strongly convex, then the $\argmin$ is one unique minimizer $\hat\Z$. In this case, the decoupled functions have a unique minimizer too and $\argmin_{\Z_1}   f^1(\Z_1)$ returns $\hat\Z_1$ and so on, and $\hat\Z =[\hat\Z_1, \hat\Z_2, \dots, \hat\Z_\gamma]$.
%
%If it is not strongly convex, then the $\argmin$ is a set. In this case, the notation $[\S_1, \S_2, \dots \S_\gamma]$ is short for their Cartesian product $\S_1 \times \S_2 \times \dots \S_\gamma$. In words, the set $\argmin_\Z  f(\Z) = \{ [\hat\Z_1, \hat\Z_2, \dots \hat\Z_\gamma]: \hat\Z_1 \in \argmin_{\Z_1}   f^1(\Z_1), \hat\Z_2 \in \argmin_{\Z_2}   f^2(\Z_2), \dots, \hat\Z_\gamma \in \argmin_{\Z_\gamma} f^\gamma(\Z_\gamma) \}$.
%%For sets $\S_1 , \S_2, \dots \S_\gamma$, the notation
%
Data-decoupled means that the above holds and that $ef^\ell(\Z_\ll)$ depends only on a disjoint subset $\calD_\ell$ of the data $\calD$. Let $\calD = [\calD_1, \calD_2, \dots \calD_\gamma]$.
We use a subscript to denote the data. Data-decoupled means that
 %$f^\ell$ also depends only on a subset $\calD_\ell$  of the data, so that,
\[
\argmin_\Z  f(\Z) =  [ \argmin_{\Z_1}   f_{\calD_1}^1(\Z_1), \argmin_{\Z_2} f_{\calD_2}^2(\Z_2), \dots, \argmin_{\Z_\ell} f_{\calD_\ell}^\ell(\Z_\ell), \dots  \argmin_{\Z_\gamma} f_{\calD_\gamma}^\gamma(\Z_\gamma) ) ]
\]
Most practical problems that are decoupled are often also data-decoupled. {\em Henceforth we use the term ``decoupled'' to also mean data-decoupled.}

%\[
%\min_\Z  f(\Z) = f_{nondec} ( \sum_\ell  \min_{\Z_\ell}  f_{\calD_\ell}^\ell(\Z_\ell) )
%\]
Partly-decoupled is a term used for optimization problems for which the unknown variable $\Z$ can be split into two parts, $\Z = \{\Za, \Zb\}$, so that the optimization over one keeping the other fixed is ``easy'' (closed form, provably correct algorithm exists, or fast).  Decoupled and data-decoupled w.r.t. $\Zb$ means that decoupling holds only for minimization over $\Zb$. To be precise, let
\[
\Zb = [(\Zb)_1, (\Zb)_2, \dots (\Zb)_\gamma] \text{ and  }  \calD = [\calD_1, \calD_2, \dots \calD_\gamma]
\]
Then,
\[
\argmin_\Zb f(\Za,\Zb) =  [ \argmin_{(\Zb)_1}   f_{\calD_1}^1(\Za,(\Zb)_1), \dots, \argmin_{(\Zb)_\ell} f_{\calD_\ell}^\ell(\Za,(\Zb)_\ell), \dots  \argmin_{\Z_\gamma} f_{\calD_\gamma}^\gamma(\Za,(\Zb)_\gamma) ) ]
\]

All the examples of partly decoupled optimization problems that we discuss in this work are those for which $g(\Z) = h(f^1, f^2, \dots f^\gamma) = \sum_{\ell=1}^\gamma f^\ell$ is a sum of the $\gamma$ functions $f^\ell$. In this case,  partly decoupled problems means that
\[
\min_\Zb f(\Za,\Zb)  =  \sum_\ell  \min_{(\Zb)_\ell}  f_{\calD_\ell}^\ell(\Za,\Zb_\ell)
\]

\bibliographystyle{IEEEtran}
\bibliography{./bibfiles/tipnewpfmt_kfcsfullpap,./bibfiles/refs,./bibfiles/citations,./bibfiles/byz,./bibfiles/refs_Silpa} %,../bib/Decentralized-ST,./selin_refs,./refs_Silpa,./citations_aditya}%} %,./refs, ./citations_aditya,./citations,

%\addbibresource{sample-now.bib}

\end{document}